\documentclass{article}

\usepackage{graphicx}
\usepackage{subcaption}
\usepackage{booktabs} 
\usepackage{multirow}
\usepackage[table]{xcolor}
\usepackage{colortbl}
\usepackage{tcolorbox}
\usepackage{wrapfig} 
\usepackage{hyperref}

\usepackage[main,final]{neurips_2026}
\makeatletter
\@ifundefined{@trackname}{\newcommand{\@trackname}{\@neuripsordinal\ Conference on Neural Information Processing Systems (NeurIPS \@neuripsyear).}}{}
\makeatother

\usepackage{amsmath}
\usepackage{amssymb}
\usepackage{mathtools}
\usepackage{amsthm}

\usepackage[capitalize,noabbrev]{cleveref}

\usepackage{tikz}
\usepackage{enumitem}
\DeclareMathOperator{\Real}{Re}

\usetikzlibrary{arrows.meta}
\usetikzlibrary{shapes.geometric, arrows, positioning, decorations.pathmorphing}
\usetikzlibrary{patterns}

\DeclareMathOperator{\diag}{diag} 
\DeclareMathOperator{\Spec}{Spec} 
\DeclareMathOperator{\vecop}{vec} 
\DeclareMathOperator{\rank}{rank}
\DeclareMathOperator{\Erank}{ERank}

\providecommand{\R}{\mathbb{R}}

\newcommand{\mat}[1]{\mathbf{#1}} 
\newcommand{\vect}[1]{\mathbf{#1}}
\newcommand{\code}[1]{\texttt{#1}}

\definecolor{myblue}{rgb}{0.2,0.4,0.8}
\definecolor{myorange}{rgb}{0.2,0.7,0.3}
\definecolor{mygreen}{rgb}{0.9,0.1,0.1}
\definecolor{rowgray}{gray}{0.96}

\definecolor{mypurple}{rgb}{0.2,0.4,0.8}
\definecolor{myred}{rgb}{0.2,0.7,0.3}
\definecolor{myteal}{rgb}{0.9,0.1,0.1}

\theoremstyle{plain}
\newtheorem{theorem}{Theorem}[section]
\newtheorem{proposition}[theorem]{Proposition}
\newtheorem{lemma}[theorem]{Lemma}
\newtheorem{corollary}[theorem]{Corollary}
\theoremstyle{definition}

\newtheorem{assumption}[theorem]{Assumption}
\theoremstyle{remark}

\usepackage[disable,textsize=tiny]{todonotes}
\usepackage{microtype}

\definecolor{newred}{rgb}{0.85,0.0,0.0}

\title{Drive vs.\ Decay: On the Training Dynamics of\\
Joint-Embedding Predictive Architectures}

\author{%
  Jos\'e Lucas De Melo Costa \quad Seong Woo Ahn \quad Fabrice Popineau \\ \textbf{Arpad Rimmel \quad Bich-Li\^en Doan} \\[3pt]
  Universit\'e Paris-Saclay, CNRS, CentraleSup\'elec, LISN, Orsay, France \\
  \texttt{jose.lucas-de-melo-costa@universite-paris-saclay.fr}
}

\begin{document}

\maketitle

\begin{abstract}
Joint-Embedding Predictive Architectures (JEPAs) are prone to representation collapse, typically mitigated through empirical heuristics. We develop an early-training stability theory that unifies these heuristics. Linearising the coupled JEPA gradient flow around the trivial fixed point reveals two competing effects: a driving force ($\gamma$) and a decay effect ($\sigma$). Under approximate spectral decoupling, a per-mode stability ratio $\mu_i = \gamma_i / \sigma_i$ factorises into independent data-side and predictor-side terms and the count of unstable modes tracks the rank of representations that can emerge. The framework predicts a phase boundary, which we confirm empirically across more than 800 Tabular-JEPA configurations. It also unifies predictor scaling, masking ratio, and EMA as distinct mechanisms for shifting $\mu$. Guided by this analysis, we introduce ResidualPred, a transformer predictor whose attention is biased toward the identity at initialisation; it improves both effective rank and downstream accuracy on tabular benchmarks and in I-JEPA pretraining on CIFAR-10, CIFAR-100, STL-10, and ImageNet. Our framework connects empirical collapse-avoidance heuristics to an explicit dynamical picture, yielding theory-driven stabilizers. Code is available at \url{https://github.com/jose-melo/drive-vs-decay}.
\end{abstract}

\section{Introduction}
\label{sec:introduction}

\begin{wrapfigure}[17]{r}{0.38\linewidth}
\vspace{-20pt}
\centering
\includegraphics[width=\linewidth]{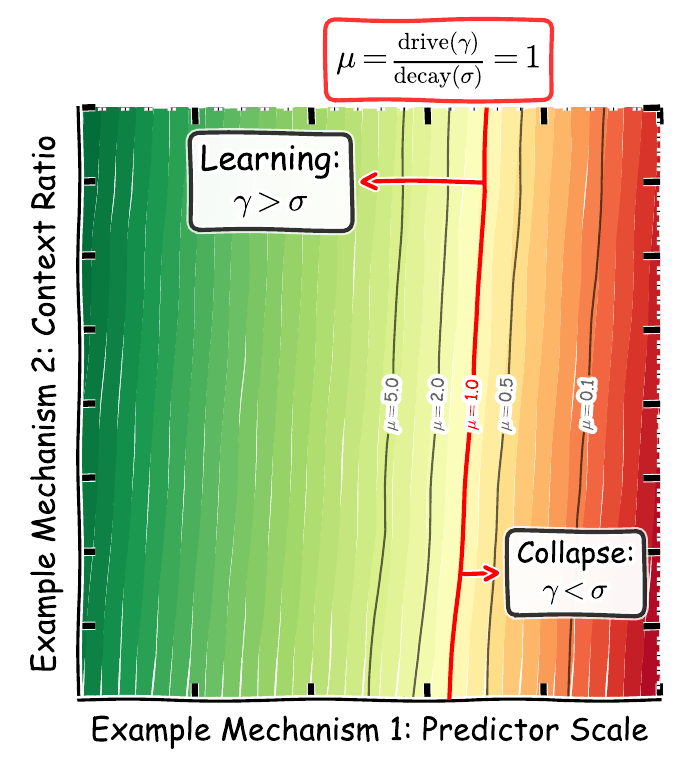}
\vspace{-25pt}
\caption{\small\textbf{Phase boundary.} Learning where the drive wins ($\mu=\gamma/\sigma>1$), collapse where the decay wins ($\mu<1$).}
\label{fig:main_figure}
\end{wrapfigure}
Self-Supervised Learning (SSL) has emerged as a powerful approach for representation learning, enabling models to learn meaningful features from vast amounts of unlabeled data \citep{shwartzzivCompressNotCompress2024}. Among SSL techniques, non-contrastive methods have gained popularity as they avoid the need for explicit negative sampling \citep{grillBootstrapYourOwn2020, bardesVICRegVarianceInvarianceCovarianceRegularization2022, chaoningzhangHowDoesSimSiam2022} and Joint Embedding Predictive Architectures (JEPAs) are a notable variant within this family, operating by predicting the latent representation of a target view of the input (e.g., a masked portion) from the representation of a context view (e.g., the remaining portion) \citep{assranSelfSupervisedLearningImages2023,feiAJEPAJointEmbeddingPredictive2024, dongBrainJEPABrainDynamics2024}
Despite the empirical success of JEPAs, a key concern is the potential for representational collapse, where the learned embeddings converge to a trivial solution or low-rank subspace, rendering the representations low-informative for the downstream task.

How can we prevent representational collapse when the collapsed solution can also minimize the very mean-squared prediction objective we optimize? Most existing answers have been empirical: the asymmetric predictor, a stop-gradient on the target branch, and slowly evolving targets via an exponential moving average (EMA) \cite{assranSelfSupervisedLearningImages2023}. A growing line of work has begun to formalize these stabilizers and to propose alternatives, including carefully designed predictor initializations \cite{bagatellaTDJEPALatentpredictiveRepresentations2025}, frozen-teacher formulations \cite{liRethinkingJEPAComputeEfficient2025}, and auxiliary regularizations that enforce embedding diversity. These solutions are largely justified by empirical heuristics, underscoring the need for a theory of JEPA training dynamics. Existing theoretical analyses have primarily focused on architectures augmented with explicit diversity-inducing regularizers \cite{balestrieroLeJEPAProvableScalable2025, yuWhyHowAuxiliary2025}. In parallel, the seminal study of \citet{tianUnderstandingSelfsupervisedLearning2021} dissected several anti-collapse mechanisms in Siamese non-contrastive frameworks, offering key dynamical insights. Yet, a substantial gap remains: we still lack a complete characterization of the original JEPA objective and under realistic update rules.

Therefore, this work proposes a theoretical framework for the training dynamics of JEPAs that yields principled mechanisms for avoiding early-stage collapse. Our analysis reveals that representation learning is governed by the spectral competition between two operators: a \emph{driving} force, induced by view construction (e.g., masking) and predictor-target alignment, and a \emph{decay} force, induced by predictor capacity and the data covariance structure (Figure~\ref{fig:main_figure}). This perspective unifies several empirical observations in the JEPA literature, and it suggests new architectural and initialization strategies that improve stability. We validate the theory through a large-scale study of $800$ Tabular-JEPA variants on standard tabular benchmarks, and we further demonstrate its relevance beyond the linear regime using transformer-based encoders. Our main contributions are:

\begin{itemize}

    \item \textbf{A drive--decay stability theory for JEPAs.} 
    We show that collapse stability is governed by a per-mode ratio $\mu_i=\gamma_i/\sigma_i$. The threshold $\mu>1$ predicts escape from collapse, while the number of super-critical modes yields a rank budget for the emerging representation.

    \item \textbf{A unified account of collapse controls.}
    Across more than $800$ configurations, we show that predictor scale and context masking act as control knobs for  $\mu$, while EMA acts as a timescale parameter. Together, these mechanisms explain transitions in effective rank and downstream performance. In the linear model, seven collapse-prevention mechanisms reduce to doses on $\mu$.

    \item \textbf{A theory-guided transformer predictor.}
    We introduce \texttt{ResidualPred}, an identity-biased transformer predictor with diagonal attention bias. It improves effective rank and linear-probe accuracy over standard predictors on brittle tabular settings and I-JEPA image benchmarks.
\end{itemize}

\section{Related Work}
\label{sec:related_work}

\paragraph{JEPAs and collapse-avoidance techniques.}
Joint Embedding Predictive Architectures (JEPAs) are self-supervised models that learn by predicting high-level representations rather than reconstructing raw inputs \citep{assranSelfSupervisedLearningImages2023, bardhanHEPJEPAFoundationModel2025, dongBrainJEPABrainDynamics2024, skenderiGraphlevelRepresentationLearning2025}. A standard JEPA contains a \emph{context encoder}, a \emph{target encoder}, and a \emph{predictor}: the context encoder processes a masked view of the input, the target encoder encodes a target view, and the predictor is trained to match the target representation. To avoid trivial solutions, where all inputs are mapped to the same representation, JEPAs usually rely on anti-collapse mechanisms such as EMA target updates \citep{grillBootstrapYourOwn2020}, variance--covariance regularization \citep{moConnectingJointEmbeddingPredictive, drozdovVideoRepresentationLearning2024}, frozen teachers \citep{liRethinkingJEPAComputeEfficient2025}, or distributional constraints such as SIGReg \citep{balestrieroLeJEPAProvableScalable2025}. Recent probabilistic and theoretical perspectives further connect these regularizers to embedding-distribution control and density estimation \citep{huangVJEPAVariationalJoint2026, balestrieroGaussianEmbeddingsHow2025}. Thus, while the JEPA principle is simple, its practical success depends critically on mechanisms that prevent representational collapse.

\paragraph{Training Dynamics of Non-Contrastive SSL}
Understanding how non-contrastive SSL methods, including JEPAs, avoid collapse without explicit negative pairs remains an active research topic. \citet{tianUnderstandingSelfsupervisedLearning2021} studied SimSiam-like dynamics with linear predictors, showing that the interaction between the predictor and stop-gradient prevents gradient alignment toward collapse, acting as an implicit regularizer. \citet{chaoningzhangHowDoesSimSiam2022} extended this view through a unified gradient-based analysis, identifying the gradient components that preserve representation diversity across non-contrastive methods. Related work further highlights predictor networks as implicit variance regularizers \citep{halvagalImplicitVarianceRegularization2023}, reads EMA and stop-gradient as orthonormalisation of a linear predictor \citep{richemondEdgeOrthogonality2023}, while JEPA-specific analyses suggest that these architectures may also benefit from an inherent learning bias against collapse \citep{littwinHowJEPAAvoids2024}.

\paragraph{Identity-biased initialisation.}
Biasing residual blocks toward the identity at initialisation is well established outside JEPAs, notably in Fixup \citep{zhangFixupInitializationResidual2018}, ReZero \citep{bachlechnerReZeroAllYou2021}, SkipInit \citep{deBatchNormalizationBiases} and IDInit \citep{panIDInitUniversalStable2025}. In these works, identity-biased initialisation is mainly used to stabilise signal propagation in deep networks. This work studies it in a different setting and for a different purpose. Rather than focusing on signal propagation, we examine how such initialisation affects the training dynamics of joint-embedding predictive models and their tendency to collapse. Closer still, mimetic initialisation \citep{trockmanMimeticInitializationSelfAttention2023} sets the query--key product of self-attention near the identity, and \citet{wenMechanismPredictionHead2022} analyse how a trainable, possibly identity-initialised prediction head shapes what non-contrastive training learns; we claim no architectural novelty for identity-biasing attention, and none of these works contains a per-mode stability ratio defined through the predictor at the JEPA fixed point, a supercriticality count, or a rank budget. We are closest in spirit to TD-JEPA \citep{bagatellaTDJEPALatentpredictiveRepresentations2025}, which also connects predictor initialisation to collapse prevention, but in a TD-RL setting; this work provides a more general approach derived from the new drive-decay framework.

\section{Drive vs. Decay: Theoretical Analysis of Coupled JEPA Dynamics}
\label{sec:theory}

We present a theoretical analysis of the Joint Embedding Predictive Architecture (JEPA) training dynamics. Unlike previous approaches that assume a static relationship between encoders, we model the system as a coupled dynamical system governed by the interaction between the gradient descent updates and the Exponential Moving Average (EMA) mechanism.

\subsection{Linear Model: Coupled Dynamical System Formulation}

Let the input data be represented by a vector $\vect{x} \in \mathbb{R}^{N}$, where $N$ is the input dimensionality.
The context encoder $f_{\theta}$ is represented by a weight matrix $\mat{W} \in \mathbb{R}^{d \times N}$.
The target encoder $f_{\bar{\theta}}$ is represented by $\mat{W}_a \in \mathbb{R}^{d \times N}$. The predictor $g_{\phi}$ is represented by $\mat{W}_p \in \mathbb{R}^{d \times d}$. 

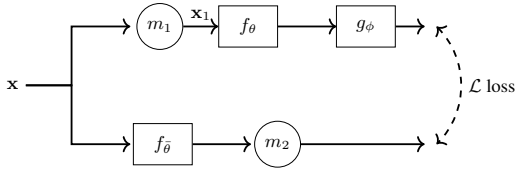
\begin{wrapfigure}[9]{r}{0.5\linewidth}
    \vspace{-8pt}
    \centering
    \resizebox{\linewidth}{!}{\begin{tikzpicture}[font=\small,scale=0.9,every node/.style={scale=0.9}]
	\node at (0, 0) (x) {$\mathbf{x}$};
	\draw[->, thick] (x) -- ++(1, 0) -- ++(0, 1) -- ++(1, 0);
	\draw[->, thick] (x) -- ++(1, 0) -- ++(0, -1) -- ++(1, 0);
    
	\node[draw, circle] at (2.5, 1) (m1) {$m_1$};
	\node at (3.2, 1.2) {$\mathbf{x}_1$};
    
	\node[draw, rectangle, minimum width=1cm, minimum height=0.7cm] at (4, 1) (f_theta) {$f_\theta$};
	\draw[->, thick] (m1) -- (f_theta);
	\draw[->, thick] (f_theta) -- ++(1.5, 0);
    
	\node[draw, rectangle, minimum width=1cm, minimum height=0.7cm] at (6, 1) (g_phi) {$g_\phi$};
	\draw[->, thick] (g_phi) -- ++(1, 0);
    
	\node[draw, rectangle, minimum width=1cm, minimum height=0.7cm] at (2.55, -1) (f_target) {$f_{\bar{\theta}}$};
	\draw[->, thick] (f_target) -- ++(1.5, 0);
    
	\node[draw, circle] at (4.5, -1) (m2) {$m_2$};
	\draw[->, thick] (m2) -- ++(2.5, 0);
    
	\draw[<->, dashed, thick] (7.2, 1) to[out=-45, in=45] node[right] {$\mathcal{L}$ loss} (7.2, -1);
\end{tikzpicture}}
    \vspace{-6pt}
    \caption{\textbf{Linear JEPA architecture.} }
    \label{fig:jepa_overview}
\end{wrapfigure}
Masking operations define the context and target views. A context mask yields $\vect{x}_1 \in \mathbb{R}^{N}$, which serves as input to the context encoder. This masking operation can be conceptualized as zeroing out elements of $\vect{x}$ by an element-wise multiplication with a random binary feature mask $\vect{m}_1$. The target encoder receives the original input $\vect{x}$ and a target feature mask $\vect{m}_2 \in \{0, 1\}^d$ is applied to the output of the target encoder, as presented in Figure \ref{fig:jepa_overview}. We further consider $\mat{M_2} = \text{diag}(\vect{m_2})$, with $m_2^i \sim \text{Bernoulli}(p_{tgt})$, where now we assume that the selected target features are masked for each row.

The JEPA objective is to minimize the expected squared Euclidean distance between the predicted representation and the masked target representation:
\begin{equation} \label{eq:loss_linear_expected}
\mathcal{L}(\mat{W}, \mat{W}_p; \mat{W}_a) = \frac{1}{2} \mathbb{E} \left[ \| \mat{W}_p \mat{W} \vect{x}_1 - \texttt{sg}(\mat{M}_2 \mat{W}_a \vect{x}) \|_2^2 \right]
\end{equation}
where \texttt{sg} is an operation that prevents the backpropagation of gradients through its argument. In the context of this loss function, it means that the term $\mat{M}_2 \mat{W}_a \vect{x}$ is treated as a fixed target or a constant when computing the gradients for updating the parameters $\mat{W}$ and $\mat{W}_p$. Note that $\mat{W}_a$ is typically updated offline, not directly by minimizing this loss.

The system evolves according to continuous-time dynamics where $\mat{W}$ and $\mat{W}_p$ follow the negative gradient of $\mathcal{L}$, while $\mat{W}_a$ follows the context weights via EMA. The Lemma \ref{lemma:coupled_dynamics} states their relationship (we refer the reader to Appendix \ref{app:detailed_derivation} for the details).

\begin{lemma}[Coupled JEPA Dynamics] \label{lemma:coupled_dynamics}
Let $\beta > 0$ represent the EMA update speed. The state evolution is governed by:
\begin{equation}
\dot{\mat{W}} = - \mat{W}_p^\top \mat{W}_p \mat{W} \mat{\Sigma}_{11} + \mat{W}_p^\top \bar{\mat{M}}_2 \mat{W}_a \mat{\Sigma}_{X1}, \qquad\quad \dot{\mat{W}}_a = \beta (\mat{W} - \mat{W}_a), \label{eq:dW_coupled}
\end{equation}
where $\mat{\Sigma}_{11} = \mathbb{E}[\vect{x}_1 \vect{x}_1^\top]$ and $\mat{\Sigma}_{X1} = \mathbb{E}[\vect{x} \vect{x}_1^\top]$ are data covariance matrices. For the stability analysis $\mat{W}_p$ is taken as given: the framework governs encoder-side escape given the predictor's state, and Lemma~\ref{lemma:joint} below says what that costs when the predictor is trained jointly.
\end{lemma}

\subsection{Jacobian Stability Analysis}

To analyze representational collapse, we examine the stability of the trivial fixed point $\mat{W} = \mat{0}, \mat{W}_a = \mat{0}$. We define the full state vector of the system as $\vect{z} = [ \vect{vec}(\mat{W})^\top, \vect{vec}(\mat{W}_a)^\top ]^\top \in \mathbb{R}^{2dN}$.

The linearized dynamics around the origin are given by $\dot{\vect{z}} = \mathcal{J}\vect{z}$, where $\mathcal{J}$ is the Jacobian matrix:

\vspace{-5pt}
\begin{equation} \label{eq:full_jacobian}
\mathcal{J} = 
\begin{bmatrix}
-(\mat{\Sigma}_{11} \otimes \mat{W}_p^\top \mat{W}_p) & (\mat{\Sigma}_{X1}^\top \otimes \mat{W}_p^\top \bar{\mat{M}}_2) \\
\beta \mat{I}_{dN} & -\beta \mat{I}_{dN}
\end{bmatrix}
\end{equation}

The spectral properties of $\mathcal{J}$ dictate the system's fate: if all eigenvalues satisfy $\text{Re}(\lambda) < 0$, the origin is a stable attractor, leading to model collapse. Conversely, learning requires an unstable origin ($\exists \lambda : \text{Re}(\lambda) > 0$) to enable the growth of non-trivial representations.

\begin{assumption}[Approximate spectral decoupling]
\label{ass:spectral_decoupling}
For analysis, we work in a basis where the \emph{data-side} covariances and the \emph{feature-side} predictor operators are (approximately) jointly diagonalizable. Concretely, we assume
$\mat{\Sigma}_{11}$ and $\mat{\Sigma}_{X1}$ share an eigenbasis, and $\mat{P}\coloneqq \mat{W}_p^\top\mat{W}_p$ and $\mat{Q}\coloneqq \mat{W}_p^\top\bar{\mat{M}}_2$ share an eigenbasis. 
This lets the coupled dynamics decompose into independent scalar modes.
\end{assumption}

The Assumption \ref{ass:spectral_decoupling} is mild in the regimes we study. On the \emph{data side}, when the context view $\vect{x}_1$ is produced by mild masking, the cross-covariance $\mat{\Sigma}_{X1}$ is typically close to a scaled version of $\mat{\Sigma}_{11}$; under common preprocessing (normalization), covariances are often near-diagonal. On the \emph{feature side}, $\mat{P}=\mat{W}_p^\top\mat{W}_p$ and $\mat{Q}=\mat{W}_p^\top\bar{\mat{M}}_2$ are close to commuting when $\bar{\mat{M}}_2$ is diagonal and $\mat{W}_p$ is near-orthogonal at initialization or kept so by regularization. 

By exploiting the block structure (see Appendix \ref{app:detailed_derivation} for step-by-step derivation), we find that:

\vspace{-15pt}
\begin{align}
 \gamma_i & \in \text{Spec}((\mat{\Sigma}_{X1}^\top \otimes \mat{W}_p^\top \bar{\mat{M}}_2) & \textbf{driving operator} \\ 
 \sigma_i & \in \text{Spec}(\mat{\Sigma}_{11} \otimes \mat{W}_p^\top \mat{W}_p) & \textbf{decay operator} \\  
 \mu_i & \coloneqq \frac{\gamma_i}{\sigma_i} & \textbf{stability variable}
\end{align}
\vspace{-10pt}

As a result, each spatial eigenmode $i$ evolves according to the characteristic equation $\lambda^2 + (\beta + \sigma_i)\lambda + \beta\sigma_i(1-\mu_i) = 0$ ,
which fully determines the local stability of the trivial solution along that mode. This form exposes two qualitatively distinct behaviors. First, when $\mu_i > 1$, the constant term becomes negative, yielding a positive real eigenvalue. The origin is then unstable and the corresponding mode exhibits exponential growth, corresponding to the learning regime. In general, the driving operator need not be symmetric, and its eigenvalues may be complex. A detailed discussion of these conditions and alternative real-valued formulations based on singular values is provided in Appendix~\ref{app:driving_spectrum_reality}. Appendix~\ref{app:nonsymmetric_stability} gives certificates without symmetry or decoupling: $\kappa(\mat{\Sigma}_{11})\,\kappa(\mat{W}_p^\top\mat{W}_p)\,\tilde\mu_{11} < 1$ forces collapse, a determinant condition forces escape, and under Assumption~\ref{ass:spectral_decoupling} both reduce to the per-mode criterion below.

\begin{tcolorbox}[
    colback=gray!10,
    colframe=gray!40,
    arc=2mm,
    boxrule=1pt,
]
JEPA dynamics are governed by the spectral ratio $\gamma/\sigma$ between drive and decay.
\begin{itemize}
    \item \textbf{Collapse Regime ($\mu \ll 1$):} Decay dominates. The trivial fixed point is a \textit{stable attractor}.
    \item \textbf{Learning Regime ($\mu \gg 1$):} Driving force dominates. The trivial solution becomes \textit{unstable}, facilitating the learning of non-trivial features.
\end{itemize}
\end{tcolorbox}

Under Assumption~\ref{ass:spectral_decoupling}, the ratio factorizes as
\begin{equation}
\mu_{k\ell} \;=\; \Big(\frac{c_k}{a_k}\Big)\cdot\Big(\frac{q_\ell}{p_\ell}\Big),
\label{eq:mu-factor}
\end{equation}
where $a_k$, $c_k$ are the eigenvalues of $\mat{\Sigma}_{11}$ and $\mat{\Sigma}_{X1}$ along data mode $k$, and $p_\ell$, $q_\ell$ are the eigenvalues of $\mat{W}_p^\top\mat{W}_p$ and $\mat{W}_p^\top\bar{\mat{M}}_2$ along predictor mode $\ell$. The two factors are independent: $c_k/a_k$ is a \emph{cross-view retention coefficient} governed by data structure and masking, while $q_\ell/p_\ell$ is a \emph{predictor alignment ratio} governed by the predictor geometry.

\subsection{Rank budget and joint training}
\label{sec:rank-control}\label{sec:joint}

The factorisation~\eqref{eq:mu-factor} promotes the scalar boundary $\mu{>}1$ to a count of unstable directions. Let $\mathcal U=\{(k,\ell):\mu_{k\ell}{>}1\}$. The proof and the loose-count corollary are in Appendix~\ref{app:rank-control-proof}.

\begin{corollary}[Super-criticality count bounds the post-escape rank]
\label{cor:rank-control}
Under the exact form of Assumption~\ref{ass:spectral_decoupling} with orthonormal eigenbases and real spectra (so that $\mat{\Sigma}_{X1}$ and $\mat{W}_p^\top\bar{\mat{M}}_2$ are symmetric) and every $a_k p_\ell > 0$, the linearised JEPA flow leaves the trivial fixed point along exactly the modes in $\mathcal U$ where the initial state has a non-zero unstable component, and the growing component $\mat W_+(t)$ of $\mat W(t) = \mat W_+(t) + \mat W_-(t)$ obeys
\begin{equation}
\rank\!\big(\mat W_+(t)\big)\;\le\; r_{\mathit{tight}}\;:=\;\min\!\big(|\pi_\ell\mathcal U|,\; |\pi_k\mathcal U|,\; d,\; N\big) \qquad \text{for all } t \ge 0,
\label{eq:r-tight}
\end{equation}
where $\pi_\ell\mathcal U=\{\ell:\exists k,(k,\ell)\in\mathcal U\}$ and similarly for $\pi_k$. If no mode has $\mu_{k\ell} = 1$ or is defective, then $\sigma_{r_{\mathit{tight}}+1}(\mat W(t)) \le \bar K Z_0 e^{-\eta t}$ with $\eta$ the slowest stable rate and $\bar K$, $Z_0$ computable at initialisation, so for any threshold $\varepsilon > 0$ at most $r_{\mathit{tight}}$ singular values of $\mat W(t)$ exceed $\varepsilon$ once $t \ge \eta^{-1}\log(\bar K Z_0/\varepsilon)$; and, in the escape regime ($\mathcal U \neq \emptyset$ with a non-zero component along some escaping mode), past an explicit initial time $\Erank(\mat W(t)) \le r_{\mathit{tight}} + \delta(t)$ with $\delta(t) \to 0$ exponentially.
\end{corollary}

The bound concerns the growing component, not the generically full algebraic rank of $\mat W(t)$; in the linear model it passes to the effective rank of the embeddings $\mat Z = \mat W \mat X$, the quantity every experiment tracks (Appendix~\ref{app:embedding-transfer}).

A JEPA trains the predictor with the encoder, at one learning rate in this paper's recipes; near collapse the fixed-predictor analysis also describes that system.

\begin{lemma}[Exact local criterion under joint training] \label{lemma:joint}
The predictor gradient is quadratic in the encoder state and vanishes at $\mat W = \mat W_a = \mat 0$, so the collapsed set is a manifold of equilibria of the full $(\mat W, \mat W_a, \mat W_p)$ gradient flow, whose Jacobian there is Eq.~\eqref{eq:full_jacobian} on top of a zero block at any ratio of learning rates. Under the conditions of Appendix~\ref{app:joint}, the predictor moves by at most $C\delta^2$ while the encoder state grows to norm $\delta$, and if every mode with real roots has margin $|\mu_{k\ell} - 1| > C'\delta^2$ the Jacobian keeps exactly $|\mathcal U|$ unstable eigenvalues throughout, each within $O(\delta^2)$ of an escape rate computed at initialisation.
\end{lemma}

Quasi-staticity of the predictor near the collapsed set is therefore a consequence of the initialisation scale rather than of a learning-rate choice (Appendix~\ref{app:joint}). How the predictor adapts after escape is measured in Section~\ref{sec:residual_predictor:scale}.

\section{Methodology}
\label{sec:methodology}

In order to validate our theoretical results, we start by evaluating the linear Tabular-JEPA on real world tabular benchmarks to verify our assumptions and explain the expected result. Then, we proceed by evaluating the standard JEPA models on both tabular and image datasets, verifying the generalization of our approach. Code for all experiments is available at \url{https://github.com/jose-melo/drive-vs-decay}.

\paragraph{Datasets and Training Protocol.} Eight datasets were used: three tabular datasets (\texttt{aloi}, \texttt{helena}, \texttt{jannis}) and five image datasets (\texttt{CIFAR-10}, \texttt{CIFAR-100}, \texttt{STL-10}, \texttt{ImageNet-100}, and \texttt{ImageNet-1k}). The tabular datasets were stratified into $80\%$ train, $10\%$ validation, and $10\%$ test splits; continuous features are scaled with min--max normalization, categorical columns are label-encoded. Image datasets followed the training protocol of \citet{assranSelfSupervisedLearningImages2023}, except that the ImageNet runs use a matched-budget protocol, at $128$ pixels unless stated (Appendix~\ref{app:image_scale}). We refer the reader to Appendix~\ref{app:training_details} for details. For the original T-JEPA and I-JEPA models, the training protocol followed the recommendations for each model. Considering the Tabular-JEPA variants, the optimization uses AdamW with learning rate $10^{-3}$, weight decay $10^{-5}$, cosine annealing to $10^{-5}$ over $40$ epochs, and gradient clipping at norm $1.0$. Hyperparameters were found using Bayesian optimization search.

\paragraph{Models.}
For the evaluation, we considered several variants of Joint-Embedding Predictive Architectures (JEPAs). We include \textbf{T-JEPA} \cite{thimonierTJEPAAugmentationFreeSelfSupervised2024}, which adapts the JEPA framework to tabular data. In addition, we evaluate a family of \textbf{Tabular-JEPA variants} inspired by T-JEPA but using linear encoders and a scaled predictor. These variants are specifically introduced to isolate and demonstrate the different stabilization and collapse mechanisms studied under the proposed drive-decay framework with a simpler architecture. For image experiments, we use \textbf{I-JEPA} \cite{assranSelfSupervisedLearningImages2023} with a ViT-S/4 encoder on CIFAR-10/100, ViT-S/8 on STL-10 and ViT-S/16 on ImageNet. Downstream accuracy is evaluated using a linear probe; the ImageNet runs add a parameter-free nearest-neighbour classifier and an attentive probe, with exact permutation tests over seeds (Appendix~\ref{app:image_scale}).

\paragraph{Evaluation Metrics} We evaluate the spectral diversity of learned embeddings with the effective rank to directly quantify whether representations concentrate in a low-dimensional subspace. Let $\mat{Z} \in \mathbb{R}^{n \times d}$ be a matrix of embeddings with singular values $\{\sigma_i\}_{i=1}^r$. Defining the normalized spectrum $p_i \;=\; \frac{\sigma_i}{\sum_{j=1}^r \sigma_j}$, we measure intrinsic dimensionality via the \emph{effective rank} (we refer the reader to the Appendix \ref{app:evaluation_metrics} for more details): $\mathrm{ERank}(\mat{Z})
\;=\;
\exp\!\left(
- \sum_{i=1}^r p_i \log p_i
\right)$.

\section{Controlling the Training Dynamics: Example of Mechanisms}
\label{sec:practical_control}

The theoretical condition for preventing collapse depends on the spectral interplay between the predictor geometry and the data statistics. Based on the definition of $\mu_i = \gamma_i / \sigma_i$, we identify three primary mechanisms that allow practical control over the training regime.

\paragraph{Mechanism 1: Predictor scaling (Inverse proportionality)}
\begin{wrapfigure}[18]{r}{0.45\linewidth}
\vspace{-1.6\baselineskip}
\centering
\includegraphics[width=\linewidth]{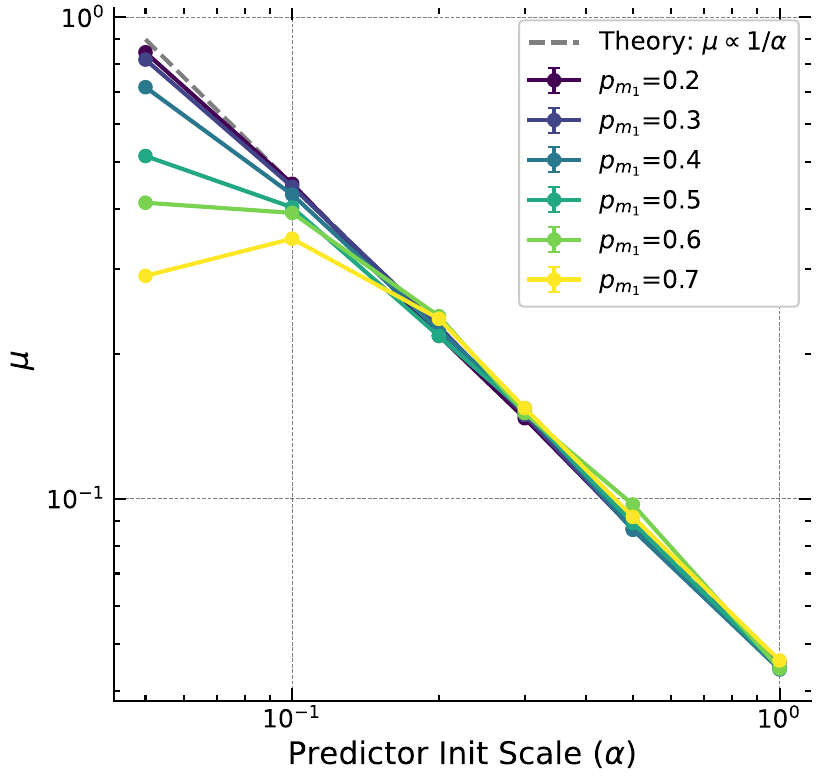}
\caption{\small\textbf{Inverse $\mu(\alpha)$ scaling.} Empirical $\mu$ vs.\ predictor scale $\alpha$ for $p_{m_1}\!\in\!\{0.2,\dots,0.7\}$.}
\label{fig:mu_predict_scale}
\vspace{-0.6\baselineskip}
\end{wrapfigure}
The most direct method to assure $\mu > 1$ is via the initialization scale of the predictor weights $\mat{W}_p$. Consider the scaling of the decay and driving operators with respect to a scalar scaling factor $\alpha$ applied to $\mat{W}_p$ (i.e., $\mat{W}_p \leftarrow \alpha \mat{W}_p$), the parameter $\mu$ scales inversely with the predictor norm:
\begin{equation}
\mu(\alpha) = \frac{\gamma(\alpha)}{\sigma(\alpha)} \propto \frac{\alpha}{\alpha^2} = \frac{1}{\alpha}
\end{equation}
To guarantee the supercritical regime ($\mu \gg 1$), one must initialize the predictor with a sufficiently small variance. Large initialization values increase the quadratic penalty (self-regularization) faster than the linear driving force, forcing $\mu \ll 1$ and causing collapse.

Figure~\ref{fig:mu_predict_scale} provides a direct empirical validation of the theoretical scaling. For small and moderate context masking ratios ($p_{m_1}\leq 0.4$), the measured $\mu$ closely tracks the dashed theoretical curve across more than an order of magnitude in $\alpha$, confirming the dominant role of predictor-induced contraction in setting the stability margin. As the masking ratio increases ($p_{m_1}\geq 0.5$), the empirical curves systematically shift downward and deviate from the ideal inverse law. This effect is expected from the theory, as presented next.

\paragraph{Mechanism 2: View Correlation (Masking Strategy)}

The second critical control lever for $\mu$ is the correlation structure between the context view $\vect{x}_1$ and the target input $\vect{x}$, governed by the context ratio $p_{m_1} \in [0, 1)$. Recall that the driving force $\gamma_i$ is derived from the spectrum of the cross-covariance $\mat{\Sigma}_{X1}$, while the decay force $\sigma_i$ is driven by the context auto-covariance $\mat{\Sigma}_{11}$.

When $\vect{x}_1$ is formed by feature-wise masking $\vect{x}_1 = \vect{m}_1 \odot \vect{x}$ with keep probability $p_{m_1} \coloneqq \mathbb{P}(m_{1j}=1)$, both operators vary with $p_{m_1}$. In expectation, the cross-covariance scales as a single factor while the context auto-covariance does not,
\begin{equation}
\mat{\Sigma}_{X1} = \mathbb{E}[\vect{x}(\vect{m}_1\odot \vect{x})^\top]
\approx p_{m_1}\,\mat{\Sigma}_{XX},
\qquad
\mat{\Sigma}_{11}=\mathbb{E}[(\vect{m}_1\odot \vect{x})(\vect{m}_1\odot \vect{x})^\top],
\end{equation}
since diagonal entries scale as $p_{m_1}$ but off-diagonal entries scale as $p_{m_1}^2$ (since $m_{1i}m_{1j}$ appears for $i\neq j$). Therefore, the spectral ratio $\mu_i=\gamma_i/\sigma_i$ depends on how the data covariance is structured.

\paragraph{Mechanism 3: EMA as a timescale}
The EMA decay rate $\beta$ is often viewed as a heuristic stabilizer. Our dynamical analysis provides a precise spectral interpretation. In the characteristic equation, $\beta$ multiplies the constant term without changing its sign, so for a real drive spectrum it cannot decide whether a mode escapes; it acts as a \emph{timescale separator}: for $\mu_i \ll 1$, smaller $\beta$ slows collapse, while for $\mu_i \gg 1$, it throttles feature growth. Prior analyses, notably \cite{tianUnderstandingSelfsupervisedLearning2021}, often assume that the target weights $\mat{W}_a$ remain proportional to the encoder weights $\mat{W}$. That proportional-target assumption is the $\beta\to\infty$ limit of the coupled analysis, in which $\mat{W}_a$ tracks $\mat{W}$ almost instantaneously; the resulting adiabatic stability condition, its critical coupling $\tau_{\text{crit}}$, the bifurcation it produces and the region of better downstream performance it delimits are derived in Appendix~\ref{app:adiabatic}.

\begin{figure*}[t]
    \centering
    \includegraphics[width=\linewidth]{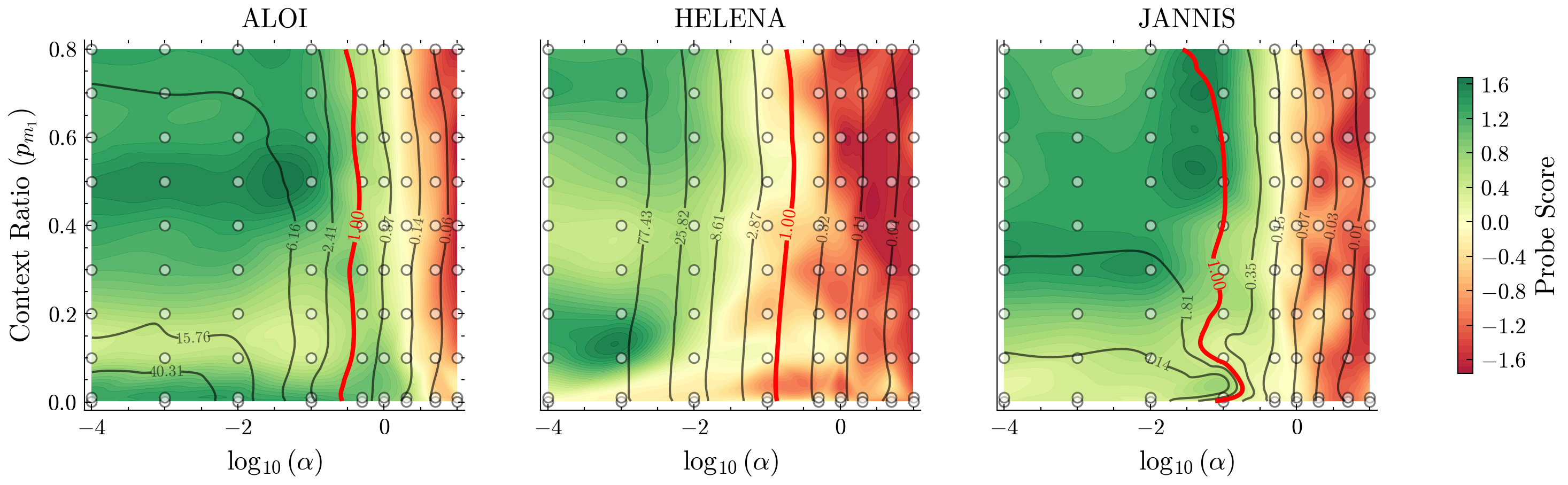}
    \caption{\small\textbf{Empirical phase diagram of Tabular JEPA stability predicted by the drive-decay framework.}
    Each panel reports a two–dimensional sweep over the \emph{predictor scale} $\log_{10}(\alpha)$  and the \emph{context masking ratio} $p_{m_1}$, for three tabular datasets.
    Colored heatmaps show the downstream normalized \emph{linear probe score}. The thick red curve highlights the critical boundary $\mu=1$, which our theory identifies as the transition between a \emph{learning} regime  and a \emph{collapsing} regime.}
    \label{fig:contours_phase_diagram}
    \vspace{-6pt}
\end{figure*}

\begin{table}[htb]
\centering
\small
\setlength{\tabcolsep}{3.5pt}
\renewcommand{\arraystretch}{1.02}

\caption{Aggregated results for Tabular--JEPA variants (aggregated over ALOI, Helena, and Jannis) and T-JEPA (Jannis). For Tabular--JEPA, only test accuracy is normalized within each dataset before aggregation. Metrics are reported as mean values over 5 runs.}
\label{tab:jepa_grid_results_merged}

\resizebox{\linewidth}{!}{%
\begin{tabular}{llc||cccc cccc cc}
\toprule
\multirow{2}{*}{\textbf{Model}} &
\multirow{2}{*}{\textbf{Context ratio}} &
\multirow{2}{*}{\textbf{Predictor scale}} &
\multicolumn{4}{c}{\textbf{Epoch 0}} &
\multicolumn{4}{c}{\textbf{Last epoch}} &
\multirow{2}{*}{\textbf{Final loss}} &
\multirow{2}{*}{\textbf{Test Acc.}} \\
\cmidrule(lr){4-7}\cmidrule(lr){8-11}
& & &
$\mu_{25}$ & $\mu_{50}$ & $\mu_{75}$ & $\mathrm{Erank}$ &
$\mu_{25}$ & $\mu_{50}$ & $\mu_{75}$ & $\mathrm{Erank}$ &
& \\
\midrule

\multirow{9}{*}{Tabular--JEPA variants}
& low    & \multirow{3}{*}{low}    & $1.49$ & $2.46$ & $3.92$ & $11.05$ & $0.76$ & \cellcolor{myblue!33}$2.71$ & $7.91$ & \cellcolor{mygreen!13}$4.56$ & $0.0180$ & \cellcolor{myorange!20}$0.509$ \\
& medium &                         & $1.54$ & $2.71$ & $4.17$ & $12.69$ & $0.58$ & \cellcolor{myblue!20}$1.64$ & $4.62$ & \cellcolor{mygreen!61}$14.49$ & $0.1798$ & \cellcolor{myorange!51}$0.737$ \\
& high   &                         & $0.96$ & $1.76$ & $3.17$ & $16.18$ & $0.42$ & \cellcolor{myblue!20}$1.63$ & $7.39$ & \cellcolor{mygreen!70}$16.26$ & $0.0984$ & \cellcolor{myorange!70}$0.872$ \\
\cmidrule(lr){2-12}
& low    & \multirow{3}{*}{medium} & $0.53$ & $0.88$ & $1.82$ & $9.93$ & $0.41$ & \cellcolor{myblue!10}$0.78$ & $4.44$ & \cellcolor{mygreen!10}$3.92$ & $0.0149$ & \cellcolor{myorange!10}$0.433$ \\
& medium &                         & $0.67$ & $1.14$ & $1.93$ & $9.97$ & $0.42$ & \cellcolor{myblue!14}$1.12$ & $7.46$ & \cellcolor{mygreen!19}$5.92$ & $0.0248$ & \cellcolor{myorange!33}$0.603$ \\
& high   &                         & $0.50$ & $0.96$ & $1.84$ & $10.63$ & $0.49$ & \cellcolor{myblue!50}$4.06$ & $36.63$ & \cellcolor{mygreen!60}$14.30$ & $0.0690$ & \cellcolor{myorange!58}$0.789$ \\
\cmidrule(lr){2-12}
& low    & \multirow{3}{*}{high}   & $0.26$ & $0.46$ & $1.12$ & $8.58$ & $0.19$ & \cellcolor{myblue!12}$0.96$ & $6.26$ & \cellcolor{mygreen!10}$4.06$ & $0.0173$ & \cellcolor{myorange!14}$0.465$ \\
& medium &                         & $0.28$ & $0.54$ & $1.15$ & $8.21$ & $0.21$ & \cellcolor{myblue!21}$1.67$ & $7.80$ & \cellcolor{mygreen!13}$4.64$ & $0.0173$ & \cellcolor{myorange!20}$0.507$ \\
& high   &                         & $0.29$ & $0.53$ & $1.05$ & $8.24$ & $0.33$ & \cellcolor{myblue!70}$5.62$ & $37.08$ & \cellcolor{mygreen!19}$5.92$ & $0.0307$ & \cellcolor{myorange!30}$0.586$ \\
\midrule

\multirow{9}{*}{T-JEPA}
& low    & \multirow{3}{*}{low}    & $0.09$ & $0.21$ & $0.55$ & $7.67$ & $0.09$ & \cellcolor{mypurple!20}$0.20$ & $0.53$ & \cellcolor{myteal!42}$17.10$ & $0.2556$ & \cellcolor{myred!57}$0.673$ \\
& medium &                         & $0.20$ & $0.46$ & $1.37$ & $6.94$ & $0.19$ & \cellcolor{mypurple!35}$0.44$ & $1.26$ & \cellcolor{myteal!64}$25.27$ & $0.1691$ & \cellcolor{myred!70}$0.729$ \\
& high   &                         & $0.39$ & $0.89$ & $2.40$ & $7.39$ & $0.40$ & \cellcolor{mypurple!70}$0.95$ & $2.56$ & \cellcolor{myteal!70}$27.48$ & $0.1266$ & \cellcolor{myred!69}$0.725$ \\
\cmidrule(lr){2-12}
& low    & \multirow{3}{*}{medium} & $0.05$ & $0.11$ & $0.32$ & $6.89$ & $0.05$ & \cellcolor{mypurple!13}$0.10$ & $0.29$ & \cellcolor{myteal!10}$4.65$ & $5.07e-07$ & \cellcolor{myred!10}$0.464$ \\
& medium &                         & $0.09$ & $0.20$ & $0.49$ & $7.35$ & $0.11$ & \cellcolor{mypurple!23}$0.25$ & $0.62$ & \cellcolor{myteal!10}$4.53$ & $6.32e-07$ & \cellcolor{myred!10}$0.464$ \\
& high   &                         & $0.22$ & $0.49$ & $1.44$ & $6.90$ & $0.25$ & \cellcolor{mypurple!45}$0.58$ & $1.61$ & \cellcolor{myteal!11}$4.94$ & $1.84e-04$ & \cellcolor{myred!10}$0.464$ \\
\cmidrule(lr){2-12}
& low    & \multirow{3}{*}{high}   & $0.02$ & $0.05$ & $0.12$ & $7.30$ & $0.02$ & \cellcolor{mypurple!10}$0.05$ & $0.11$ & \cellcolor{myteal!10}$4.90$ & $6.52e-06$ & \cellcolor{myred!10}$0.464$ \\
& medium &                         & $0.05$ & $0.12$ & $0.36$ & $6.65$ & $0.05$ & \cellcolor{mypurple!14}$0.12$ & $0.36$ & \cellcolor{myteal!10}$4.67$ & $5.54e-07$ & \cellcolor{myred!10}$0.464$ \\
\bottomrule
\end{tabular}%
}
\end{table}

\paragraph{Gathering Mechanisms: Empirical Phase Transitions.}
Figure~\ref{fig:contours_phase_diagram} illustrates a phase transition linking downstream performance to the $\mu$ stability coefficient. This relationship is further demonstrated by the results in Table~\ref{tab:jepa_grid_results_merged}, which reports mean values over five random seeds, aggregated by context-ratio and predictor-scale subgroups. 

Moving along the predictor scale axis, we observe a sharp degradation in test accuracy as the scale $\alpha$ increases.  For instance, in the \textit{Tabular-JEPA} variants with high context ratios, increasing the predictor scale results in a drop in normalized test accuracy from $0.872$ to $0.586$. This aligns with the theoretical scaling $\mu(\alpha) \propto 1/\alpha$. The role of the context masking ratio $p_{m_1}$ is also critical. Low context ratios, corresponding to higher masking, reduce the correlation between the context and target views, weakening the cross-covariance contribution to the growth factor $\gamma$ and subsequently lowering $\mu$. Table~\ref{tab:jepa_grid_results_merged} shows that for \textit{T-JEPA}, configurations with low context ratios and medium-to-high predictor scales exhibit median $\mu$ values near zero, $\mu_{50} \in [0.05, 0.11]$, and collapsed rank, $\mathrm{Erank} \approx 4.6$, leading to a failure to learn meaningful representations, with Test Acc. $\approx 0.464$. We refer the reader to Appendix~\ref{app:full_results} for a detailed analysis.

\paragraph{Empirical validation of the rank-control corollary.}
\begin{wrapfigure}{r}{0.42\linewidth}
\vspace{-10pt}
\centering
\includegraphics[width=\linewidth]{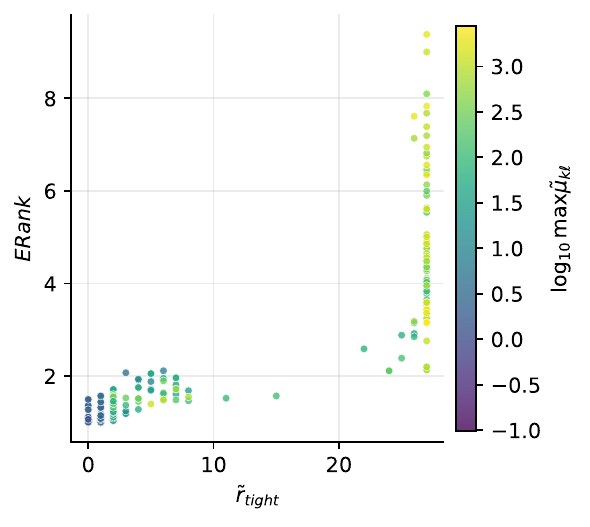}
\vspace{-16pt}
\caption{\small $\Erank$ vs $\tilde r_{\mathit{tight}}$ on \texttt{helena} (one point per run); colour is $\log_{10}\max_{k\ell}\tilde\mu_{k\ell}$.}
\label{fig:rank-control-scatter}
\vspace{6pt}
\small
\setlength{\tabcolsep}{4pt}
\renewcommand{\arraystretch}{1.05}
\begin{tabular}{lc}
\toprule
\textbf{Dataset} & $\rho(\Erank,\tilde r_{\mathit{tight}})$\\
\midrule
\texttt{helena}  & $+0.906$\\
\texttt{aloi}    & $+0.905$\\
\texttt{jannis}  & $+0.619$\\
\bottomrule
\end{tabular}
\captionof{table}{\small \textbf{Correlation per dataset.} The weaker \texttt{jannis} value comes with a near-degenerate top spectrum (Appendix~\ref{app:assumptions:tabular}).}
\label{tab:rank-control}
\vspace{-18pt}
\end{wrapfigure}
The metric of interest is $\tilde r_{\mathit{tight}}$ (Corollary~\ref{cor:rank-control}), computed via the singular-value form $\tilde\mu_{k\ell}$. For more details on the limits of the Assumption \ref{ass:spectral_decoupling} and effects of singular-value calculations, we refer the reader to Appendix~\ref{app:assumptions} and \ref{app:rank-control-signed-vs-abs}. The count of super-critical modes is converted into a discrete rank budget. Across $810$ runs, this budget tracks the representation ERank with $\rho={+}0.91$ on \texttt{helena} (Figure~\ref{fig:rank-control-scatter}).

The rank-budget view refines the scalar criterion $\mu>1$: rather than only predicting escape from collapse, it estimates how many rank-one directions are available after escape. Table~\ref{tab:rank-control} shows that the prediction is not specific to a single dataset: the correlation between $\tilde r_{\mathit{tight}}$ and the observed effective rank exceeds $0.9$ on \texttt{helena} and \texttt{aloi}, and remains positive, though weaker, on \texttt{jannis}. The measured data-side commutator on \texttt{jannis} is no larger than on \texttt{helena}, which is predicted well; what distinguishes \texttt{jannis} is a near-degenerate top spectrum, where mode identities mix under a small violation while the count survives, the signature that Appendix~\ref{app:assumptions:tabular} measures. Together with the $\mu_{50}$ phase diagnostics and Figure~\ref{fig:contours_phase_diagram}, these results support the role of $\mu$ both as a stability threshold and as a quantitative proxy for the number of active rank-one learning directions.

\paragraph{Other mechanisms as doses on $\mu$.}
The three mechanisms above are not special cases. Appendix~\ref{app:mechanisms} takes seven anti-collapse devices from the literature and writes each as a term of the linear model whose effect on $\mu$ is derived; Table~\ref{tab:unification-main} shows the five beyond predictor scaling and EMA: encoder weight decay raises the decay and the loss-side regularisers SIGReg and VICReg lower its predictor factor, mimetic initialisation moves the predictor factors, and a frozen identity predictor fixes them, so that escape becomes a condition on the data spectrum. For each mechanism the number of modes that escape is predicted from the initialisation, and we verified the prediction numerically by integrating the linear model with the mechanism added.

\begin{table}[t]
\centering\small
\setlength{\tabcolsep}{3pt}
\renewcommand{\arraystretch}{1.0}
\caption{\textbf{Five further anti-collapse mechanisms as doses on $\mu$}, derived in Appendix~\ref{app:mechanisms}. $p$ is the context keep probability and $\{s_k\}$ the spectrum of the unit-diagonal data covariance.}
\label{tab:unification-main}
\begin{tabular}{@{}lll@{}}
\toprule
\textbf{Mechanism} & \textbf{Derived dose} & \textbf{Mode escapes iff} \\
\midrule
Weight decay & additive floor: $\sigma_i + \eta_{wd}$ & $\mu_i > 1 + \eta_{wd}/\sigma_i$ \\
SIGReg & lowers the predictor factor: $p_\ell - \rho$ & $\mu > 1 - \rho/p_\ell$ \\
VICReg hinge & lowers it by $\tilde\lambda$, its small-weight dose & $\mu > 1 - \tilde\lambda/p_\ell$ \\
Mimetic initialisation & $q_\ell/p_\ell$ dosed once, washed out at rate $\zeta$ & integrated growth clears $\log 10^3$ \\
Identity predictor & $\mat P, \mat Q$ fixed; a data-spectrum condition & $s_k > (1-p)/(p_{tgt}-p)$, $p_{tgt} > p$ \\
\bottomrule
\end{tabular}
\end{table}

\section{\texttt{ResidualPred}: The Residual Transformer Predictor}
\label{sec:residual_predictor}

Large predictor weights can make decay dominate drive ($\mu \ll 1$). We avoid this by parametrizing the predictor as a residual correction to the identity, so early dynamics are driven by feature alignment rather than arbitrary initialization scale.

\subsection{Theoretical Motivation}

\begin{table*}[t]
\centering\footnotesize
\setlength{\tabcolsep}{2.4pt}
\renewcommand{\arraystretch}{1.08}
\caption{\textbf{ResidualPred across tabular and image datasets.} Linear-probe accuracy (\%). More ImageNet results: Appendix~\ref{app:image_scale}, Figure~\ref{fig:scale}.}
\label{tab:respred-tabular-image}
\begin{tabular}{@{}lcc|lcccc@{}}
\toprule
\multicolumn{3}{c|}{\textbf{Tabular brittle regime}} & \multicolumn{5}{c}{\textbf{Image}} \\
\cmidrule(lr){1-3}\cmidrule(lr){4-8}
\textbf{Variant} & \textbf{Jannis} & \textbf{Helena} & \textbf{Variant} & \textbf{CIFAR-10} & \textbf{CIFAR-100} & \textbf{STL-10} & \textbf{IN-100} \\
\midrule
T-JEPA & $43.4 \pm 1.7$ & $\phantom{0}8.6 \pm 1.1$ & I-JEPA & $64.55 \pm 1.71$ & $36.83 \pm 1.49$ & $71.25 \pm 1.14$ & $37.09 \pm 0.68$ \\
\rowcolor{rowgray}
+\,\code{ResPred} & $\mathbf{46.9 \pm 3.5}$ & $\mathbf{21.3 \pm 3.0}$ & +\,\code{ResPred} & $\mathbf{68.61 \pm 0.59}$ & $\mathbf{40.85 \pm 0.66}$ & $\mathbf{75.31 \pm 0.78}$ & $\mathbf{42.73 \pm 0.41}$ \\
\bottomrule
\end{tabular}
\end{table*}

We parametrize the predictor as a residual map, $\mat{W}_p \;=\; \mat{I} + \Delta \mat{W}_p, \Delta \mat{W}_p \approx \mat{0}\ \text{at init.}$
Hence the predictor Gram matrix satisfies $\mat{W}_p^\top \mat{W}_p = \mat{I} + \mathcal{O}(\|\Delta \mat{W}_p\|)$, preventing an arbitrary predictor-induced contraction at initialization. In our spectral characterization, with $\mat{W}_p \approx \mat{I}$, the initialization regime becomes \emph{view-limited}: $\sigma_i$ is set by the masked-view energy $\mat{\Sigma}_{11}$ and $\gamma_i$ by the view alignment $\mat{\Sigma}_{X1}$, rather than by predictor scale. To extend this identity-preserving strategy to the general transformer architecture \cite{vaswaniAttentionAllYou2017}, we must ensure the predictor block acts as an identity map in \textit{token space}. Standard self-attention initialization typically yields uniform maps ($\mat{A} \approx \frac{1}{N}\mat{1}\mat{1}^\top$), which averages out spatial signals and collapses $\mu$.
We propose \texttt{ResidualPred}, which enforces $\mat{A} \approx \mat{I}$ via a Diagonal Attention Bias. To prevent the internal attention mechanism from diffusing gradients uniformly, we add a diagonal logit bias: $\text{Attention}(\mat{Q}, \mat{K}, \mat{V}) = \text{Softmax}\left( \frac{\mat{Q}\mat{K}^\top}{\sqrt{d}} + \beta_{init} \cdot \mat{I}_{token} \right) \mat{V}$
where $\beta_{init} \gg 1$ (e.g., $10.0$, see Appendix~\ref{app:beta_choice}). This forces the softmax output to approximate the identity matrix, ensuring that token $i$ attends strictly to itself at the start of training.

\begin{figure}[t]
  \centering
  \begin{subfigure}[b]{0.32\linewidth}
    \centering
    \footnotesize
    \setlength{\tabcolsep}{1.8pt}
    \renewcommand{\arraystretch}{1.10}
    \begin{tabular}{lcccc}
    \toprule
    & \multicolumn{2}{c}{\textbf{T-JEPA}} & \multicolumn{2}{c}{\textbf{+\,\code{ResPred}}} \\
    \cmidrule(lr){2-3}\cmidrule(lr){4-5}
    \textbf{Metric} & no & yes & no & yes \\
    \midrule
    ERank$_{0}$       & $7.1$  & $6.8$  & $6.8$  & $6.9$  \\
    ERank             & $4.2$  & $4.3$  & $\mathbf{12.5}$ & $\mathbf{7.2}$ \\
    Loss              & $0.000$ & $0.048$ & $0.457$ & $0.029$ \\
    \rowcolor{rowgray}
    Test (\%)         & $28.2$ & $27.2$ & $\mathbf{43.3}$ & $\mathbf{35.1}$ \\
    \bottomrule
    \end{tabular}
    \caption{\small
Aggregated training diagnostics for T-JEPA on \texttt{jannis} and \texttt{helena}, averaged over context ratios and predictor scales. Columns separate non-brittle (\textbf{no}) and brittle (\textbf{yes}). ERank$_{0}$ is the ERank at epoch 0.
}
    \label{tab:respred-brittle-aggregate}
  \end{subfigure}\hfill
  \begin{subfigure}[b]{0.32\linewidth}
    \centering
    \includegraphics[width=\linewidth]{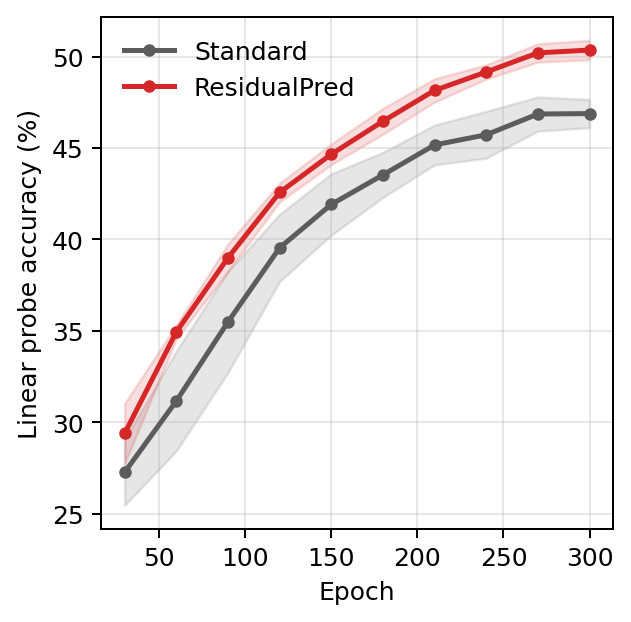}
    \caption{\small I-JEPA on CIFAR-100, 300 epochs: linear-probe accuracy.}
    \label{fig:cifar100_300ep_residualpred_vs_standard}
  \end{subfigure}\hfill
  \begin{subfigure}[b]{0.32\linewidth}
    \centering
    \includegraphics[width=\linewidth]{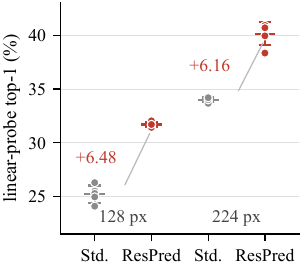}
    \caption{\small ImageNet-1k pilot at a matched budget: final linear-probe top-1 per seed.}
    \label{fig:scale}
  \end{subfigure}
  \caption{\small\textbf{\code{ResidualPred} -- aggregated brittle-vs-non-brittle diagnostics, image-domain training trajectory, and the ImageNet-1k pilot.} Panel (a) carries the general tabular claim: the predictors start matched and Standard finishes below its own initialisation in both regimes. Panel (c): ViT-S/16, predictor depth $6$, $30$ epochs, five seeds each, mean $\pm$ s.d.; every \code{ResidualPred} seed lies above every standard seed at both resolutions.}
  \label{fig:respred_3block}
\vspace{-4pt}
\end{figure}

\subsection{\texttt{ResidualPred} Results}
\label{sec:residual_predictor:scale}

\code{ResidualPred} improves accuracy in every setting of Table~\ref{tab:respred-tabular-image}, most on Helena. In brittle ($\mu<1$) and non-brittle configurations, Standard T-JEPA collapses while \code{ResidualPred} increases rank (Figure~\ref{fig:respred_3block}(a)), and on \texttt{helena} keeps it above Standard's throughout training (Figure~\ref{fig:erank_trajectories_helena_avg}). 

\begin{wraptable}{r}{0.47\linewidth}
\vspace{-12pt}
\centering\footnotesize
\setlength{\tabcolsep}{3pt}
\renewcommand{\arraystretch}{1.05}
\begin{tabular}{@{}lcc@{}}
\toprule
& \textbf{Linear (\%)} & \textbf{Effective rank} \\
\midrule
I-JEPA & $25.22 \pm 0.80$ & $107.2 \pm 7.7$ \\
+\,SIGReg & $29.19 \pm 0.54$ & $185.6 \pm 5.6$ \\
+\,\code{ResPred} & $31.69 \pm 0.22$ & $188.5 \pm 3.3$ \\
+\,\code{ResPred}\,+\,SIGReg & $\mathbf{32.23 \pm 0.54}$ & $\mathbf{206.7 \pm 1.0}$ \\
\bottomrule
\end{tabular}
\caption{\small\textbf{Against SIGReg on the ImageNet-1k pilot.} Mean over five seeds.}
\label{tab:loss-side}
\vspace{-8pt}
\end{wraptable}
Final training loss is therefore a poor diagnostic: near-zero MSE can accompany degenerate embeddings. On CIFAR-100, \code{ResidualPred} leads throughout 300 epochs (Figure~\ref{fig:cifar100_300ep_residualpred_vs_standard}). Its gain holds on ImageNet-100 at predictor depth $2$ and on ViT-B/16, and on a thousand-class ImageNet-1k pilot at a matched budget, at $128$ and $224$ pixels (Figure~\ref{fig:scale}); tripling the schedule to $90$ epochs attenuates the margin (Appendix~\ref{app:image_scale}).

The criterion is proven in the linear regime, while its extension to images is an empirical prediction whose mediating variable is measured after predictor equilibration, not at initialisation. At initialisation, a whole-block estimator does not separate the predictors; the target-aligned drive after a few predictor-only steps is higher for \code{ResidualPred} in $9$ of $9$ matched comparisons, evidence for the first link of the mechanism only (Appendix~\ref{app:equilibration}). Frozen at initialisation and denied learning, the \code{ResidualPred} predictor collapses the representation and loses the entire gain (Appendix~\ref{app:image_scale:controls}). Effective rank on images is necessary for accuracy without being sufficient.

\code{ResidualPred} beats SIGReg \citep{balestrieroLeJEPAProvableScalable2025} and the two compose (Table~\ref{tab:loss-side}, Appendix~\ref{app:image_scale:comparison}). Mimetic initialisation \citep{trockmanMimeticInitializationSelfAttention2023} gains $1.6$ points with seed overlap and trails \code{ResidualPred} by $4.9$. \code{ResidualPred} has one hyperparameter, unchanged everywhere, and no loss term or per-step cost.

\section{Discussion}
\label{sec:discussion}

These results support the drive--decay view of JEPA stability: representation learning is governed by a spectral competition between an alignment-driven force $\gamma$ and a predictor-induced decay force $\sigma$. This single mechanism explains why collapse can occur despite loss minimization and unifies several empirical anti-collapse heuristics. The phase diagrams show a sharp transition aligned with the theoretical boundary: when $\mu \ll 1$, representations collapse to low rank and downstream performance degrades. Viewed through this lens, existing methods in the literature can be interpreted as indirect ways of manipulating $\mu$. For example, explicit regularization terms that enforce embedding diversity \citep{balestrieroLeJEPAProvableScalable2025} effectively modify the Jacobian spectrum, shifting eigenvalues away from the stable collapsed regime and variance--covariance regularisers \citep{bardesVICRegVarianceInvarianceCovarianceRegularization2022} act as an explicit modification of the decay operator $\sigma$ in our framework. \texttt{ResidualPred} provides a structural alternative: by initializing the predictor near the identity, it lets view alignment dominate early training, allowing informative modes to grow before nonlinear effects take over.

Table~\ref{tab:unification-main} makes that reinterpretation precise for seven published mechanisms, each a derived dose on $\mu$ whose predicted escaping count the linear model reproduces. DirectPred, IsoLoss, the rate equation of \citet{halvagalImplicitVarianceRegularization2023} and predictor orthonormalisation are recovered as special cases.

\paragraph{Limitations.}
Our theory is derived in a linear setting and relies on the spectral-decoupling approximation of Assumption~\ref{ass:spectral_decoupling}, whose violation we measure (Appendix~\ref{app:assumptions:tabular}); on \texttt{jannis}, whose top spectrum is near-degenerate, per-mode identification degrades while the count-level prediction survives. The criterion is proven in the linear regime, while its extension to images is an empirical prediction whose mediating variable is measured after predictor equilibration, not at initialisation (Section~\ref{sec:residual_predictor:scale}, Appendix~\ref{app:equilibration}); the quantitative relation between the initialisation-time budget and the effective rank is claimed on tabular data only. The framework governs encoder-side escape given the predictor's state, whose equilibration the frozen-predictor control tests. On images, effective rank is necessary for accuracy without being sufficient (Section~\ref{sec:residual_predictor:scale}). At deeply subcritical tabular settings, \texttt{ResidualPred} widens the escape basin without guaranteeing a large rank budget: escape there is partial and seed-dependent (Appendix~\ref{app:full_results_tjepa_respred}). \texttt{ResidualPred} extends the identity-bias principle to attention by analogy rather than formal derivation. Finally, $\mu=1$ is a local initialisation criterion, and the ImageNet-1k results are matched-budget pilots; the official full-resolution, full-schedule I-JEPA recipe is untested.

\section{Conclusion}
\label{sec:conclusion}

This work establishes a dynamical systems perspective on JEPAs, identifying the spectral ratio $\mu = \gamma / \sigma$ as the critical coefficient governing the trade-off between representation growth and collapse. Guided by this theory, we introduced \texttt{ResidualPred}, an identity-biased predictor that structurally enforces non-trivial learning dynamics by tying initial stability to view correlation. Beyond JEPAs, the drive-decay framework suggests a broader viewpoint on non-contrastive self-supervised learning: collapse is governed by an explicit spectral competition between alignment-driven growth and predictor-induced decay. Framing stability in these terms clarifies the role of EMA, stop-gradient, and predictor design as mechanisms for inducing timescale separation and amplifying informative modes. This perspective opens the door to theory-driven architecture design, replacing ad hoc collapse-avoidance heuristics with analyzable structural controls.

\section*{Acknowledgements and Disclosures}
This work was performed using HPC resources from the ``M\'e\-so\-cen\-tre'' computing center of CentraleSup\'elec and \'Ecole Normale Sup\'erieure Paris-Saclay supported by CNRS and R\'egion \^Ile-de-France (\url{http://mesocentre.centralesupelec.fr/}). This project was provided with computing AI and storage resources by GENCI at IDRIS thanks to the grant 2026-AD011011828R6 on the supercomputer Jean Zay's H100 partition. This research benefited from the support of the Chair ``Artificial Intelligence applied to the banking domain and to global health'' led by CentraleSup\'elec and sponsored by Lusis, and from a CIFRE fellowship. Competing interests: the first author is employed by Lusis under a CIFRE convention; the authors declare no other financial relationships that could be perceived as influencing this work.

\bibliography{references}
\bibliographystyle{plainnat}

\appendix
\onecolumn
\section{Detailed Derivation of Dynamics and Stability}
\label{app:detailed_derivation}

This appendix provides a comprehensive derivation of the coupled dynamical system, the vectorization of the matrix differential equations, and the spectral analysis that leads to the stability conditions presented in the main text.

\subsection{Derivation of the Coupled Gradient Dynamics}
We begin with the JEPA loss function defined in Eq. (\ref{eq:loss_linear_expected}). The objective is to minimize the expected squared error between the predictor output and the target projection:
\begin{equation}
\mathcal{L} = \frac{1}{2} \mathbb{E} \left[ \| \mat{W}_p \mat{W} \vect{x}_1 - \texttt{sg}(\mat{M}_2 \mat{W}_a \vect{x}) \|_2^2 \right]
\end{equation}

To determine the training dynamics of the context encoder $\mat{W}$, we compute the gradient $\nabla_{\mat{W}} \mathcal{L}$. Let $\vect{e} = \mat{W}_p \mat{W} \vect{x}_1 - \mat{M}_2 \mat{W}_a \vect{x}$ be the error vector. The gradient of the squared norm $\frac{1}{2}\|\vect{e}\|_2^2$ with respect to $\mat{W}$ is given by chain rule:
\begin{equation}
\nabla_{\mat{W}} \left( \frac{1}{2}\|\vect{e}\|_2^2 \right) = \mat{W}_p^\top \vect{e} \vect{x}_1^\top
\end{equation}
Note that the term $\mat{M}_2 \mat{W}_a \vect{x}$ is treated as a constant due to the stop-gradient ($\texttt{sg}$) operator. Taking the expectation over the data distribution:
\begin{align}
\nabla_{\mat{W}} \mathcal{L} &= \mathbb{E} \left[ \mat{W}_p^\top (\mat{W}_p \mat{W} \vect{x}_1 - \mat{M}_2 \mat{W}_a \vect{x}) \vect{x}_1^\top \right] \\
&= \mat{W}_p^\top \mat{W}_p \mat{W} \mathbb{E}[\vect{x}_1 \vect{x}_1^\top] - \mat{W}_p^\top \bar{\mat{M}}_2 \mat{W}_a \mathbb{E}[\vect{x} \vect{x}_1^\top]
\end{align}
Substituting the covariance matrices $\mat{\Sigma}_{11} = \mathbb{E}[\vect{x}_1 \vect{x}_1^\top]$ and $\mat{\Sigma}_{X1} = \mathbb{E}[\vect{x} \vect{x}_1^\top]$ yields the gradient flow $\dot{\mat{W}} \propto -\nabla_{\mat{W}} \mathcal{L}$:
\begin{equation} \label{eq:app_dW}
\dot{\mat{W}} = - \mat{W}_p^\top \mat{W}_p \mat{W} \mat{\Sigma}_{11} + \mat{W}_p^\top \bar{\mat{M}}_2 \mat{W}_a \mat{\Sigma}_{X1}
\end{equation}

Simultaneously, the target encoder $\mat{W}_a$ is updated via an Exponential Moving Average (EMA) of the context encoder $\mat{W}$. In discrete time with momentum $m$, this is $\mat{W}_a^{(t+1)} = m \mat{W}_a^{(t)} + (1-m) \mat{W}^{(t)}$. In the continuous-time limit, taking the update step size to zero, this becomes the differential equation:
\begin{equation} \label{eq:app_dWa}
\dot{\mat{W}}_a = \beta (\mat{W} - \mat{W}_a)
\end{equation}
where $\beta$ is to the update speed $(1-m)$.

\subsection{Vectorization and Jacobian Construction}
To analyze the stability of the zero solution ($\mat{W}=\mat{0}, \mat{W}_a=\mat{0}$), we must represent the system as a standard linear ODE $\dot{\vect{z}} = \mathcal{J}\vect{z}$. We define the state vector $\vect{z}$ by stacking the vectorized weight matrices:
\begin{equation}
\vect{z} = \begin{bmatrix} \vect{vec}(\mat{W}) \\ \vect{vec}(\mat{W}_a) \end{bmatrix} \in \mathbb{R}^{2dN}
\end{equation}

We utilize the fundamental property of the Kronecker product: for any matrices $\mat{A}, \mat{X}, \mat{B}$ of compatible dimensions,
\begin{equation}
\vect{vec}(\mat{A}\mat{X}\mat{B}) = (\mat{B}^\top \otimes \mat{A})\vect{vec}(\mat{X})
\end{equation}

\textbf{Step 1: Vectorizing $\dot{\mat{W}}$ (Eq. \ref{eq:app_dW})}
\begin{itemize}
    \item \textit{First term:} $\mat{W}_p^\top \mat{W}_p \mat{W} \mat{\Sigma}_{11}$.
    Here $\mat{A} = \mat{W}_p^\top \mat{W}_p$, $\mat{X} = \mat{W}$, and $\mat{B} = \mat{\Sigma}_{11}$.
    Using the identity: $\vect{vec}(\mat{W}_p^\top \mat{W}_p \mat{W} \mat{\Sigma}_{11}) = (\mat{\Sigma}_{11}^\top \otimes \mat{W}_p^\top \mat{W}_p) \vect{vec}(\mat{W})$.
    Since $\mat{\Sigma}_{11}$ is symmetric, $\mat{\Sigma}_{11}^\top = \mat{\Sigma}_{11}$.
    
    \item \textit{Second term:} $\mat{W}_p^\top \bar{\mat{M}}_2 \mat{W}_a \mat{\Sigma}_{X1}$.
    Here $\mat{A} = \mat{W}_p^\top \bar{\mat{M}}_2$, $\mat{X} = \mat{W}_a$, and $\mat{B} = \mat{\Sigma}_{X1}$.
    Using the identity: $\vect{vec}(\mat{W}_p^\top \bar{\mat{M}}_2 \mat{W}_a \mat{\Sigma}_{X1}) = (\mat{\Sigma}_{X1}^\top \otimes \mat{W}_p^\top \bar{\mat{M}}_2) \vect{vec}(\mat{W}_a)$.
\end{itemize}
Thus, the vectorized equation for $\mat{W}$ is:
\begin{equation}
\vect{vec}(\dot{\mat{W}}) = -(\mat{\Sigma}_{11} \otimes \mat{W}_p^\top \mat{W}_p)\vect{vec}(\mat{W}) + (\mat{\Sigma}_{X1}^\top \otimes \mat{W}_p^\top \bar{\mat{M}}_2)\vect{vec}(\mat{W}_a)
\end{equation}

\textbf{Step 2: Vectorizing $\dot{\mat{W}}_a$ (Eq. \ref{eq:app_dWa})}
The EMA equation is linear and acts element-wise.
\begin{equation}
\vect{vec}(\dot{\mat{W}}_a) = \beta \vect{vec}(\mat{W}) - \beta \vect{vec}(\mat{W}_a)
\end{equation}
In Kronecker notation (where $\mat{I}_{dN} = \mat{I}_N \otimes \mat{I}_d$), this is:
\begin{equation}
\vect{vec}(\dot{\mat{W}}_a) = (\beta \mat{I}_{dN}) \vect{vec}(\mat{W}) - (\beta \mat{I}_{dN}) \vect{vec}(\mat{W}_a)
\end{equation}

\textbf{Step 3: Constructing the Block Jacobian}
Combining these results into the matrix form $\dot{\vect{z}} = \mathcal{J}\vect{z}$:
\begin{equation}
\begin{bmatrix} \vect{vec}(\dot{\mat{W}}) \\ \vect{vec}(\dot{\mat{W}}_a) \end{bmatrix} = 
\begin{bmatrix} 
-(\mat{\Sigma}_{11} \otimes \mat{W}_p^\top \mat{W}_p) & (\mat{\Sigma}_{X1}^\top \otimes \mat{W}_p^\top \bar{\mat{M}}_2) \\
\beta \mat{I}_{dN} & -\beta \mat{I}_{dN}
\end{bmatrix}
\begin{bmatrix} \vect{vec}(\mat{W}) \\ \vect{vec}(\mat{W}_a) \end{bmatrix}
\end{equation}
This confirms Eq. (\ref{eq:full_jacobian}) in the main text.

\subsection{Characteristic Equation and Eigenvalue Analysis}
To determine stability, we solve for the eigenvalues $\lambda$ of $\mathcal{J}$ using the characteristic equation $\det(\mathcal{J} - \lambda \mat{I}) = 0$.
\begin{equation}
\det \begin{pmatrix}
-(\mat{\Sigma}_{11} \otimes \mat{W}_p^\top \mat{W}_p) - \lambda \mat{I} & (\mat{\Sigma}_{X1}^\top \otimes \mat{W}_p^\top \bar{\mat{M}}_2) \\
\beta \mat{I} & -\beta \mat{I} - \lambda \mat{I}
\end{pmatrix} = 0
\end{equation}

We simplify this by assuming the system can be decomposed into eigenmodes. Let us assume we are working in a basis where the spatial operators are diagonal.
Let $\sigma_i$ be an eigenvalue of the "decay" operator $(\mat{\Sigma}_{11} \otimes \mat{W}_p^\top \mat{W}_p)$.
Let $\gamma_i$ be the corresponding eigenvalue of the "driving" operator $(\mat{\Sigma}_{X1}^\top \otimes \mat{W}_p^\top \bar{\mat{M}}_2)$.
For a specific mode $i$, the determinant of the $2 \times 2$ block becomes:
\begin{equation}
\det \begin{pmatrix}
-\sigma_i - \lambda & \gamma_i \\
\beta & -(\beta + \lambda)
\end{pmatrix} = 0
\end{equation}
Expanding the determinant $(ad - bc)$:
\begin{align}
(-\sigma_i - \lambda)( -(\beta + \lambda) ) - (\beta)(\gamma_i) &= 0 \\
(\lambda + \sigma_i)(\lambda + \beta) - \beta \gamma_i &= 0 \\
\lambda^2 + (\beta + \sigma_i)\lambda + \beta\sigma_i - \beta\gamma_i &= 0
\end{align}

To relate this to the paper's main theorem, we define $\mu_i$ as the ratio of driving force to decay force for that mode: $\mu_i = \gamma_i / \sigma_i$. Substituting $\gamma_i = \mu_i \sigma_i$:
\begin{equation}
\lambda^2 + (\beta + \sigma_i)\lambda + \beta\sigma_i(1 - \mu_i) = 0
\end{equation}

\subsection{Reality and Positivity of the Driving Spectrum}
\label{app:driving_spectrum_reality}

Our stability analysis defines the driving and decay spectral quantities as
\begin{equation}
\gamma_i \in \mathrm{Spec}\!\left( \mat{\Sigma}_{X1}^\top \otimes \mat{W}_p^\top \bar{\mat{M}}_2 \right), 
\qquad
\sigma_i \in \mathrm{Spec}\!\left( \mat{\Sigma}_{11} \otimes \mat{W}_p^\top \mat{W}_p \right),
\qquad
\mu_i = \gamma_i / \sigma_i .
\end{equation}
By construction, $\mat{\Sigma}_{11}$ is positive semidefinite and $\mat{W}_p^\top \mat{W}_p$ is positive semidefinite, which implies $\sigma_i \ge 0$ and $\sigma_i \in \mathbb{R}$.

In contrast, the driving operator $\mat{\Sigma}_{X1}^\top \otimes \mat{W}_p^\top \bar{\mat{M}}_2$ is not symmetric in general. The cross-covariance $\mat{\Sigma}_{X1}$ need not be symmetric (it is under the feature-wise masking model of Appendix~\ref{app:assumptions}, where $\mat{\Sigma}_{X1} = p_{m_1}\mat{\Sigma}_{XX}$), and the matrix $\mat{W}_p^\top \bar{\mat{M}}_2$ is not symmetric unless additional structural constraints are imposed. As a result, $\gamma_i$ may be negative or complex in general. Consequently, the scalar stability ratio $\mu_i = \gamma_i / \sigma_i$ is not guaranteed to be real or ordered, and the threshold condition $\mu_i > 1$ is not, in full generality, a well-defined universal stability boundary.

\paragraph{When the scalar criterion is exact.}
The per-mode criterion $\mu_i > 1$ used in the main text is exact when the following two conditions hold together:

\begin{enumerate}
    \item the driving operator is symmetric, that is, $\mat{\Sigma}_{X1}$ and $\mat{W}_p^\top \bar{\mat{M}}_2$ are both symmetric, so that the Kronecker product is symmetric and every $\gamma_i$ is real;

    \item the exact form of Assumption~\ref{ass:spectral_decoupling} holds with orthonormal eigenbases, so that the coupled dynamics split into independent modes and each $\gamma_i$ is paired with the $\sigma_i$ of the same mode.
\end{enumerate}
Neither condition suffices alone. A symmetric driving operator whose eigenbasis differs from that of the decay operator has real eigenvalues, but the flow does not split into modes, and pairing the $k$-th largest $\gamma$ with the $k$-th largest $\sigma$ can miscount the unstable directions. Conversely, a shared eigenbasis does not make the drive real. With $\bar{\mat{M}}_2 = p_{tgt}\mat{I}$ and a scaled rotation $\mat{W}_p = s\,\mat{R}_\theta$, the matrix $\mat{W}_p^\top\mat{W}_p = s^2\mat{I}$ commutes with $\mat{W}_p^\top\bar{\mat{M}}_2 = p_{tgt}\,s\,\mat{R}_\theta^\top$, so Assumption~\ref{ass:spectral_decoupling} holds exactly, yet $q_\ell = p_{tgt}\,s\,e^{\pm \mathrm{i}\theta}$ is complex. Under the masking model of Appendix~\ref{app:assumptions}, where $\mat{\Sigma}_{X1}$ is symmetric and $\bar{\mat{M}}_2 = p_{tgt}\mat{I}$, the two conditions amount to a symmetric predictor $\mat{W}_p$ together with commuting $\mathrm{diag}(\mat{\Sigma}_{XX})$ and $\mat{\Sigma}_{XX}$.

When the modes decouple but a driving eigenvalue $\gamma_i$ is complex, the characteristic equation above still governs mode $i$, and it has a root with positive real part if and only if
\begin{equation}
\Real\gamma_i + \frac{\beta\,(\operatorname{Im}\gamma_i)^2}{(\beta+\sigma_i)^2} \;>\; \sigma_i ,
\label{eq:complex-drive}
\end{equation}
which reduces to $\gamma_i > \sigma_i$, that is $\mu_i > 1$, for a real drive. To see this, write $b = \beta + \sigma_i$ and $\Delta = b^2 - 4\beta(\sigma_i - \gamma_i)$. The roots are $\frac{1}{2}(-b \pm \sqrt{\Delta})$, so the larger real part is positive exactly when $\Real\sqrt{\Delta} > b$, which is $|\Delta| > b^2 + 4\beta(\sigma_i - \Real\gamma_i)$. If the right-hand side is negative, then $\Real\gamma_i > \sigma_i$ and both conditions hold; otherwise squaring gives $\beta(\operatorname{Im}\gamma_i)^2 > b^2(\sigma_i - \Real\gamma_i)$, which is Eq.~\eqref{eq:complex-drive}. Since $\beta\sigma_i/(\beta+\sigma_i)^2 \le 1/4$, it follows for $\sigma_i > 0$ that $\Real\mu_i > 1$ is sufficient for escape and $|\mu_i| > 1$ is necessary, while neither is the criterion. Without shared eigenbases, Appendix~\ref{app:nonsymmetric_stability} gives certificates that need neither condition.

\paragraph{Alternative real-valued formulation.}
A fully real and nonnegative driving strength can be obtained by replacing eigenvalues with \emph{singular values} (which are well-defined for any matrix). For each pair of indices $(k,\ell)$, define
\begin{equation}
\tilde{\gamma}_{k\ell}
\;:=\;
\sigma^{(k)}\!\bigl(\mat{\Sigma}_{X1}\bigr)
\,\cdot\,
\sigma^{(\ell)}\!\bigl(\mat{W}_p^\top \bar{\mat{M}}_2\bigr),
\qquad
\tilde{\mu}_{k\ell} \;:=\; \frac{\tilde{\gamma}_{k\ell}}{a_k\, p_\ell},
\label{eq:tilde-mu-app}
\end{equation}
where $\sigma^{(k)}(\mat A)$ denotes the $k$-th singular value of $\mat A$ in non-increasing order, $a_k = \sigma^{(k)}(\mat{\Sigma}_{11})$, and $p_\ell = \sigma^{(\ell)}(\mat{W}_p^\top \mat{W}_p)$. By construction $\tilde{\gamma}_{k\ell}, \tilde{\mu}_{k\ell} \in \mathbb{R}_{\ge 0}$ and the threshold condition $\tilde{\mu}_{k\ell} > 1$ is rigorously well-defined regardless of the symmetry of $\mat{\Sigma}_{X1}$ or $\mat{W}_p^\top \bar{\mat{M}}_2$. When both operators are symmetric and share the orthonormal eigenbases of $\mat{\Sigma}_{11}$ and $\mat{W}_p^\top \mat{W}_p$, their singular values are the $|c_k|$ and $|q_\ell|$, but Eq.~\eqref{eq:tilde-mu-app} sorts them independently of the $a_k$ and $p_\ell$. It therefore agrees with the eigenvector-paired absolute-value form $(|c_k|/a_k)\cdot(|q_\ell|/p_\ell)$, which the rank budget of App.~\ref{app:rank-control-proof} uses, only when the two orderings coincide. For example, $\mat{\Sigma}_{11} = \mathrm{diag}(2,1)$ and $\mat{\Sigma}_{X1} = \mathrm{diag}(0.5,1.5)$ with a scalar predictor give one supercritical mode in the paired form, whose flow escapes, and none in Eq.~\eqref{eq:tilde-mu-app}.

\subsection{Stability of the full nonsymmetric block Jacobian}
\label{app:nonsymmetric_stability}

The discussion above states what the scalar criterion needs and stops there. This section proves what the operator class admits without any symmetry or decoupling assumption, and shows how the statements collapse to the per-mode criterion of Section~\ref{sec:theory} when Assumption~\ref{ass:spectral_decoupling} holds exactly. Throughout, $n := dN$,
\begin{equation}
\mat{S} := \mat{\Sigma}_{11} \otimes \mat{W}_p^\top \mat{W}_p, \qquad
\mat{G} := \mat{\Sigma}_{X1}^\top \otimes \mat{W}_p^\top \bar{\mat{M}}_2, \qquad
\mathcal{J} = \begin{pmatrix} -\mat{S} & \mat{G} \\ \beta \mat{I}_n & -\beta \mat{I}_n \end{pmatrix},
\label{eq:SG}
\end{equation}
so that $\mat{S}$ is real symmetric positive semi-definite, a Kronecker product of two such factors, while $\mat{G}$ is a general real matrix. We write $\alpha(\mathcal{J}) := \max\{\Real \lambda : \lambda \in \Spec(\mathcal{J})\}$ for the spectral abscissa, so that the linearised flow collapses when $\alpha(\mathcal{J}) < 0$ and escapes when $\alpha(\mathcal{J}) > 0$, and $r(\mat{G}) := \max_{\|w\| = 1} |w^\ast \mat{G} w|$ for the numerical radius. All statements are local at the collapsed fixed point, within the scope of Lemma~\ref{lemma:coupled_dynamics}; they say nothing about non-trivial equilibria elsewhere.

\begin{theorem}[Quadratic-pencil reduction; no assumptions]
\label{thm:pencil}
Let $Q(\lambda) := \lambda^2 \mat{I}_n + \lambda(\mat{S} + \beta \mat{I}_n) + \beta(\mat{S} - \mat{G})$. Then, as an identity of polynomials in $\lambda$,
$\det(\lambda \mat{I}_{2n} - \mathcal{J}) = \det Q(\lambda)$,
so $\Spec(\mathcal{J})$, with algebraic multiplicities, is the solution set of the quadratic eigenvalue problem $\det Q(\lambda) = 0$. Moreover: (i) for $\lambda \neq -\beta$, pencil eigenpairs $(\lambda, w)$ with $Q(\lambda) w = 0$, $w \neq 0$, correspond bijectively to eigenpairs $(\lambda, [w; \tfrac{\beta}{\lambda + \beta} w])$ of $\mathcal{J}$; (ii) $-\beta \in \Spec(\mathcal{J})$ iff $\mat{G}$ is singular; (iii) $\det \mathcal{J} = \beta^n \det(\mat{S} - \mat{G})$, hence $0 \in \Spec(\mathcal{J})$ iff $\det(\mat{S} - \mat{G}) = 0$.
\end{theorem}

\begin{theorem}[Unconditional abscissa control and collapse certificate]
\label{thm:collapse_certificate}
(i) Let $\lambda \in \Spec(\mathcal{J})$ have $\Real \lambda \ge 0$, with eigenvector $[w; w_a]$ normalised so that $\|w\| = 1$. Then
\begin{equation}
\Real \lambda \;\le\; |w^\ast \mat{G} w| - w^\ast \mat{S} w \;\le\; r(\mat{G}) - \lambda_{\min}(\mat{S}) \;\le\; \sigma_{\max}(\mat{G}) - \lambda_{\min}(\mat{S}).
\end{equation}
(ii) Hence $\alpha(\mathcal{J}) \le \max\{0,\, \sigma_{\max}(\mat{G}) - \lambda_{\min}(\mat{S})\}$, and by the Kronecker structure $\sigma_{\max}(\mat{G}) = \sigma^{(1)}(\mat{\Sigma}_{X1})\, \sigma^{(1)}(\mat{W}_p^\top \bar{\mat{M}}_2) = \tilde\gamma_{11}$ and $\lambda_{\min}(\mat{S}) = a_N p_d$, in the notation of Eq.~\eqref{eq:tilde-mu-app}.
(iii) (Collapse certificate.) If $\sigma_{\max}(\mat{G}) < \lambda_{\min}(\mat{S})$, equivalently, when $\lambda_{\min}(\mat{S}) > 0$, if
\begin{equation}
\check\mu \;:=\; \frac{\sigma_{\max}(\mat{G})}{\lambda_{\min}(\mat{S})} \;=\; \frac{\tilde\gamma_{11}}{a_N p_d} \;=\; \kappa(\mat{\Sigma}_{11})\, \kappa(\mat{W}_p^\top \mat{W}_p)\, \tilde\mu_{11} \;<\; 1,
\label{eq:mucheck}
\end{equation}
then $\alpha(\mathcal{J}) < 0$: the origin is exponentially stable for the linearised flow, with no symmetry, normality or decoupling assumption. Moreover $\tilde\mu_{k\ell} \le \check\mu$ for every $(k,\ell)$, so the certificate is built from the singular values that define the proxy of Eq.~\eqref{eq:tilde-mu-app} and dominates every per-mode $\tilde\mu_{k\ell}$.
(iv) (Explicit rate.) For any $c \in [0, \beta)$ with $\beta \sigma_{\max}(\mat{G})/(\beta - c) - \lambda_{\min}(\mat{S}) < -c$, one has $\alpha(\mathcal{J}) < -c$.
(v) (Non-necessity.) The converse of (iii) is false: $\check\mu \ge 1$ does not imply instability. With $\mat{S} = \mat{I}_n$ and $\mat{G}$ whose sole non-zero entry is $G_{12} = 100$, so that $\check\mu = 100$, one has $\det Q(\lambda) = ((\lambda + 1)(\lambda + \beta))^n$, hence $\Spec(\mathcal{J}) = \{-1, -\beta\}$ and $\alpha(\mathcal{J}) < 0$. Singular values alone cannot decide the stability of $\mathcal{J}$; the gap is a non-normality effect.
\end{theorem}

\begin{theorem}[Escape certificates]
\label{thm:escape_certificate}
(i) (General $\mat{G}$.) If $\det(\mat{S} - \mat{G}) < 0$, then $\mathcal{J}$ has a real eigenvalue $\lambda^\ast > 0$ and the origin is exponentially unstable.
(ii) (Symmetric drive: a sharp trichotomy.) Assume $\mat{G} = \mat{G}^\top$, which holds under sufficient condition~1 of Appendix~\ref{app:driving_spectrum_reality} and in particular under the exact form of Assumption~\ref{ass:spectral_decoupling} with orthonormal shared bases. Call a unit $w \in \R^n$ admissible if $\delta_w := -w^\top(\mat{S} - \mat{G}) w > 0$. Then:
(a) $\mat{S} - \mat{G} \succ 0$ iff $\alpha(\mathcal{J}) < 0$;
(b) $\lambda_{\min}(\mat{S} - \mat{G}) = 0$ implies $\alpha(\mathcal{J}) = 0$;
(c) $\lambda_{\min}(\mat{S} - \mat{G}) < 0$ implies that $\mathcal{J}$ has a real eigenvalue $\lambda^\ast > 0$, and for every admissible $w$ some real eigenvalue of $\mathcal{J}$ lies in $[\lambda_+(w), \Lambda]$, where
\begin{equation}
\begin{aligned}
\lambda_+(w) &:= \tfrac12\Bigl[\sqrt{(w^\top \mat{S} w + \beta)^2 + 4\beta\delta_w} - (w^\top \mat{S} w + \beta)\Bigr] > 0,\\
\Lambda &:= \tfrac12\bigl(-\beta + \sqrt{\beta^2 + 4\beta\|\mat{S} - \mat{G}\|_2}\bigr);
\end{aligned}
\end{equation}
every real non-negative eigenvalue of $\mathcal{J}$ is at most $\Lambda$, and $\alpha(\mathcal{J}) \ge \sup_w \lambda_+(w)$ over admissible $w$.
In particular, for symmetric drive, $\alpha(\mathcal{J}) > 0$ iff $\lambda_{\min}(\mat{S} - \mat{G}) < 0$ iff some direction has $w^\top \mat{G} w > w^\top \mat{S} w$: drive beats decay as quadratic forms. This is a complete stability statement with no decoupling assumption. Under symmetry, (ii) refines (i), since an even number of negative eigenvalues of $\mat{S} - \mat{G}$ has positive determinant and is detected only by (ii).
\end{theorem}

\begin{corollary}[Exact Assumption~\ref{ass:spectral_decoupling} recovers the per-mode criterion]
\label{cor:pencil_reduction}
Assume Assumption~\ref{ass:spectral_decoupling} exactly, in the form of Step~1 of Appendix~\ref{app:rank-control-proof}: orthogonal $\mat{U}$, $\mat{V}$ with $\mat{\Sigma}_{11} = \mat{U}\diag(a_k)\mat{U}^\top$, $\mat{\Sigma}_{X1} = \mat{U}\diag(c_k)\mat{U}^\top$, $\mat{W}_p^\top\mat{W}_p = \mat{V}\diag(p_\ell)\mat{V}^\top$, $\mat{W}_p^\top\bar{\mat{M}}_2 = \mat{V}\diag(q_\ell)\mat{V}^\top$. This forces $\mat{G} = \mat{G}^\top$ and $[\mat{S}, \mat{G}] = 0$. With $\sigma_{k\ell} := a_k p_\ell$ and $\gamma_{k\ell} := c_k q_\ell$:
\begin{enumerate}
\item $\det(\lambda \mat{I}_{2n} - \mathcal{J}) = \prod_{k,\ell}\bigl[\lambda^2 + (\beta + \sigma_{k\ell})\lambda + \beta(\sigma_{k\ell} - \gamma_{k\ell})\bigr]$, the per-mode characteristic polynomial of Section~\ref{sec:theory}, with constant term $\beta\sigma_{k\ell}(1 - \mu_{k\ell})$ whenever $\sigma_{k\ell} > 0$.
\item $\lambda_{\min}(\mat{S} - \mat{G}) = \min_{k\ell}(\sigma_{k\ell} - \gamma_{k\ell})$, so Theorem~\ref{thm:escape_certificate}(ii) reduces mode by mode to the paper's criterion: $\mathcal{J}$ is Hurwitz iff $\sigma_{k\ell} > \gamma_{k\ell}$ for all modes, i.e.\ $\max_{k\ell}\mu_{k\ell} < 1$ when all $\sigma_{k\ell} > 0$; escape iff $\gamma_{k\ell} > \sigma_{k\ell}$ for some mode. Taking $w = \vect{u}_k \otimes \vect{v}_\ell$ gives $\lambda_+(w) = \lambda^+_{k\ell}$, the per-mode escape rate, and
\begin{equation}
\alpha(\mathcal{J}) \;=\; \max_{\{(k,\ell)\,:\,\gamma_{k\ell} > \sigma_{k\ell}\}} \lambda^+_{k\ell}
\end{equation}
whenever a supercritical mode exists: the variational lower bound is attained with equality. The index set is written in the division-free form so that degenerate escaping modes, $\sigma_{k\ell} = 0 < \gamma_{k\ell}$, for which $\mu_{k\ell}$ is undefined, are included.
\item $\tilde\mu_{k\ell} \le \check\mu = \kappa(\mat{\Sigma}_{11})\kappa(\mat{W}_p^\top\mat{W}_p)\,\tilde\mu_{11}$ for all $(k,\ell)$: the unconditional certificate is conservative relative to the per-mode threshold by exactly the condition numbers of the decay factors, and tightens to the $\tilde\mu$ criterion as those factors become well conditioned.
\item (Adiabatic tightness.) $\max_{\|u\|=1}(|u^\ast \mat{G} u| - u^\ast \mat{S} u) = \max_{k\ell}(|\gamma_{k\ell}| - \sigma_{k\ell})$, and on a supercritical mode $\lambda^+_{k\ell} \uparrow \gamma_{k\ell} - \sigma_{k\ell}$ monotonically as $\beta \to \infty$. Hence the bound of Theorem~\ref{thm:collapse_certificate}(i) is asymptotically attained in the adiabatic limit provided the maximiser of $|\gamma_{k\ell}| - \sigma_{k\ell}$ has $\gamma_{k\ell} > 0$, in particular whenever all $c_k q_\ell \ge 0$; with negative drive eigenvalues it can be strictly conservative for every $\beta$ (modes $(\sigma,\gamma) = (1,-5)$ and $(2, 0.5)$ give the bound $4$ while $\mathcal{J}$ is Hurwitz for every $\beta$).
\end{enumerate}
\end{corollary}

Together, the three theorems bracket $\alpha(\mathcal{J})$ between computable singular-value and quadratic-form quantities, and the bracket closes to the exact per-mode answer as Assumption~\ref{ass:spectral_decoupling} becomes exact.

\subsubsection*{Proofs}

\begin{proof}[Proof of Theorem~\ref{thm:pencil}]
Fix $\lambda \in \mathbb{C}$ with $\lambda \neq -\beta$ and write $\lambda \mat{I}_{2n} - \mathcal{J} = \begin{psmallmatrix} \lambda\mat{I} + \mat{S} & -\mat{G} \\ -\beta\mat{I} & (\lambda + \beta)\mat{I} \end{psmallmatrix}$. The lower-right block is invertible, so the Schur determinant formula \citep{hornMatrixAnalysis2013} gives
$\det(\lambda \mat{I}_{2n} - \mathcal{J}) = (\lambda + \beta)^n \det\bigl(\lambda\mat{I} + \mat{S} - \tfrac{\beta}{\lambda+\beta}\mat{G}\bigr) = \det\bigl((\lambda+\beta)(\lambda\mat{I} + \mat{S}) - \beta\mat{G}\bigr) = \det Q(\lambda)$,
using $(\lambda+\beta)(\lambda\mat{I} + \mat{S}) = \lambda^2\mat{I} + \lambda\mat{S} + \beta\lambda\mat{I} + \beta\mat{S}$. Both sides are polynomials in $\lambda$ agreeing on $\mathbb{C} \setminus \{-\beta\}$, hence identically, and both are monic of degree $2n$, so the spectrum with multiplicities is the root set of $\det Q$.
(i) Let $\mathcal{J}z = \lambda z$ with $z = [w; w_a] \neq 0$ and $\lambda \neq -\beta$. The second block row reads $\beta w = (\lambda + \beta) w_a$, so $w_a = \tfrac{\beta}{\lambda+\beta} w$ and $w \neq 0$; substituting into the first row and multiplying by $(\lambda + \beta)$ gives $Q(\lambda) w = 0$. The steps reverse.
(ii) If $\lambda = -\beta$, the second row forces $w = 0$ and the first reduces to $\mat{G} w_a = 0$; consistently $Q(-\beta) = -\beta\mat{G}$.
(iii) Setting $\lambda = 0$: $\det(-\mathcal{J}) = \det Q(0) = \beta^n \det(\mat{S} - \mat{G})$, and $\det(-\mathcal{J}) = \det\mathcal{J}$ because $2n$ is even.
\end{proof}

\begin{proof}[Proof of Theorem~\ref{thm:collapse_certificate}]
(i) Let $\lambda \in \Spec(\mathcal{J})$ with $\Real\lambda \ge 0$ and eigenvector $[w; w_a]$. Since $\Real(\lambda + \beta) \ge \beta > 0$, $\lambda \neq -\beta$, so $w_a = \tfrac{\beta}{\lambda+\beta}w$ with $w \neq 0$; normalise $\|w\| = 1$. The inner product of the first block row with $w$ gives $\lambda = -w^\ast \mat{S} w + \tfrac{\beta}{\lambda+\beta} w^\ast \mat{G} w$. For real symmetric $\mat{S}$, $w^\ast \mat{S} w \in \R$ and $w^\ast \mat{S} w \ge \lambda_{\min}(\mat{S}) \ge 0$ by the Rayleigh bound; also $|\lambda + \beta| \ge \beta$, so $|\beta/(\lambda+\beta)| \le 1$. Taking real parts, $\Real\lambda \le -w^\ast \mat{S} w + |w^\ast \mat{G} w|$, and the chain follows from $|w^\ast \mat{G} w| \le r(\mat{G}) \le \sigma_{\max}(\mat{G})$.
(ii) If $\alpha(\mathcal{J}) \ge 0$ the maximising eigenvalue satisfies (i). Singular values of $\mat{A} \otimes \mat{B}$ are the products $\sigma_i(\mat{A})\sigma_j(\mat{B})$ and eigenvalues of a Kronecker product of symmetric matrices are the products of eigenvalues \citep{hornTopicsMatrixAnalysis1991}, and transposition preserves singular values; hence $\sigma_{\max}(\mat{G}) = \tilde\gamma_{11}$ and, both factors of $\mat{S}$ being positive semi-definite, $\lambda_{\min}(\mat{S}) = a_N p_d$.
(iii) If $\sigma_{\max}(\mat{G}) < \lambda_{\min}(\mat{S})$, an eigenvalue with $\Real\lambda \ge 0$ would satisfy $0 \le \Real\lambda < 0$; the spectrum is finite, so $\alpha(\mathcal{J}) < 0$ strictly. The identity $\check\mu = \kappa(\mat{\Sigma}_{11})\kappa(\mat{W}_p^\top\mat{W}_p)\tilde\mu_{11}$ follows from (ii) and $\tilde\mu_{11} = \tilde\gamma_{11}/(a_1 p_1)$. For the domination, $\sigma^{(k)}(\mat{\Sigma}_{X1}) \le \sigma^{(1)}(\mat{\Sigma}_{X1})$, $\sigma^{(\ell)}(\mat{W}_p^\top\bar{\mat{M}}_2) \le \sigma^{(1)}(\mat{W}_p^\top\bar{\mat{M}}_2)$, $a_k \ge a_N$ and $p_\ell \ge p_d$, hence $\tilde\mu_{k\ell} \le \tilde\gamma_{11}/(a_N p_d) = \check\mu$.
(iv) If some eigenvalue had $\Real\lambda \ge -c$, then $\Real(\lambda+\beta) \ge \beta - c > 0$, the elimination of (i) is valid with $|\beta/(\lambda+\beta)| \le \beta/(\beta - c)$, and the real-part estimate gives $\Real\lambda \le -\lambda_{\min}(\mat{S}) + \tfrac{\beta}{\beta-c}\sigma_{\max}(\mat{G}) < -c$, a contradiction.
(v) With $\mat{S} = \mat{I}_n$ and $\mat{G}$ as stated, strictly upper triangular with a single non-zero entry, $Q(\lambda)$ is upper triangular with diagonal entries $\lambda^2 + (1+\beta)\lambda + \beta = (\lambda+1)(\lambda+\beta)$, so $\det Q(\lambda) = ((\lambda+1)(\lambda+\beta))^n$ and $\Spec(\mathcal{J}) = \{-1, -\beta\}$ by Theorem~\ref{thm:pencil}, while $\check\mu = |G_{12}|$ is arbitrary.
\end{proof}

\begin{proof}[Proof of Theorem~\ref{thm:escape_certificate}]
(i) $p(\lambda) := \det(\lambda\mat{I}_{2n} - \mathcal{J})$ is a monic real polynomial with $p(0) = \beta^n\det(\mat{S} - \mat{G}) < 0$ by Theorem~\ref{thm:pencil}; since $p(\lambda) \to +\infty$ along the positive reals, the intermediate value theorem gives a real root $\lambda^\ast > 0$.
(ii) Assume $\mat{G} = \mat{G}^\top$, so that the coefficient matrices of $Q$ are symmetric. \emph{Case (a).} Let $\lambda \in \Spec(\mathcal{J})$; by Theorem~\ref{thm:pencil} there is $x \neq 0$ with $Q(\lambda)x = 0$, hence $\lambda^2 m_x + \lambda c + k = 0$ with $m_x := \|x\|^2 > 0$, $c := x^\ast(\mat{S} + \beta\mat{I})x \ge \beta m_x > 0$ and $k := \beta x^\ast(\mat{S} - \mat{G})x > 0$, the last two real because the forms are Hermitian. A real quadratic with positive coefficients has both roots in the open left half-plane, so $\alpha(\mathcal{J}) < 0$; the converse is Cases (b) and (c). \emph{Case (b).} The same computation with $k \ge 0$ gives $\Real\lambda \le 0$ for every eigenvalue, and $\det(\mat{S} - \mat{G}) = 0$ puts $0$ in the spectrum by Theorem~\ref{thm:pencil}(iii). \emph{Case (c).} Let $h(\lambda) := \lambda_{\min}(Q(\lambda))$ on $[0,\infty)$, continuous because $\lambda_{\min}$ is $1$-Lipschitz in the spectral norm by Weyl's inequality. For an admissible unit $w$ the scalar quadratic $f_w(\lambda) := w^\top Q(\lambda) w = \lambda^2 + (w^\top\mat{S}w + \beta)\lambda - \beta\delta_w$ has exactly one positive root, $\lambda_+(w)$, and the variational characterisation gives $h(\lambda_+(w)) \le f_w(\lambda_+(w)) = 0$. For any unit $v$ and $\lambda \ge 0$, $v^\top Q(\lambda) v \ge \lambda^2 + \beta\lambda - \beta\|\mat{S} - \mat{G}\|_2$, so $h(\lambda) > 0$ for $\lambda > \Lambda$: every real non-negative eigenvalue lies in $[0, \Lambda]$, and the intermediate value theorem gives a real eigenvalue in $[\lambda_+(w), \Lambda]$. The three cases are exhaustive and exclusive in $\lambda_{\min}(\mat{S} - \mat{G})$. For the refinement, $\det(\mat{S} - \mat{G}) < 0$ iff the number of negative eigenvalues is odd, which forces $\lambda_{\min}(\mat{S} - \mat{G}) < 0$.
\end{proof}

\begin{proof}[Proof of Corollary~\ref{cor:pencil_reduction}]
(1) Under the hypothesis $\mat{\Sigma}_{X1}$ and $\mat{W}_p^\top\bar{\mat{M}}_2$ are symmetric, so $\mat{G}$ is symmetric; with $\mat{\Pi} := \mat{U} \otimes \mat{V}$ the mixed-product rule gives $\mat{S} = \mat{\Pi}\diag(\sigma_{k\ell})\mat{\Pi}^\top$ and $\mat{G} = \mat{\Pi}\diag(\gamma_{k\ell})\mat{\Pi}^\top$, hence $\mat{\Pi}^\top Q(\lambda)\mat{\Pi}$ is diagonal with the per-mode quadratics and $\det Q$ factors; Theorem~\ref{thm:pencil} identifies the factorisation with the spectrum. (2) $\mat{S} - \mat{G}$ is diagonal in the joint basis, so $\lambda_{\min}(\mat{S} - \mat{G}) = \min_{k\ell}(\sigma_{k\ell} - \gamma_{k\ell})$ and Theorem~\ref{thm:escape_certificate}(ii) gives the mode-wise equivalence; $w = \vect{u}_k \otimes \vect{v}_\ell$ has $w^\top\mat{S}w = \sigma_{k\ell}$ and $\delta_w = \gamma_{k\ell} - \sigma_{k\ell}$, so $\lambda_+(w)$ is the positive root of the mode's quadratic. Since the spectrum is exactly the union of per-mode roots, and each quadratic with $b = \beta + \sigma_{k\ell} > 0$ has both roots with $\Real\lambda \le 0$ when its constant term is non-negative and a unique positive root otherwise, the maximum over supercritical modes is attained. (3) is proved inside Theorem~\ref{thm:collapse_certificate}(iii). (4) In joint coordinates $|w^\ast\mat{G}w| - w^\ast\mat{S}w \le \sum_{k\ell}(|\gamma_{k\ell}| - \sigma_{k\ell})|\hat w_{k\ell}|^2 \le \max_{k\ell}(|\gamma_{k\ell}| - \sigma_{k\ell})$, with equality at the maximising mode. For a supercritical mode, $F(\lambda,\beta) := \lambda^2 + (\beta+\sigma)\lambda - \beta(\gamma - \sigma)$ satisfies $F(\gamma - \sigma, \beta) = (\gamma-\sigma)\gamma > 0$, so $\lambda^+ < \gamma - \sigma$; implicit differentiation gives $d\lambda^+/d\beta > 0$; and dividing $F = 0$ by $\beta$ and letting $\beta \to \infty$ gives $\lambda^+ \to \gamma - \sigma$. A mode whose $|\gamma_{k\ell}| - \sigma_{k\ell}$ is maximal but has $\gamma_{k\ell} < 0$ is stable and contributes nothing to the abscissa, which is why the sign qualifier is needed.
\end{proof}

\section{Spectral-Decoupling Regime: Analytical and Empirical Evidence}
\label{app:assumptions}

We complement the discussion of Assumption~\ref{ass:spectral_decoupling} with a closed-form algebraic identity and a quantitative measurement of the resulting commutators on real data. The picture that emerges is a clean separation between a \emph{masking-side factor} that we control through the context ratio $p_{m_1}$ and a \emph{data-side factor} that depends on the spectral geometry of each benchmark. Throughout, we keep the discussion constructive: the identity below makes the regime of validity \emph{operationally falsifiable}, and the measurements show that all three datasets used in the main text sit comfortably inside it.

\subsection{Commutator definition and what it controls}
\label{app:assumptions:commutators}

Assumption~\ref{ass:spectral_decoupling} requires that the data-side covariances $\mat{\Sigma}_{11}$, $\mat{\Sigma}_{X1}$ and the predictor-side operators $\mat{P} := \mat{W}_p^\top \mat{W}_p$, $\mat{Q} := \mat{W}_p^\top \bar{\mat{M}}_2$ admit (approximately) a common eigenbasis on each side. The natural quantitative proxy is the relative Frobenius commutator,
\begin{equation}
C_\Sigma \;:=\; \frac{\bigl\| \mat{\Sigma}_{11} \mat{\Sigma}_{X1} - \mat{\Sigma}_{X1} \mat{\Sigma}_{11} \bigr\|_F}{\| \mat{\Sigma}_{11} \|_F \, \| \mat{\Sigma}_{X1} \|_F}\,, \qquad
C_P \;:=\; \frac{\| \mat{P}\mat{Q} - \mat{Q}\mat{P} \|_F}{\| \mat{P} \|_F \, \| \mat{Q} \|_F}\,,
\label{eq:commutators}
\end{equation}
which vanish exactly when joint diagonalisation holds and grow continuously as the assumption tightens or relaxes. The scalar $\mu$-prediction is exact when $C_\Sigma = C_P = 0$; the per-mode rank-control corollary degrades smoothly in $C_\Sigma + C_P$ and is recovered as soon as either side becomes small.

\subsection{Closed-form factorisation of the data-side commutator}
\label{app:assumptions:closedform}

Under the masking model used in the paper, $\vect{x}_1 = \vect{m}_1 \odot \vect{x}$ with $m_{1j} \overset{\text{iid}}{\sim} \mathrm{Bernoulli}(p_{m_1})$, the data covariances admit the explicit decomposition
\begin{equation}
\mat{\Sigma}_{X1} \;=\; p_{m_1}\, \mat{\Sigma}_{XX}\,, \qquad
\mat{\Sigma}_{11} \;=\; p_{m_1}\, \mathrm{diag}(\mat{\Sigma}_{XX}) \,+\, p_{m_1}^2\, \mathrm{offdiag}(\mat{\Sigma}_{XX})\,,
\label{eq:masking_cov}
\end{equation}
where $\mathrm{diag}(\cdot)$ extracts the diagonal and $\mathrm{offdiag}(\mat A) := \mat A - \mathrm{diag}(\mat A)$. Substituting~\eqref{eq:masking_cov} into the data-side commutator and using $[\mat{\Sigma}_{XX}, \mat{\Sigma}_{XX}] = \mat{0}$ gives the closed-form identity
\begin{equation}
\boxed{\;\bigl[ \mat{\Sigma}_{11}, \mat{\Sigma}_{X1} \bigr] \;=\; p_{m_1}^2\,(1 - p_{m_1}) \;\bigl[\, \mathrm{diag}(\mat{\Sigma}_{XX}),\, \mat{\Sigma}_{XX} \,\bigr]\;}
\label{eq:closed_form_commutator}
\end{equation}
The masking-side and data-side dependences factorise. Two observations follow immediately.

\paragraph{Masking factor is mild at the operating point.}
The scalar $f(p) := p^2(1-p)$ is maximised at $p^\star = 2/3$ with $f(p^\star) = 4/27 \approx 0.148$, and decays for $p \to 0$ and $p \to 1$. At the main image operating point $p_{m_1} = 0.85$ used in the I-JEPA experiments, $f(0.85) = 0.108$, i.e.\ already below $3/4$ of the worst case; at $p_{m_1} = 0.95$ it falls to $0.045$. The masking-side multiplier therefore never amplifies the data-side commutator and, in the limit of mild masking, drives it to zero independently of $\mat{\Sigma}_{XX}$.

\paragraph{Data factor is a property of $\mat{\Sigma}_{XX}$ alone.}
The remaining factor $\mat{D} := \bigl[\, \mathrm{diag}(\mat{\Sigma}_{XX}),\, \mat{\Sigma}_{XX} \,\bigr]$ depends only on how far the input covariance is from being diagonal in the canonical (feature-token) basis. It is invariant under rescaling of $\mat{\Sigma}_{XX}$, which means that the relative commutator $C_\Sigma$ in~\eqref{eq:commutators} is determined by the \emph{shape} of the spectrum, not its overall magnitude.

\subsection{Empirical commutators across spectra}
\label{app:assumptions:empirical}

Using the closed-form identity~\eqref{eq:closed_form_commutator}, we evaluate $C_\Sigma$ on five spectra spanning the regimes the framework needs to cover: an isotropic baseline ($\mat{\Sigma}_{XX} = \mat{I}$), a diagonal power-law (canonical-basis-aligned), a `cifar-like' synthetic with random rotation and a power-law spectrum, a strongly rotated rank-decay spectrum, and the empirical patch covariance of CIFAR-10 ($n = 128{,}000$ patches of size $4\times 4 \times 3$, channel-standardised; $D = 48$). Sweeping $p_{m_1}$ from $0.05$ to $0.99$ in $50$ steps yields Figure~\ref{fig:commutator_vs_p}.

\begin{figure}[t]
    \centering
    \includegraphics[width=0.92\linewidth]{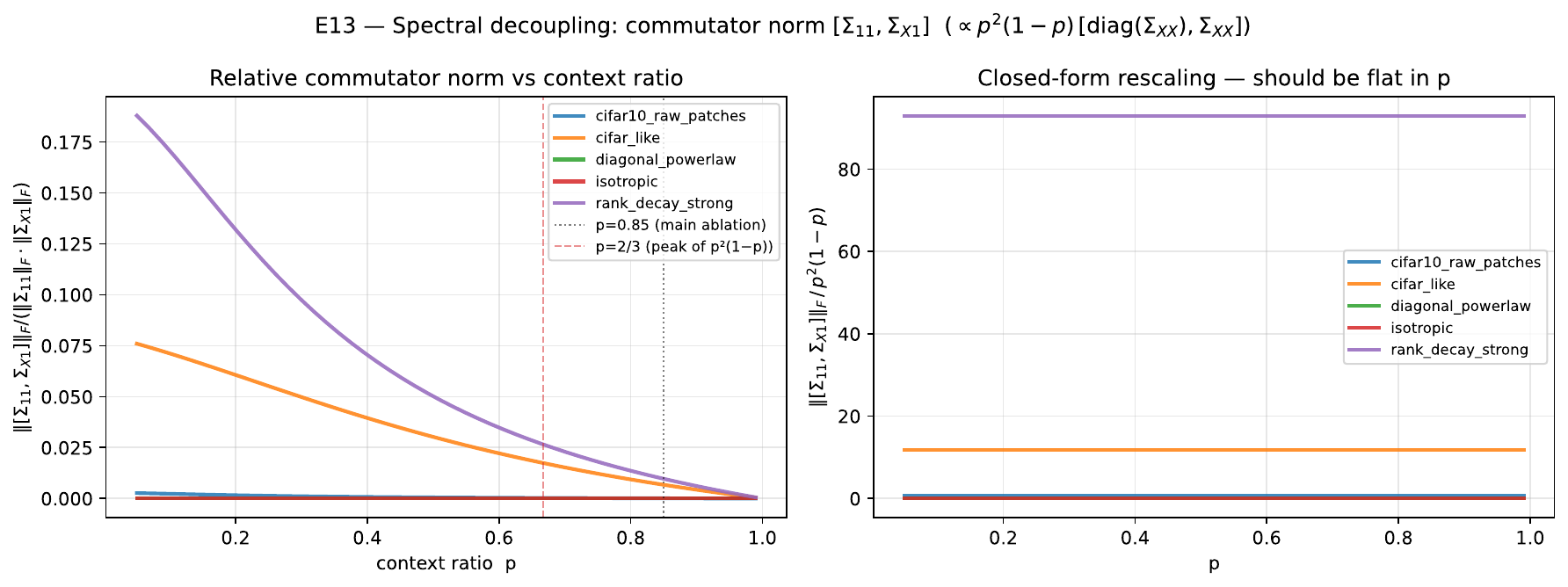}
    \caption{\small\textbf{Relative data-side commutator $C_\Sigma$ as a function of context ratio $p_{m_1}$.} \textit{Left}: raw $C_\Sigma$. \textit{Right}: $\| [\mat{\Sigma}_{11}, \mat{\Sigma}_{X1}] \|_F / \bigl(p_{m_1}^2(1-p_{m_1})\bigr)$ -- flat in $p_{m_1}$ (residual $< 10^{-12}$ relative), confirming the closed-form identity~\eqref{eq:closed_form_commutator}. The CIFAR-10 raw-patch curve sits well below the synthetic benchmarks at all operating points.}
    \label{fig:commutator_vs_p}
\end{figure}

\begin{table}[t]
\centering
\small
\setlength{\tabcolsep}{6pt}
\caption{\small\textbf{Relative data-side commutator $C_\Sigma$ across spectra at $p_{m_1} = 0.85$, and the maximum over $p_{m_1} \in [0.05, 0.99]$.} The masking factor $f(0.85) = 0.108$. Real CIFAR-10 raw-patch covariance gives $C_\Sigma$ on the order of $10^{-4}$, two to three orders of magnitude tighter than rotated power-law surrogates.}
\label{tab:commutator_spectra}
\begin{tabular}{lcc}
\toprule
$\mat{\Sigma}_{XX}$ & $C_\Sigma$ at $p_{m_1}=0.85$ & $\max_{p_{m_1}} C_\Sigma$ \\
\midrule
isotropic ($\mat{I}$)                & $0.0$        & $0.0$        \\
diagonal power-law                   & $0.0$        & $0.0$        \\
CIFAR-10 raw patches ($n{=}128{,}000$, $D{=}48$) & $\mathbf{9.7 \times 10^{-5}}$ & $\mathbf{2.7 \times 10^{-3}}$ \\
synthetic rotated power-law          & $6.7 \times 10^{-3}$ & $7.6 \times 10^{-2}$ \\
synthetic strong rank-decay          & $9.7 \times 10^{-3}$ & $1.9 \times 10^{-1}$ \\
\bottomrule
\end{tabular}
\end{table}

Three points stand out from Table~\ref{tab:commutator_spectra} and Figure~\ref{fig:commutator_vs_p}.
\begin{enumerate}
    \item \textbf{Real-image regime is well inside the decoupled limit.} On real CIFAR-10 raw patches at the main operating point, $C_\Sigma \approx 10^{-4}$. Pixel-level correlation structure aligns the dominant directions of $\mat{\Sigma}_{XX}$ with the canonical basis closely enough that $\mat{D}$ is small in absolute terms, even though the off-diagonal energy of $\mat{\Sigma}_{XX}$ itself is substantial (the off-diagonal-to-total Frobenius ratio is $0.98$). Assumption~\ref{ass:spectral_decoupling} therefore holds at the $10^{-4}$ level on this benchmark, much tighter than typical `approximate' assumptions in linear analyses.
    \item \textbf{Synthetic surrogates trace the boundary of the regime.} Rotated power-law covariances give $C_\Sigma$ in the $10^{-3}$--$10^{-2}$ band, defining a soft envelope above which the per-mode story would need to be re-cast in the eigenbasis of $\mat{D}$ rather than the canonical one.
    \item \textbf{The masking knob always works in our favour.} Because of the $p_{m_1}^2(1 - p_{m_1})$ factor, increasing the context retention probability towards $1$ drives $C_\Sigma$ to zero at a quadratic rate -- regardless of $\mat{\Sigma}_{XX}$. The high-ctx regime studied in the main paper is therefore the regime in which the assumption is sharpest.
\end{enumerate}

\subsection{Tabular benchmarks: \texttt{aloi}, \texttt{helena}, \texttt{jannis}}
\label{app:assumptions:tabular}

The same closed-form identity applies feature-wise to the three tabular benchmarks, so the assumption can be checked before training; we measure the violation on the exact features the sweep trains on, after the min--max preprocessing used throughout the paper.

The commutator alone cannot single out \texttt{jannis}: \texttt{helena} carries a near-identical violation and is predicted well. What distinguishes \texttt{jannis} is spectral degeneracy: its minimum relative gap among the ten largest eigenvalues of $\mat{\Sigma}_{XX}$ is roughly fifteen times smaller than \texttt{helena}'s. This predicts the observed signature, in which per-mode identification degrades while the count-level prediction survives: on \texttt{jannis} the per-mode correlation falls to $+0.619$, against $+0.906$ on \texttt{helena} and $+0.905$ on \texttt{aloi}, while the signed and loose-count correlations stay at $+0.648$ and $+0.697$ with positive confidence intervals (Table~\ref{tab:rank-control-full}). With three datasets we state this as a measured, mechanism-consistent hypothesis and not as a fitted law.

\paragraph{Incorporating the interaction.} Write the linearised generator as $\mathcal{J} = \mat{D} + \epsilon\mat{R}$, with $\mat{D}$ decoupled and the commutator of Eq.~\eqref{eq:commutators} as the measurable small parameter $\epsilon$. First-order terms shift each $\mu_i$ through the diagonal of $\mat{R}$, and modes mix through inverse-eigengap-weighted off-diagonals. Near-degenerate blocks are therefore where interaction first affects mode identity, which is the \texttt{jannis} case. Low-dimensional projectors with multiple views, as in \citet{ghoshHarnessingSmallProjectors2024}, act on exactly that quantity, since they limit how many modes are simultaneously active and thereby shrink the effective cross-talk; their setting is projector-based vision SSL rather than JEPA, so we take the mechanism and not the system. We do not claim the coupled derivation here.

\section{Full proof of the rank-control corollary (Corollary~\ref{cor:rank-control})}
\label{app:rank-control-proof}

We prove the corollary in five steps; Appendix~\ref{app:finite_time_rank} gives the finite-time statements.

\subsection{Step 1 -- Joint diagonalisation in the shared basis}

Let $\mat U=[\vect u_1\,\cdots\,\vect u_N]\in\R^{N\times N}$ and $\mat V=[\vect v_1\,\cdots\,\vect v_d]\in\R^{d\times d}$ be orthonormal eigenbases as in the exact form of Assumption~\ref{ass:spectral_decoupling}, taken with real spectra as in Corollary~\ref{cor:rank-control} (equivalently, $\mat\Sigma_{X1}$ and $\mat W_p^\top\bar{\mat M}_2$ symmetric), so that
\[
\begin{aligned}
\mat\Sigma_{11}&=\mat U\,\mathrm{diag}(a_k)\,\mat U^\top,&\quad
\mat\Sigma_{X1}&=\mat U\,\mathrm{diag}(c_k)\,\mat U^\top,\\
\mat W_p^\top\mat W_p&=\mat V\,\mathrm{diag}(p_\ell)\,\mat V^\top,&\quad
\mat W_p^\top\bar{\mat M}_2&=\mat V\,\mathrm{diag}(q_\ell)\,\mat V^\top.
\end{aligned}
\]
Define $\widetilde{\mat W}:=\mat V^\top\mat W\mat U$ and $\widetilde{\mat W}_a:=\mat V^\top\mat W_a\mat U$. Substituting in the linearised flow (Eq.~\eqref{eq:dW_coupled}),
\begin{align}
\dot{\widetilde{\mat W}} &= -\mathrm{diag}(p_\ell)\,\widetilde{\mat W}\,\mathrm{diag}(a_k) + \mathrm{diag}(q_\ell)\,\widetilde{\mat W}_a\,\mathrm{diag}(c_k), \label{eq:rotW}\\
\dot{\widetilde{\mat W}}_a &= \beta(\widetilde{\mat W}-\widetilde{\mat W}_a). \label{eq:rotWa}
\end{align}
The $(\ell,k)$ component of \eqref{eq:rotW}--\eqref{eq:rotWa} depends only on the $(\ell,k)$ entries. Writing $w_{\ell k}:=[\widetilde{\mat W}]_{\ell k}$ and $w^a_{\ell k}:=[\widetilde{\mat W}_a]_{\ell k}$,
\begin{equation}
\label{eq:scalar-mode}
\begin{bmatrix}\dot w_{\ell k}\\ \dot w^a_{\ell k}\end{bmatrix}
= \underbrace{\begin{bmatrix}-a_k p_\ell & c_k q_\ell\\ \beta & -\beta\end{bmatrix}}_{=:\,\mat J_{k\ell}}
\begin{bmatrix}w_{\ell k}\\ w^a_{\ell k}\end{bmatrix},
\end{equation}
so the $2dN$-dimensional flow splits into $dN$ independent planar flows.

\subsection{Step 2 -- The $2\times 2$ spectrum and saddle condition}

The characteristic polynomial of $\mat J_{k\ell}$ is $\lambda^2 + (\beta + a_k p_\ell)\lambda + \beta a_k p_\ell(1-\mu_{k\ell}) = 0$, with $\mu_{k\ell}:=(c_k q_\ell)/(a_k p_\ell)$, and its discriminant $\Delta_{k\ell} = (\beta - a_k p_\ell)^2 + 4\beta c_k q_\ell$ is non-negative whenever $c_k q_\ell\ge 0$. The eigenvalues are $\lambda^\pm_{k\ell} = \tfrac{1}{2}\big(-(\beta+a_k p_\ell)\pm\sqrt{\Delta_{k\ell}}\big)$, with product $\beta\sigma_{k\ell}(1-\mu_{k\ell})$, so
\[
\mu_{k\ell}>1 \iff \lambda^+_{k\ell}\lambda^-_{k\ell}<0 \iff \mat J_{k\ell}\text{ is a saddle.}
\]
Super-critical modes have $\lambda^+_{k\ell}>0>\lambda^-_{k\ell}$; for sub-critical modes both eigenvalues have negative real parts.

\subsection{Step 3 -- Mode-wise unstable manifolds}

For a saddle block the unstable eigenvector is $\vect e^+_{k\ell}=(c_k q_\ell,\,\lambda^+_{k\ell}+a_k p_\ell)^\top$, and $\vect e^-_{k\ell}$ is its analogue with $\lambda^-_{k\ell}$. Every solution is
\[
\begin{bmatrix}w_{\ell k}(t)\\ w^a_{\ell k}(t)\end{bmatrix}
= \alpha^+_{k\ell}e^{\lambda^+_{k\ell}t}\vect e^+_{k\ell} + \alpha^-_{k\ell}e^{\lambda^-_{k\ell}t}\vect e^-_{k\ell},
\]
so $w_{\ell k}(t) \simeq \alpha^+_{k\ell}(c_k q_\ell)\,e^{\lambda^+_{k\ell}t}$ for $(k,\ell)\in\mathcal U$ when $\alpha^+_{k\ell}\neq 0$, and $w_{\ell k}(t)\to 0$ for sub-critical modes.

\subsection{Step 4 -- Reassembly into the original basis}

Since $\mat W(t) = \mat V\widetilde{\mat W}(t)\mat U^\top = \sum_{\ell,k} w_{\ell k}(t)\vect v_\ell\vect u_k^\top$, Step~3 gives
\begin{equation}
\mat W(t) \simeq \sum_{(k,\ell)\in\mathcal U} a_{k\ell}\,e^{\lambda^+_{k\ell}t}\,\vect v_\ell\vect u_k^\top,
\qquad a_{k\ell}=\alpha^+_{k\ell}c_k q_\ell.
\label{eq:reassembly}
\end{equation}

\subsection{Step 5 -- The rank bound (unstable-component statement)}

Split $\mat W(t) = \mat W^+(t) + \mat W^-(t)$ along the unstable and stable invariant subspaces of the block-diagonal Jacobian. Eq.~\eqref{eq:reassembly} is exact for $\mat W^+(t)$ (Lemma~\ref{lem:saddle_split}), and $\mat W^-(t)$ decays at the slowest stable rate $\eta$ of Eq.~\eqref{eq:eta_slowest}. For a generic initialisation $\mat W^-(t)\neq 0$, so the algebraic rank of $\mat W(t)$ is generically full at every finite $t$; the corollary therefore bounds $\mat W^+(t)$.

Grouping the terms of $\mat W^+(t)$ by predictor index writes it as a sum of $|\pi_\ell\mathcal U|$ rank-one matrices $\vect v_\ell\vect r_\ell(t)^\top$, where $\pi_\ell(\mathcal U)=\{\ell:\exists k,(k,\ell)\in\mathcal U\}$; grouping by data index gives $|\pi_k\mathcal U|$ terms. Hence
\begin{equation}
\rank(\mat W^+(t))\;\le\;r_{\mathrm{tight}}\;:=\;\min\big(|\pi_\ell\mathcal U|,\,|\pi_k\mathcal U|,\,d,\,N\big)\;\le\;\min(|\mathcal U|,d,N),
\end{equation}
which is Eq.~\eqref{eq:r-tight} of the main text. The experiments measure $\Erank$ of the embeddings rather than an algebraic rank; Theorems~\ref{thm:erank_finite_time} and~\ref{thm:erank-budget} give the corresponding finite-time bounds for $\mat W(t)$ and for the embeddings. \hfill$\square$

\subsection{Finite-time restatement of the rank budget}
\label{app:finite_time_rank}

With $\mat{W}_p$ held fixed, the flow of Eq.~\eqref{eq:dW_coupled} is linear in $(\mat{W}, \mat{W}_a)$, so the statements below hold for all $t$. The assumptions are: (A1) the fixed-predictor flow with $\beta > 0$; (A2) the exact form of Assumption~\ref{ass:spectral_decoupling} with orthogonal $\mat{U}$, $\mat{V}$ and real spectra, as in Step~1; (A3) non-degeneracy, namely $a_k > 0$ and $p_\ell > 0$, $\mu_{k\ell} \neq 1$ and $\Delta_{k\ell} \neq 0$ for every mode, where $\Delta_{k\ell} = (\beta - \sigma_{k\ell})^2 + 4\beta\gamma_{k\ell}$ (the last two fail only on a measure-zero set). The approximate version of the assumption, with small commutators, is not covered here.

\paragraph{Notation.} $m := \min(d,N)$; in this section $r := r_{\mathit{tight}}$ of Eq.~\eqref{eq:r-tight}, not the loose count $|\mathcal U|$ of Step~5; $\operatorname{rank}_\varepsilon(\mat{M})$ counts the singular values of $\mat{M}$ strictly above $\varepsilon$; $\Erank$ is the effective rank of Appendix~\ref{app:evaluation_metrics}. The decay rate is
\begin{equation}
\eta := \min\{|\Real\lambda| : \lambda \in \Spec(\mathcal{J}),\ \Real\lambda < 0\} > 0,
\label{eq:eta_slowest}
\end{equation}
the slowest stable rate: per mode, $|\Real\lambda^+_{k\ell}|$ on subcritical modes and $|\lambda^-_{k\ell}|$ on supercritical ones. The transient-growth constants are
\begin{equation}
\mat{N}_{k\ell} := \mat{J}_{k\ell} - \tfrac12\operatorname{tr}(\mat{J}_{k\ell})\mat{I}_2, \qquad
K_{k\ell} := 1 + \frac{2\|\mat{N}_{k\ell}\|_2}{\sqrt{|\Delta_{k\ell}|}}, \qquad
\bar K := \max\bigl(1, \max_{k,\ell} K_{k\ell}\bigr),
\end{equation}
with $Z_0 := (\|\mat{W}(0)\|_F^2 + \|\mat{W}_a(0)\|_F^2)^{1/2}$, $w_{k\ell}(0) := \vect{v}_\ell^\top\mat{W}(0)\vect{u}_k$, $y_{k\ell}(0) := \vect{v}_\ell^\top\mat{W}_a(0)\vect{u}_k$, and, for $(k,\ell) \in \mathcal{U}$, the amplitude $a_{k\ell}$ of the expansion of Step~4, written in initialisation-computable form,
\begin{equation}
a_{k\ell} := \frac{(\beta + \lambda^+_{k\ell})\bigl[\beta\,w_{k\ell}(0) - (\beta + \lambda^-_{k\ell})\,y_{k\ell}(0)\bigr]}{\beta\,(\lambda^+_{k\ell} - \lambda^-_{k\ell})}.
\label{eq:alphaplus}
\end{equation}
With $\mat{W}(t) = \mat{W}_+(t) + \mat{W}_-(t)$ the split of Step~5, Lemma~\ref{lem:saddle_split} below gives
\begin{equation}
\mat{W}_+(t) = \sum_{(k,\ell)\in\mathcal{U}} a_{k\ell}\, e^{\lambda^+_{k\ell}t}\, \vect{v}_\ell\vect{u}_k^\top \qquad \text{for all } t.
\label{eq:Wplus_exact}
\end{equation}

\begin{theorem}[Exact rank budget, all times]
\label{thm:rank_all_times}
Under (A1)--(A3), for every $t \ge 0$, $\rank\mat{W}_+(t) \le r_{\mathit{tight}}$. If some axis projection ($\pi_\ell$ or $\pi_k$) is injective on $\mathcal{U}$ and the corresponding amplitudes $a_{k\ell}$ are all non-zero, the bound is attained for all $t \ge 0$.
\end{theorem}

\begin{theorem}[Finite-time thresholded rank]
\label{thm:thresholded_rank}
Under (A1)--(A3), for all $t \ge 0$,
$\sigma_{r+1}(\mat{W}(t)) \le \|\mat{W}_-(t)\|_F \le \bar K Z_0 e^{-\eta t}$,
and consequently, for every threshold $\varepsilon > 0$, $\operatorname{rank}_\varepsilon(\mat{W}(t)) \le r_{\mathit{tight}}$ for all $t \ge T(\varepsilon) := \max\{0, \eta^{-1}\log(\bar K Z_0/\varepsilon)\}$.
\end{theorem}

\begin{theorem}[Finite-time effective-rank bound]
\label{thm:erank_finite_time}
Assume (A1)--(A3), $\mathcal{U} \neq \emptyset$, and some $(k^\ast,\ell^\ast) \in \mathcal{U}$ with $A_0 := |a_{k^\ast\ell^\ast}| > 0$, which holds almost surely for an absolutely continuous initialisation; write $\lambda_\ast := \lambda^+_{k^\ast\ell^\ast}$. If $r = m$ then $\Erank(\mat{W}(t)) \le r$ trivially. Otherwise define
\begin{equation}
\begin{aligned}
\bar\theta(t) &:= \frac{2(m-r)\bar K Z_0}{A_0}\,e^{-(\lambda_\ast + \eta)t}, \qquad
T_0 := \max\Bigl\{0, \frac{1}{\lambda_\ast + \eta}\log\frac{4(m-r)\bar K Z_0}{A_0}\Bigr\},\\
\varphi(x) &:= x\Bigl(1 + \log\frac{1}{x} + \log_+\frac{m-r}{r}\Bigr).
\end{aligned}
\end{equation}
Then for all $t \ge T_0$, $\Erank(\mat{W}(t)) \le r_{\mathit{tight}}\, e^{\varphi(\bar\theta(t))}$, and whenever additionally $\varphi(\bar\theta(t)) \le 1$,
\begin{equation}
\Erank(\mat{W}(t)) \le r_{\mathit{tight}} + \delta(t), \qquad \delta(t) := (e-1)\,r\,\varphi(\bar\theta(t)) = O\bigl((1+t)e^{-(\lambda_\ast + \eta)t}\bigr) \to 0.
\end{equation}
\end{theorem}

\begin{proposition}[Transfer to the singular-value form]
\label{prop:transfer_tilde}
With the eigenvector-paired $\tilde\mu_{k\ell} = (|c_k|/a_k)(|q_\ell|/p_\ell)$ one has $\mu_{k\ell} \le \tilde\mu_{k\ell}$, hence $\mathcal{U} \subseteq \widetilde{\mathcal{U}}$ and $r_{\mathit{tight}} \le \tilde r_{\mathit{tight}}$. Theorems~\ref{thm:rank_all_times}--\ref{thm:erank_finite_time} hold with the leading $r_{\mathit{tight}}$ replaced by $\tilde r_{\mathit{tight}}$, all other constants still computed from the signed spectra and the signed $r_{\mathit{tight}}$. Substituting $\tilde r$ inside $\varphi$, $\delta$ or $T_0$ is not proved.
\end{proposition}

\subsubsection*{Proofs}

\begin{lemma}[Mode decoupling and classification]
\label{lem:mode_classification}
Under (A1)--(A3), the coordinates $w_{k\ell} = \vect{v}_\ell^\top\mat{W}\vect{u}_k$ and $y_{k\ell} = \vect{v}_\ell^\top\mat{W}_a\vect{u}_k$ decouple the flow into $dN$ planar systems $\dot\zeta_{k\ell} = \mat{J}_{k\ell}\zeta_{k\ell}$ with $\mat{J}_{k\ell}$ as in Eq.~\eqref{eq:scalar-mode}, and $\|\mat{W}\|_F^2 = \sum_{k\ell} w_{k\ell}^2$. If $\mu_{k\ell} > 1$ then $\Delta_{k\ell} > (\beta + \sigma_{k\ell})^2 > 0$ and $\lambda^+_{k\ell} > 0 > -(\beta + \sigma_{k\ell}) > \lambda^-_{k\ell}$. If $\mu_{k\ell} < 1$ then both eigenvalues have negative real part; they are real iff $\Delta_{k\ell} \ge 0$, and the slow stable rate is $|\Real\lambda^+_{k\ell}|$. Hence $\eta > 0$ and $\mathcal{J}$ is hyperbolic.
\end{lemma}
\begin{proof}
The matrices $\vect{v}_\ell\vect{u}_k^\top$ are Frobenius-orthonormal, and substituting $\mat{W} = \sum w_{k\ell}\vect{v}_\ell\vect{u}_k^\top$ into the flow gives $\dot w_{k\ell} = -\sigma_{k\ell}w_{k\ell} + \gamma_{k\ell}y_{k\ell}$, $\dot y_{k\ell} = \beta(w_{k\ell} - y_{k\ell})$. For $\mu > 1$ the constant term of the characteristic polynomial is negative, so the roots are real with opposite signs, and $\Delta = (\beta + \sigma)^2 + 4\beta(\gamma - \sigma) > (\beta+\sigma)^2$ gives $\lambda^- < -(\beta + \sigma)$. For $\mu < 1$ the root sum is $-(\beta + \sigma) < 0$ and the product is positive: both real negative, or a complex pair with real part $-\tfrac12(\beta + \sigma)$. By (A3) no eigenvalue lies on the imaginary axis.
\end{proof}

\begin{lemma}[Planar semigroup bound]
\label{lem:planar_bound}
Let $\mat{M} \in \R^{2\times2}$ have eigenvalues $\lambda_1 \neq \lambda_2$, spectral abscissa $a$, and $\mat{N} := \mat{M} - \tfrac12\operatorname{tr}(\mat{M})\mat{I}$. Then $\|e^{t\mat{M}}\|_2 \le (1 + 2\|\mat{N}\|_2/|\lambda_1 - \lambda_2|)\,e^{at}$ for all $t \ge 0$.
\end{lemma}
\begin{proof}
$\mat{N}$ is traceless with eigenvalues $\pm\delta$, $\delta := \tfrac12(\lambda_1 - \lambda_2) \neq 0$, so $\mat{N}^2 = \delta^2\mat{I}$ and $e^{t\mat{N}} = \cosh(\delta t)\mat{I} + \tfrac{\sinh(\delta t)}{\delta}\mat{N}$, with $e^{t\mat{M}} = e^{\frac12\operatorname{tr}(\mat{M})t}e^{t\mat{N}}$. In the real case $a = \tfrac12\operatorname{tr}(\mat{M}) + \delta$ and $\cosh(\delta t) \le e^{\delta t}$, $\sinh(\delta t) \le \tfrac12 e^{\delta t}$; in the complex case $\delta = i\omega$, $a = \tfrac12\operatorname{tr}(\mat{M})$, $|\cos\omega t| \le 1$ and $|\sin(\omega t)/\omega| \le 2/|\lambda_1 - \lambda_2|$.
\end{proof}
For $\mat{M} = \mat{J}_{k\ell}$, $|\lambda_1 - \lambda_2| = \sqrt{|\Delta_{k\ell}|}$ and the constant is $K_{k\ell}$.

\begin{lemma}[Saddle splitting with initialisation-computable coefficients]
\label{lem:saddle_split}
For $(k,\ell) \in \mathcal{U}$, the eigenvectors of $\mat{J}_{k\ell}$ are $e_\pm = (1, b_\pm)^\top$ with $b_\pm := \beta/(\beta + \lambda^\pm)$, well defined and independent because $\chi(-\beta) = (\beta + \lambda^+)(\beta + \lambda^-) = -\beta\gamma < 0$ and $b_- - b_+ = -\sqrt{\Delta}/\gamma \neq 0$. Every solution splits uniquely as $\zeta(t) = a\,e^{\lambda^+t}e_+ + a^-e^{\lambda^-t}e_-$ with $a = (b_-w(0) - y(0))/(b_- - b_+)$ equal to Eq.~\eqref{eq:alphaplus}, $a^- = (y(0) - b_+w(0))/(b_- - b_+)$, and $|a^-| \le \sqrt2\,(\gamma/\sqrt\Delta)\,\|\zeta(0)\| \le K_{k\ell}\|\zeta(0)\|$; since the first component of $e_+$ is one, $a$ is the amplitude of the $\mat{W}$-coordinate along the unstable eigenline.
\end{lemma}
\begin{proof}
The second row of $(\mat{J}_{k\ell} - \lambda\mat{I})e = 0$ reads $\beta w = (\beta + \lambda)y$, giving $e_\pm$. Evaluating the monic characteristic polynomial at $-\beta$ gives $\chi(-\beta) = -\beta\sigma\mu = -\beta\gamma < 0$ on $\mathcal{U}$, and also $\chi(-\beta) = (\beta + \lambda^+)(\beta + \lambda^-)$, whence $b_- - b_+ = \beta(\lambda^+ - \lambda^-)/((\beta+\lambda^+)(\beta+\lambda^-)) = -\sqrt\Delta/\gamma$. Inverting the $2\times2$ system of eigenvector coordinates gives the coefficients; Cauchy--Schwarz gives $|a^-| \le \sqrt{1 + b_+^2}\,\|\zeta(0)\|\,\gamma/\sqrt\Delta$, with $b_+ \in (0,1)$ since $\lambda^+ > 0$, and $\sqrt2\gamma \le 2\|\mat{N}_{k\ell}\|_2$ because $\gamma$ is an entry of $\mat{N}_{k\ell}$. The two forms of $a$ agree after multiplying numerator and denominator by $(\beta + \lambda^-)$ and using $-\gamma/(\beta + \lambda^-) = (\beta + \lambda^+)/\beta$; with the eigenvector normalisation of Step~3 this amplitude reads $\alpha^+_{k\ell}c_kq_\ell$, which is the $a_{k\ell}$ of Step~4.
\end{proof}

With $\mat{W}_- := \mat{W} - \mat{W}_+$, mode-wise $w^-_{k\ell}(t) = a^-_{k\ell}e^{\lambda^-_{k\ell}t}$ for $(k,\ell) \in \mathcal{U}$ and $w^-_{k\ell}(t) = w_{k\ell}(t)$ otherwise.

\begin{proof}[Proof of Theorem~\ref{thm:rank_all_times}]
Every term of $\mat{W}_+(t)$ has column space in $\operatorname{span}\{\vect{v}_\ell : \ell \in \pi_\ell\mathcal{U}\}$ and row space in $\operatorname{span}\{\vect{u}_k : k \in \pi_k\mathcal{U}\}$, so $\rank\mat{W}_+(t) \le \min(|\pi_\ell\mathcal{U}|, |\pi_k\mathcal{U}|) \le \min(d,N)$ for every $t$. If $\pi_\ell$ is injective on $\mathcal{U}$, each $\ell$ occurs in exactly one pair $(k(\ell),\ell)$ and $\mat{W}_+(t) = \sum_\ell \vect{v}_\ell g_\ell(t)^\top$ with $g_\ell(t) = a_{k(\ell)\ell}e^{\lambda^+t}\vect{u}_{k(\ell)}$; the $\vect{v}_\ell$ are orthonormal and the $g_\ell(t)$ are non-zero multiples of distinct orthonormal vectors at every $t$, since exponentials never vanish, so the rank equals $|\pi_k\mathcal{U}| = r_{\mathit{tight}}$ for all $t$. Symmetrically for $\pi_k$.
\end{proof}

\begin{proof}[Proof of Theorem~\ref{thm:thresholded_rank}]
For $(k,\ell) \notin \mathcal{U}$, Lemma~\ref{lem:mode_classification} places the spectral abscissa of $\mat{J}_{k\ell}$ at $-\eta_{k\ell}$ with $\eta_{k\ell} \ge \eta$, and Lemma~\ref{lem:planar_bound} gives $|w^-_{k\ell}(t)| \le \bar K e^{-\eta t}\|\zeta_{k\ell}(0)\|$; for $(k,\ell) \in \mathcal{U}$, Lemma~\ref{lem:saddle_split} gives $|w^-_{k\ell}(t)| = |a^-_{k\ell}|e^{\lambda^-_{k\ell}t} \le K_{k\ell}\|\zeta_{k\ell}(0)\|e^{-\eta t}$ since $|\lambda^-_{k\ell}| \ge \eta$. The $w^-_{k\ell}$ are the entries of $\mat{V}^\top\mat{W}_-\mat{U}$ with orthogonal factors, so $\|\mat{W}_-(t)\|_F^2 = \sum_{k\ell}|w^-_{k\ell}(t)|^2 \le \bar K^2 Z_0^2 e^{-2\eta t}$. If $r = m$ the singular-value claim is trivial. Otherwise the singular-value perturbation inequality $\sigma_{i+j-1}(\mat{A} + \mat{B}) \le \sigma_i(\mat{A}) + \sigma_j(\mat{B})$ \citep{hornTopicsMatrixAnalysis1991} with $i = r+1$, $j = 1$, $\mat{A} = \mat{W}_+(t)$, $\mat{B} = \mat{W}_-(t)$ and $\sigma_{r+1}(\mat{W}_+(t)) = 0$ (Theorem~\ref{thm:rank_all_times}) gives $\sigma_{r+1}(\mat{W}(t)) \le \|\mat{W}_-(t)\|_2 \le \|\mat{W}_-(t)\|_F$. For $t \ge T(\varepsilon)$, $\sigma_{r+1}(\mat{W}(t)) \le \varepsilon$, and singular values are non-increasing, so at most $r$ of them exceed $\varepsilon$.
\end{proof}

\begin{proof}[Proof of Theorem~\ref{thm:erank_finite_time}]
Assume $r < m$ and write $\varrho(t) := \bar K Z_0 e^{-\eta t}$ and $\sigma_i := \sigma_i(\mat{W}(t))$. \emph{Head.} Since $\vect{u}_{k^\ast}$ and $\vect{v}_{\ell^\ast}$ are unit vectors, $\sigma_1 \ge |\vect{v}_{\ell^\ast}^\top\mat{W}(t)\vect{u}_{k^\ast}| \ge A_0 e^{\lambda_\ast t} - \varrho(t)$ by Lemma~\ref{lem:saddle_split}; for $t \ge T_0$, $\bar\theta(t) \le \tfrac12$ gives $\varrho(t) \le \tfrac14 A_0 e^{\lambda_\ast t}$, hence $\sigma_1 \ge \tfrac12 A_0 e^{\lambda_\ast t} > 0$. \emph{Tail.} By Theorem~\ref{thm:thresholded_rank}, $\sum_{i > r}\sigma_i \le (m-r)\varrho(t)$, so the tail mass $\theta(t) := \sum_{i>r}\sigma_i/\sum_i\sigma_i$ satisfies $\theta(t) \le \bar\theta(t) \le \tfrac12$. \emph{Entropy grouping.} If $\theta = 0$ then $\Erank \le r$ directly. Otherwise group the normalised spectrum into head and tail blocks of masses $(1-\theta, \theta)$; the chain rule for entropy and the maximum-entropy bound on each block give $H(p) \le H_b(\theta) + (1-\theta)\log r + \theta\log(m-r)$, so $\Erank \le r\,e^{u(\theta)}$ with $u(\theta) := H_b(\theta) + \theta\log\frac{m-r}{r}$ \citep{coverElementsInformationTheory2006}. \emph{Envelope.} On $(0,1]$, $-(1-\theta)\log(1-\theta) \le \theta$ gives $H_b(\theta) \le \theta(1 - \log\theta)$ and $u \le \varphi$; $\varphi'(x) = \log\frac1x + \log_+\frac{m-r}{r} \ge 0$ on $(0,1]$, so $\varphi$ is non-decreasing and $u(\theta) \le \varphi(\bar\theta(t))$. \emph{Chord.} $e^x \le 1 + (e-1)x$ on $[0,1]$, which gives the additive form when $\varphi(\bar\theta(t)) \le 1$. \emph{Rate.} With $\bar\theta = Ae^{-bt}$, $b = \lambda_\ast + \eta$, $\delta(t) = (e-1)rAe^{-bt}(1 + bt - \log A + \log_+\frac{m-r}{r})$, which is $O((1+t)e^{-bt})$.
\end{proof}

\begin{proof}[Proof of Proposition~\ref{prop:transfer_tilde}]
$\mu_{k\ell} = c_kq_\ell/(a_kp_\ell) \le |c_k||q_\ell|/(a_kp_\ell) = \tilde\mu_{k\ell}$, so $\mathcal{U} \subseteq \widetilde{\mathcal{U}}$ and the projections are nested, giving $r_{\mathit{tight}} \le \tilde r_{\mathit{tight}}$. All three theorems are upper bounds monotone in the leading $r$, so they hold with $\tilde r_{\mathit{tight}}$ substituted in that position only.
\end{proof}

These statements concern the linear flow with a fixed predictor, in which the fastest escaping mode eventually dominates and the effective rank relaxes toward one (Appendix~\ref{app:image_scale:controls}); $r_{\mathit{tight}}$ is therefore an upper bound and not a prediction of $\Erank$. The case of a trained predictor and a non-linear encoder is measured in Appendix~\ref{app:image_scale}.

\subsection{Singular-value formulation and per-dataset results}
\label{app:rank-control-empirical}

\subsubsection*{Setup}
The sweep is a $3{\times}9{\times}10{\times}3$ grid (datasets $\in\{$\texttt{aloi},\texttt{helena},\texttt{jannis}$\}$; $9$ predictor scales $\alpha\in\{10^{-4},10^{-3},10^{-2},10^{-1},0.5,1,2,5,10\}$; $10$ context ratios $p_{m_1}\in\{10^{-3},10^{-2},0.1,0.2,\ldots,0.8\}$; $3$ seeds), $810$ cells, $804$ completed. Linear T-JEPA, $40$ epochs, AdamW, embed\_dim $d{=}64$, no residual predictor; final-checkpoint values reported.

\subsubsection*{Singular-value formulation $\tilde\mu_{k\ell}$}
\label{app:tilde-mu}

The ratio $\mu_{k\ell}=(c_k/a_k)(q_\ell/p_\ell)$ needs $\mat\Sigma_{X1}$ and $\mat W_p^\top\bar{\mat M}_2$ to have real spectra; in general $\mu_{k\ell}$ may be negative or complex. The eigenvector-paired absolute-value form
\begin{equation}
\tilde\mu_{k\ell}\;:=\;\frac{|c_k|}{a_k}\cdot\frac{|q_\ell|}{p_\ell}\;\ge 0
\label{eq:tilde-mu}
\end{equation}
is real and non-negative and equals $\mu_{k\ell}$ when $c_k, q_\ell>0$. With $\widetilde{\mathcal U}=\{(k,\ell):\tilde\mu_{k\ell}{>}1\}$,
\begin{equation}
\tilde r_{\mathit{tight}}\;:=\;\min\!\big(|\pi_\ell\widetilde{\mathcal U}|,\;|\pi_k\widetilde{\mathcal U}|,\;d,\;N\big).
\end{equation}
Since $\mu_{k\ell} \le \tilde\mu_{k\ell}$, $r_{\mathit{tight}} \le \tilde r_{\mathit{tight}}$, and the bounds of Cor.~\ref{cor:rank-control} hold with $\tilde r_{\mathit{tight}}$ in the leading position (Proposition~\ref{prop:transfer_tilde}).

\begin{table}[h]
\centering
\small
\caption{Per-dataset Spearman $\rho(\Erank,\cdot)$ for the absolute-value and signed formulations and the loose counts, with $1000$-resample bootstrap $95\%$ CIs.}
\label{tab:rank-control-full}
\begin{tabular}{lccc}
\toprule
& \texttt{aloi} ($n{=}264$) & \texttt{helena} ($n{=}270$) & \texttt{jannis} ($n{=}270$) \\
\midrule
$\tilde r_{\mathit{tight}}$ \,(abs., Cor.~\ref{cor:rank-control}) & $\mathbf{+0.905}$ $[.87,.93]$ & $\mathbf{+0.906}$ $[.88,.93]$ & $+0.619$ $[.54,.70]$ \\
$\tilde r_{\mathit{loose}}=|\widetilde{\mathcal U}|$ (abs.)        & $+0.766$ $[.71,.81]$ & $+0.897$ $[.86,.92]$ & $+0.574$ $[.49,.65]$ \\
$r_{\mathit{tight}}$ (signed)                                       & $+0.672$ $[.61,.72]$ & $+0.908$ $[.88,.93]$ & $+0.648$ $[.59,.70]$ \\
$r_{\mathit{loose}}=|\mathcal U|$ (signed)                          & $+0.772$ $[.72,.81]$ & $+0.897$ $[.87,.92]$ & $+0.697$ $[.65,.73]$ \\
$\log_{10}\!\max_{k\ell}\tilde\mu_{k\ell}$ (scalar)                & $+0.493$ & $+0.844$ & $+0.470$ \\
\bottomrule
\end{tabular}
\end{table}

\subsection{Signed versus absolute-value formulation}
\label{app:rank-control-signed-vs-abs}

The two formulations coincide when $c_k$ and $q_\ell$ are positive on every active mode. Otherwise near-zero eigenvalues can change sign across nearby configurations, which makes the signed count unstable, while $\tilde\mu_{k\ell}$ stays defined and bounds it from above (Proposition~\ref{prop:transfer_tilde}). On \texttt{helena}, whose smallest gap among the ten leading eigenvalues is about fifteen times that of \texttt{jannis} (Appendix~\ref{app:assumptions:tabular}), the two agree: $\rho(\Erank, \tilde r_{\mathit{tight}}) = +0.906$ and $\rho(\Erank, r_{\mathit{tight}}) = +0.908$ (Figure~\ref{fig:helena-abs-vs-signed}). On \texttt{aloi} the absolute-value form keeps $\rho = +0.905$ where the signed form drops to $+0.672$; on \texttt{jannis} both soften ($+0.62$ to $+0.65$). We therefore report $\tilde\mu_{k\ell}$ in the main text.

\begin{figure}[h]
\centering
\includegraphics[width=\linewidth]{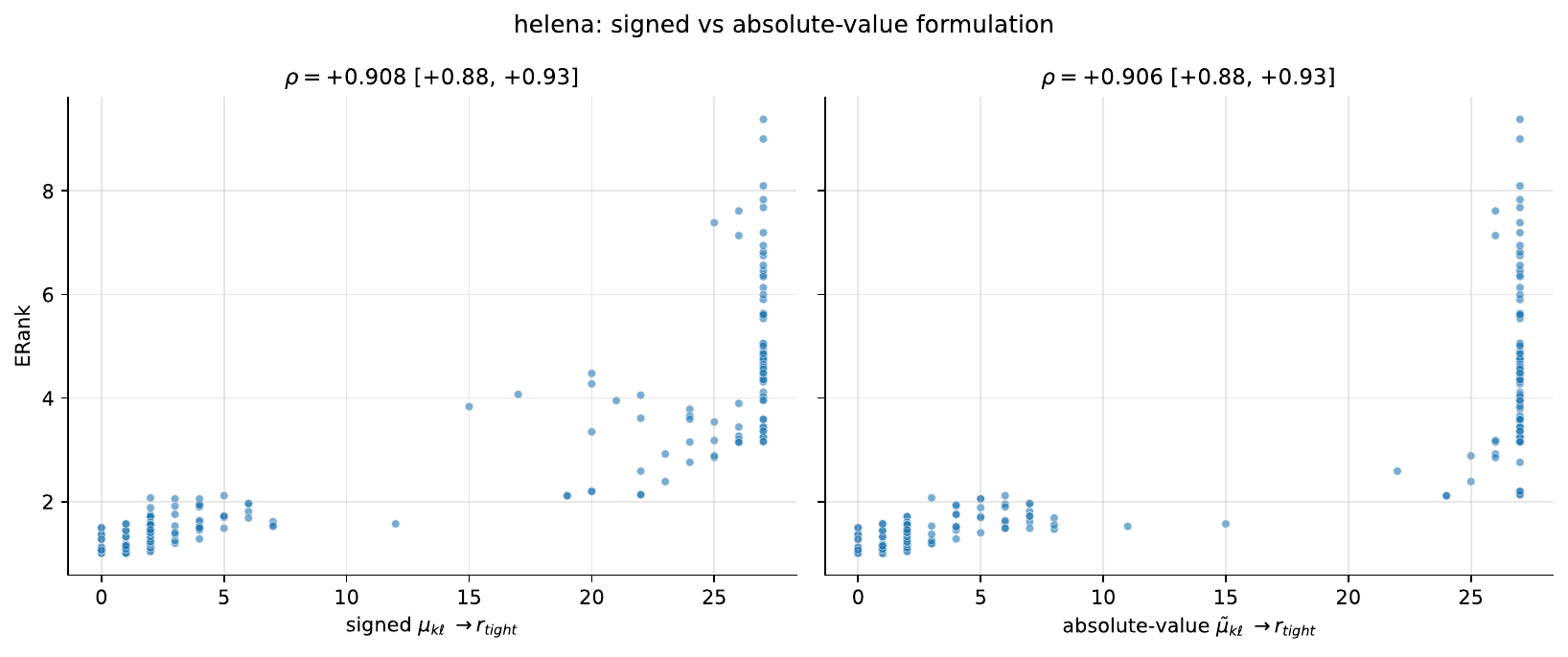}
\caption{\textbf{On \texttt{helena} the two formulations agree.} The signed $\mu_{k\ell}$ and the absolute-value $\tilde\mu_{k\ell}$ give essentially identical scatter against $\Erank$.}
\label{fig:helena-abs-vs-signed}
\end{figure}

\subsection{Transfer of the budget to the embedding spectrum}
\label{app:embedding-transfer}

The experiments read the effective rank of embeddings. This section transfers the budget of
Corollary~\ref{cor:rank-control} from $\mat{W}$ to $\mat{Z}$ in the linear model, under the assumptions of
Appendix~\ref{app:finite_time_rank}. For non-linear encoders the transfer is an empirical prediction,
tested in Section~\ref{sec:residual_predictor} and Appendix~\ref{app:image_scale}.

\subsubsection{The covariance as a readout}

Write $\vect{z} = [\mathrm{vec}(\mat{W})^\top, \mathrm{vec}(\mat{W}_a)^\top]^\top$ for the state and
$\mat{C}(t) = \mat{W}(t)\mat{\Sigma}\mat{W}(t)^\top$ for the embedding covariance under a fixed
positive semi-definite second moment $\mat{\Sigma}$, the population $\mat{\Sigma}_{XX}$ or the
empirical second moment of a fixed evaluation set.

\begin{proposition}[Closure one level up]
\label{prop:lyapunov}
Differentiating $\mat{C}$ along Eq.~\eqref{eq:dW_coupled} gives
\begin{equation}
\dot{\mat{C}} = -\bigl(\mat{P}\mat{W}\mat{\Sigma}_{11}\mat{\Sigma}\mat{W}^\top
  + \mat{W}\mat{\Sigma}\mat{\Sigma}_{11}\mat{W}^\top\mat{P}\bigr)
  + \bigl(\mat{Q}\mat{W}_a\mat{\Sigma}_{X1}\mat{\Sigma}\mat{W}^\top
  + \mat{W}\mat{\Sigma}\mat{\Sigma}_{X1}^\top\mat{W}_a^\top\mat{Q}^\top\bigr),
\end{equation}
which is not a function of $\mat{C}$ alone. The joint second moment
$\mat{\Xi} := \vect{z}\vect{z}^\top$ is closed,
\begin{equation}
\dot{\mat{\Xi}} = \mathcal{J}\mat{\Xi} + \mat{\Xi}\mathcal{J}^\top,
\qquad \mat{\Xi}(t) = e^{\mathcal{J}t}\mat{\Xi}(0)e^{\mathcal{J}^\top t},
\label{eq:dlyap}
\end{equation}
and $\mat{C}(t)$ is the fixed linear readout
$\mat{C}(t) = \mathrm{Tr}_N\bigl[(\mat{\Sigma}\otimes\mat{I}_d)\,\mat{\Xi}_{WW}(t)\bigr]$.
\end{proposition}

\begin{proof}
The first display is the product rule applied to $\mat{C} = \mat{W}\mat{\Sigma}\mat{W}^\top$ with
Eq.~\eqref{eq:dW_coupled} substituted. It involves $\mat{W}\mat{\Sigma}_{11}\mat{\Sigma}\mat{W}^\top$ and
$\mat{W}_a\mat{\Sigma}_{X1}\mat{\Sigma}\mat{W}^\top$, neither of which is a function of
$\mat{W}\mat{\Sigma}\mat{W}^\top$. Vectorising with
$\mathrm{vec}(\mat{A}\mat{X}\mat{B}) = (\mat{B}^\top\otimes\mat{A})\mathrm{vec}(\mat{X})$ turns the
coupled flow into $\dot{\vect{z}} = \mathcal{J}\vect{z}$, whence
$\dot{\mat{\Xi}} = \mathcal{J}\mat{\Xi} + \mat{\Xi}\mathcal{J}^\top$. For the readout, the $(m,n)$ block of
$\mat{\Xi}_{WW}$ is $(\mat{W}\vect{e}_m)(\mat{W}\vect{e}_n)^\top$, so
$\sum_{m,n}\Sigma_{mn}(\mat{W}\vect{e}_m)(\mat{W}\vect{e}_n)^\top = \mat{W}\mat{\Sigma}\mat{W}^\top = \mat{C}$.
\end{proof}

\subsubsection{Growth at twice the rate}

The generator $\mathcal{L}:\mat{X}\mapsto\mathcal{J}\mat{X} + \mat{X}\mathcal{J}^\top$ has spectrum
$\{\lambda_i+\lambda_j : \lambda_i,\lambda_j\in\Spec(\mathcal{J})\}$. When $\mathcal{J}$ is diagonalisable,
as under the hypotheses of Theorem~\ref{thm:erank-budget}, the outer products of its eigenvectors form an
eigenbasis of $\mathcal{L}$, and the top growth rate is $2\max\mathrm{Re}\,\Spec(\mathcal{J})$.

Per mode, a super-critical pair evolves as $\vect{z}_{k\ell}(t) = \alpha_{k\ell}e^{\lambda^+t}\vect{\xi}^+_{k\ell} + \cdots$,
where the $\mat{W}$-component of $\vect{\xi}^+_{k\ell}$ is one, so
$\widetilde{W}_{\ell k}(t)^2 \sim \alpha_{k\ell}^2 e^{2\lambda^+_{k\ell}t}$ for a generic initialisation
($\alpha_{k\ell}\ne 0$). If $\mat{\Sigma} = \mat{U}\mathrm{diag}(s_k)\mat{U}^\top$ with $s_k > 0$, then
$(\mat{V}^\top\mat{C}\mat{V})_{\ell\ell} = \sum_k s_k\widetilde{W}_{\ell k}^2$, so the embedding variance
along $\vect{v}_\ell$ grows at rate $2\max\{\lambda^+_{k\ell} : \mu_{k\ell} > 1\}$ if and only if some $k$
has $\mu_{k\ell} > 1$, and decays otherwise: the collapsed covariance is unstable exactly when
$\mathcal{U} \neq \emptyset$. This per-mode form needs $\mat{\Sigma}$ to share the data eigenbasis, which a
generic empirical second moment does not; the transfer below does not need it.

\subsubsection{The spectral transfer, two-sided}

\begin{lemma}[Weights to embeddings]
\label{lem:transfer}
For any $\mat{W}\in\R^{d\times N}$ and positive semi-definite $\mat{\Sigma}$,
$\sigma_i(\mat{W}\mat{\Sigma}^{1/2}) \le \sigma_{\max}(\mat{\Sigma})^{1/2}\sigma_i(\mat{W})$ for every
$i$; if $\mat{\Sigma}\succ 0$ then also
$\sigma_i(\mat{W}\mat{\Sigma}^{1/2}) \ge \sigma_{\min}(\mat{\Sigma})^{1/2}\sigma_i(\mat{W})$. Hence, for
$\mat{\Sigma}\succ 0$, $\rank(\mat{W}\mat{\Sigma}^{1/2}) = \rank(\mat{W})$, the eigenvalues of $\mat{C}$ are sandwiched
between $\lambda_{\min}(\mat{\Sigma})\sigma_i(\mat{W})^2$ and
$\lambda_{\max}(\mat{\Sigma})\sigma_i(\mat{W})^2$, and for every $\varepsilon>0$ the thresholded ranks
obey
$\rank_{\varepsilon/\sigma_{\min}(\mat{\Sigma})^{1/2}}(\mat{W}) \le
 \rank_\varepsilon(\mat{W}\mat{\Sigma}^{1/2}) \le
 \rank_{\varepsilon/\sigma_{\max}(\mat{\Sigma})^{1/2}}(\mat{W})$.
\end{lemma}

\begin{proof}
The upper bound is the singular-value product inequality
$\sigma_{i+j-1}(\mat{A}\mat{B})\le\sigma_i(\mat{A})\sigma_j(\mat{B})$ at $j=1$ \citep{hornTopicsMatrixAnalysis1991}. The lower bound applies
the same inequality to $\mat{W} = (\mat{W}\mat{\Sigma}^{1/2})\mat{\Sigma}^{-1/2}$. The
three consequences follow index by index.
\end{proof}

The two constants differ by $\kappa(\mat{\Sigma})^{1/2}$, with
$\kappa(\mat{A}) := \sigma_{\max}(\mat{A})/\sigma_{\min}(\mat{A})$. This ratio enters only the head-to-tail
ratio of the theorem below, never $r_{\mathit{tight}}$.

\subsubsection{The budget, at the embedding level}

Let $m := \min(d,N)$, write $r := r_{\mathit{tight}}$ with $1 \le r < m$, let $q(t)$ be the fraction
of the singular-value mass of $\mat{W}(t)\mat{\Sigma}^{1/2}$ carried by the indices beyond $r$, and
set $\psi(q) := q\bigl[\log(1/q) + 1 + \log(m-r)\bigr]$.

\begin{lemma}[Entropy grouping]
\label{lem:entropy-grouping}
Let $s_1 \ge \dots \ge s_m \ge 0$, not all zero, $p_i = s_i/\sum_j s_j$, and
$q := \sum_{i>r} p_i \in [0,1)$ for some $1 \le r < m$. Then
$\Erank = \exp\bigl(-\sum_i p_i\log p_i\bigr) \le r\,e^{\psi(q)}$, and $\psi$ is increasing on
$(0,1)$ with $\psi(q)\to 0$ as $q \downarrow 0$.
\end{lemma}

\begin{proof}
Group the indices into the head $\{1,\dots,r\}$, of mass $1-q$, and the tail, of mass $q$. The
chain rule for entropy gives $H(p) = h(q) + (1-q)H_{\mathrm{head}} + qH_{\mathrm{tail}}$, with $h$ the
binary entropy and $H_{\mathrm{head}}, H_{\mathrm{tail}}$ the entropies of the renormalised head and
tail. Bounding each by the uniform value gives $H(p) \le h(q) + (1-q)\log r + q\log(m-r)$, hence
$\Erank \le r\exp\bigl(h(q) + q\log\tfrac{m-r}{r}\bigr)$. Since $\log(1+x)\le x$ at
$x = q/(1-q)$, $-(1-q)\log(1-q) \le q$, so $h(q) \le q\log(1/q) + q$; and
$q\log\tfrac{m-r}{r} \le q\log(m-r)$ because $r \ge 1$. Finally
$\psi'(q) = \log\bigl((m-r)/q\bigr) > 0$ for $q < 1 \le m-r$.
\end{proof}

\begin{theorem}[Effective rank obeys the budget past an explicit time]
\label{thm:erank-budget}
Under the exact form of Assumption~\ref{ass:spectral_decoupling}, in the escape regime $\mathcal{U} \ne \emptyset$ with a
generic initialisation, with no critical or defective mode ($\mu_{k\ell}\ne 1$ and
$\Delta_{k\ell}\ne 0$ for all $(k,\ell)$), and with $\mat{\Sigma}\succ 0$ and $r < m$, the constants
\[
D := \frac{2(m-r)\,\kappa(\mat{\Sigma})^{1/2}\,\kappa_\star\,\|\vect{z}(0)\|}{\alpha_\star},
\qquad
T_1 := \max\bigl(T_0, \log(2D)/(\lambda^+_{\max}+\eta)\bigr),
\]
computable at initialisation, with the conditioning $\kappa_\star$ of the mode blocks, the
amplitude $\alpha_\star$ of the initial state along the fastest escaping mode and the time $T_0$
defined in the proof, are such that
\begin{equation}
q(t) \le D\,e^{-(\lambda^+_{\max}+\eta)t} \quad (t \ge T_0)
\qquad\text{and}\qquad
\Erank\bigl(\mat{W}(t)\mat{\Sigma}^{1/2}\bigr) \le r\, e^{\psi(q(t))} \quad (t \ge T_1),
\end{equation}
so $\limsup_{t\to\infty}\Erank \le r_{\mathit{tight}}$, with excess
$\psi(q(t)) \le C'(1+t)e^{-(\lambda^+_{\max}+\eta)t}$.
\end{theorem}

\begin{proof}
Recall from Step~5 of this appendix the split $\mat{W}(t) = \mat{W}_+(t) + \mat{W}_-(t)$ along
the unstable and stable invariant subspaces, and $\eta$, the slowest stable rate of
Eq.~\eqref{eq:eta_slowest}, which is positive because no mode is critical. Let
$\mat{T}_{k\ell}$ be the matrix of unit-norm eigenvectors of $\mat{J}_{k\ell}$ in
Eq.~\eqref{eq:scalar-mode}, which exists because $\Delta_{k\ell}\ne 0$, and set
$\kappa_\star := \max_{k,\ell}\kappa(\mat{T}_{k\ell})$. Let $\lambda^+_{\max} := \max_{\mathcal{U}}\lambda^+_{k\ell}$,
attained at $(k_\star,\ell_\star)$, and $\alpha_\star := |\alpha_{k_\star\ell_\star}|$ with
$\alpha_{k\ell}$ the amplitude above. Generic initialisation means
$\alpha_\star \ne 0$, which fails only on a hyperplane of initial states. Write
$\varsigma_i := \sigma_i(\mat{W}(t)\mat{\Sigma}^{1/2})$ and $c := \lambda^+_{\max}+\eta$.

First, the stable part. On a saddle block the stable component evolves as
$e^{\lambda^-_{k\ell}t}\mat{\Pi}^-_{k\ell}\vect{z}_{k\ell}(0)$, with $\mat{\Pi}^-_{k\ell}$ the spectral projector onto
the stable eigenvector and $\|\mat{\Pi}^-_{k\ell}\|_2 \le \kappa(\mat{T}_{k\ell})$, and on a
stable block $\|e^{\mat{J}_{k\ell}t}\|_2 \le \kappa(\mat{T}_{k\ell})\max_\pm e^{\mathrm{Re}\,\lambda^\pm_{k\ell}t}$.
Every stable rate is at most $-\eta$, and the rotation to $\widetilde{\mat{W}}$ is orthogonal, so
summing the blocks gives $\|\mat{W}_-(t)\|_2 \le \|\mat{W}_-(t)\|_F \le \kappa_\star e^{-\eta t}\|\vect{z}(0)\|$.

Second, the tail. By Corollary~\ref{cor:rank-control}, $\rank(\mat{W}_+(t)\mat{\Sigma}^{1/2}) \le r$, so
$\sigma_{r+1}(\mat{W}_+\mat{\Sigma}^{1/2}) = 0$. The additive inequality
$\sigma_{i+j-1}(\mat{A}+\mat{B}) \le \sigma_i(\mat{A}) + \sigma_j(\mat{B})$ at $i=r+1$, $j=1$ then
gives $\varsigma_{r+1} \le \sigma_{\max}(\mat{\Sigma})^{1/2}\|\mat{W}_-(t)\|_2 \le
\sigma_{\max}(\mat{\Sigma})^{1/2}\kappa_\star\|\vect{z}(0)\|e^{-\eta t}$.

Third, the head. Since $\{\vect{v}_\ell\vect{u}_k^\top\}$ is orthonormal and the $\mat{W}$-component of
$\vect{\xi}^+_{k\ell}$ is one, $\vect{v}_{\ell_\star}^\top\mat{W}_+(t)\vect{u}_{k_\star} =
\alpha_{k_\star\ell_\star}e^{\lambda^+_{\max}t}$, so $\sigma_1(\mat{W}_+(t)) \ge \alpha_\star e^{\lambda^+_{\max}t}$.
The same additive inequality and the lower bound of Lemma~\ref{lem:transfer} give
$\varsigma_1 \ge \sigma_{\min}(\mat{\Sigma})^{1/2}\bigl(\alpha_\star e^{\lambda^+_{\max}t} - \kappa_\star\|\vect{z}(0)\|e^{-\eta t}\bigr)$.

Finally, set $T_0 := \max\bigl(0, \log(2\kappa_\star\|\vect{z}(0)\|/\alpha_\star)/c\bigr)$. For
$t \ge T_0$ the stable term is at most half the head term, so
$\varsigma_1 \ge \sigma_{\min}(\mat{\Sigma})^{1/2}\tfrac{\alpha_\star}{2}e^{\lambda^+_{\max}t} > 0$. At most
$m-r$ tail values are non-zero and the total mass is at least $\varsigma_1$, so
$q(t) \le (m-r)\varsigma_{r+1}/\varsigma_1 \le D e^{-ct}$ with $D$ as stated.
For $t \ge T_1$ this gives $q(t) \le \tfrac12$, and Lemma~\ref{lem:entropy-grouping} yields
$\Erank \le r\,e^{\psi(q(t))}$. Because $\psi$ is increasing on $(0,1)$,
$\psi(q(t)) \le \psi(De^{-ct}) = De^{-ct}\bigl[\log(1/D) + ct + 1 + \log(m-r)\bigr] \le C'(1+t)e^{-ct}$
with $C' := D\bigl[1 + |\log D| + \log(m-r) + c\bigr]$, and the $\limsup$ follows.
\end{proof}

The additive form of Theorem~\ref{thm:erank_finite_time} needs the extra condition $\varphi(\bar\theta(t)) \le 1$;
the multiplicative form here does not.

\begin{corollary}[The quantity the experiments measure]
\label{cor:erank-readout}
Fix an evaluation set of $n$ inputs, the same at every $t$, stacked as the columns of
$\mat{X}\in\R^{N\times n}$, and let $\mat{Z}(t) = \mat{W}(t)\mat{X}\in\R^{d\times n}$ and
$\hat{\mat{\Sigma}} := \tfrac1n\mat{X}\mat{X}^\top$. Then
$\sigma_i(\mat{Z}) = \sqrt{n}\,\sigma_i(\mat{W}\hat{\mat{\Sigma}}^{1/2})$ and
$\Erank(\mat{Z}(t)) = \Erank(\mat{W}(t)\hat{\mat{\Sigma}}^{1/2})$. If $\mat{X}$ has rank $N$ (so $n \ge N$),
$\hat{\mat{\Sigma}}\succ 0$ and Theorem~\ref{thm:erank-budget} holds with $\hat{\mat{\Sigma}}$ in place of
$\mat{\Sigma}$; no compatibility of $\hat{\mat{\Sigma}}$ with Assumption~\ref{ass:spectral_decoupling} is
needed.
\end{corollary}

\begin{proof}
$\mat{Z}\mat{Z}^\top = n\,\mat{W}\hat{\mat{\Sigma}}\mat{W}^\top$ gives the
singular-value identity, and the effective rank is invariant under the factor $\sqrt{n}$ and under
zero padding of the spectrum. $\vect{v}^\top\hat{\mat{\Sigma}}\vect{v} = \tfrac1n\|\mat{X}^\top\vect{v}\|^2$
vanishes for some $\vect{v}\ne\vect{0}$ exactly when $\rank(\mat{X}) < N$. The proof of
Theorem~\ref{thm:erank-budget} uses only that the weighting matrix is fixed and positive definite,
so it applies with $\hat{\mat{\Sigma}}$.
\end{proof}

\section{Choice of the Diagonal-Attention Bias $\beta_{\text{init}}$}
\label{app:beta_choice}

\texttt{ResidualPred} (Section~\ref{sec:residual_predictor}) introduces a single new hyperparameter, the diagonal-attention bias $\beta_{\text{init}}$, applied at initialisation as
\begin{equation}
\mathrm{Attention}(\mat{Q}, \mat{K}, \mat{V}) \;=\; \mathrm{Softmax}\!\left( \frac{\mat{Q} \mat{K}^\top}{\sqrt{d}} \;+\; \beta \cdot \mat{B} \right) \mat{V}\,.
\label{eq:diagonal_bias_appendix}
\end{equation}

\begin{figure}[t]
    \centering
    \begin{subfigure}[b]{0.49\linewidth}
        \centering
        \includegraphics[width=\linewidth]{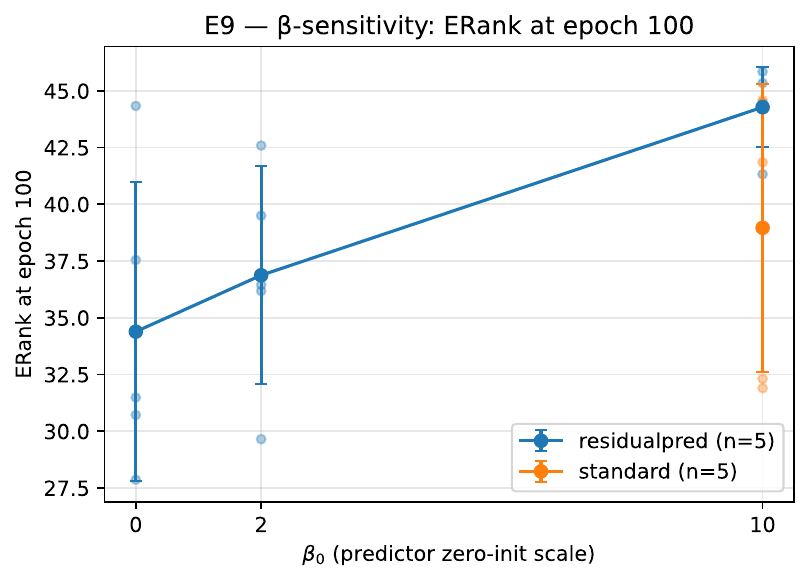}
        \caption{ERank vs.\ $\beta_{\text{init}}$.}
        \label{fig:beta_erank}
    \end{subfigure}\hfill
    \begin{subfigure}[b]{0.49\linewidth}
        \centering
        \includegraphics[width=\linewidth]{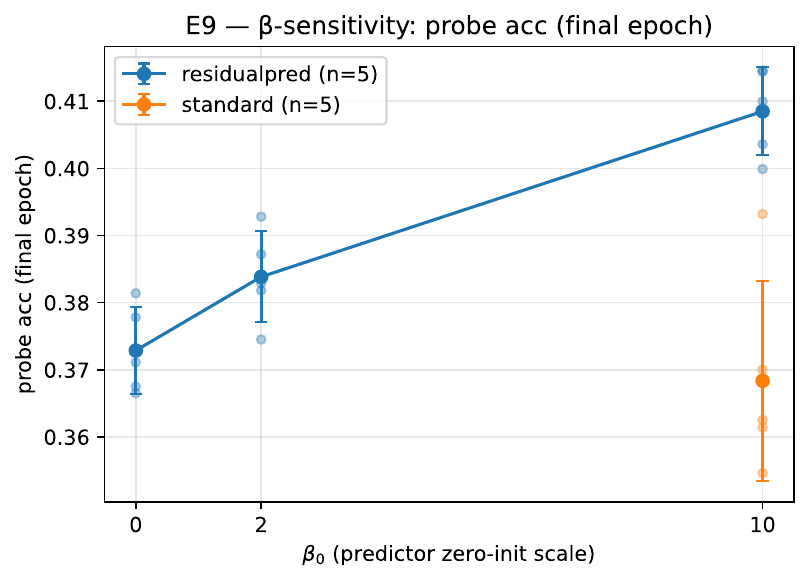}
        \caption{Linear-probe accuracy vs.\ $\beta_{\text{init}}$.}
        \label{fig:beta_probe}
    \end{subfigure}
    \caption{\small\textbf{Monotone effect of $\beta_{\text{init}}$ on ERank and downstream accuracy.}}
    \label{fig:beta_sensitivity}
\end{figure}

The image and tabular pipelines instantiate Eq.~\eqref{eq:diagonal_bias_appendix} differently due to their distinct sequence lengths and token structure.

\paragraph{Images: learnable per-head $\beta$, context-only identity.} Each head has a learnable $\beta_h$ initialised at $\beta_{\text{init}}$ (no weight decay), affecting only the initial bias. The matrix $\mat{B}$ is the identity on context tokens ($i = j \le n_{\text{ctx}}$) and zero elsewhere. Thus, context tokens are biased toward identity, while mask tokens use standard attention to read from context. Per-head $\beta$ captures ViT head specialisation, and restricting the bias to context ensures target tokens can be predicted from context at init.

\paragraph{Tabular: fixed shared $\beta$, full identity.} We use a fixed scalar $\beta = 10$ shared across heads. This simpler design is sufficient because: (i) small $N$ already yields near-saturated diagonals, leaving no benefit for learning $\beta$; (ii) tokens are homogeneous, so per-head variation is unnecessary. Both variants achieve the same objective (identity predictor, near-zero targets) via different mechanisms. The fixed $\beta$ also keeps the tabular sweep lightweight.

\subsection{Saturation criterion}
\label{app:beta_choice:saturation}

For a context-token row, the diagonal entry of the row-stochastic attention matrix at initialisation is
\begin{equation}
A_{ii}(\beta_{\text{init}}) \;=\; \frac{1}{1 + (n_{\text{ctx}}-1)\, e^{-\beta_{\text{init}}}}\,,
\label{eq:saturation}
\end{equation}
under the standard zero-mean assumption on the off-diagonal logits $\mat{Q}\mat{K}^\top / \sqrt{d}$ at initialisation. The threshold $\beta_{\text{init}} \approx \log n_{\text{ctx}} + 5$ guarantees $A_{ii} > 0.99$, i.e.\ a softmax operationally identical to the identity. With keep probability $p_{m_1} = 0.85$, the context-block size is $n_{\text{ctx}} \!\le\! 54$ on CIFAR-10/100 (ViT-S/4, $N = 64$) and $n_{\text{ctx}} \!\le\! 122$ on STL-10 (ViT-S/8, $N = 144$); for both, $\beta_{\text{init}} = 10$ saturates Eq.~\eqref{eq:saturation} above $0.997$.

\section{Zero-Init MLP Output: An Ablated Design Choice}
\label{app:zero_init}

The \texttt{ResidualPred} predictor of Section~\ref{sec:residual_predictor} adapts the identity-bias principle to the transformer architecture through a single structural ingredient: a diagonal-attention bias $\beta_{\text{init}}$ at initialisation. During the design of \texttt{ResidualPred} we considered a second, complementary candidate -- a zero-initialisation of the MLP output projection ($\mat{W}_{\text{fc2}} \!\leftarrow\! \mat{0}$) -- which extends the identity-bias intuition to the second residual sub-block of the transformer. This appendix documents the candidate design, the $2\times 2$ ablation that disentangles its effect from the diagonal-attention bias, and the reason why we ultimately retain only the diagonal bias in the main paper. The exposition is in two registers: the formula and theoretical motivation, and the empirical evidence that informs the design choice.

\subsection{Motivation: a natural extension of the identity-bias principle}
\label{app:zero_init:motivation}

In the linear analysis of Section~\ref{sec:theory}, the predictor enters the dynamics through $\mat{P} := \mat{W}_p^\top \mat{W}_p$ and $\mat{Q} := \mat{W}_p^\top \bar{\mat{M}}_2$, with the per-mode decay rate $\sigma_i$ controlled by the spectrum of $\mat{P}$. Initialising $\mat{W}_p \approx \mat{I} + \mat{\Delta}$ with $\|\mat{\Delta}\| \to 0$ places the predictor in the regime in which view alignment dominates the early dynamics.

A standard transformer-based predictor block (pre-LN convention) is composed of \emph{two} residual sub-blocks,
\begin{equation}
\mathrm{block}(\vect{x}) \;=\; \vect{x} \,+\, \mathrm{Attn}(\mathrm{LN}_1(\vect{x})) \,+\, \mathrm{MLP}\!\bigl(\mathrm{LN}_2(\vect{x} + \mathrm{Attn}(\mathrm{LN}_1(\vect{x})))\bigr)\,,
\label{eq:transformer_block_zi}
\end{equation}
so an exact $\mat{W}_p \approx \mat{I}$ at $t = 0$ would in principle require \emph{both} sub-blocks to behave as the identity. The diagonal-attention bias addresses the attention sub-block (Appendix~\ref{app:beta_choice}). Independently, one can address the MLP sub-block by zeroing its output projection at initialisation, which is the route followed by ReZero \citep{bachlechnerReZeroAllYou2021}, Fixup \citep{zhangFixupInitializationResidual2018}, SkipInit \citep{deBatchNormalizationBiases}, and IDInit \citep{panIDInitUniversalStable2025} for deep residual signal propagation. Whether this route adds measurable value when applied on top of the diagonal-attention bias \emph{specifically for collapse avoidance} is an empirical question we test below.

\subsection{The candidate formula}
\label{app:zero_init:formula}

In a standard transformer, the MLP sub-block is
\begin{equation}
\mathrm{MLP}(\vect{z}) \;=\; \mat{W}_{\text{fc2}}\, \phi\!\bigl(\mat{W}_{\text{fc1}}\, \vect{z} + \vect{b}_{\text{fc1}}\bigr) \,+\, \vect{b}_{\text{fc2}}\,,
\label{eq:mlp_definition_zi}
\end{equation}
with $\phi(\cdot) = \mathrm{GELU}(\cdot)$. The candidate `zero-init' modification is the structural choice
\begin{equation}
\boxed{\quad \mat{W}_{\text{fc2}} \;\leftarrow\; \mat{0}\,, \qquad \vect{b}_{\text{fc2}} \;\leftarrow\; \vect{0} \quad \text{at } t=0\,. \quad}
\label{eq:zero_init_formula}
\end{equation}
Substituting~\eqref{eq:zero_init_formula} into~\eqref{eq:mlp_definition_zi} gives $\mathrm{MLP}(\vect{z}) \equiv \mat{0}$ at $t = 0$ for every input $\vect{z}$, and therefore
\begin{equation}
\mathrm{block}(\vect{x}) \;=\; \vect{x} \,+\, \mathrm{Attn}(\mathrm{LN}_1(\vect{x})) \quad\text{at } t = 0\,.
\label{eq:zero_init_block}
\end{equation}

\paragraph{Scope.} Zero-init applies only to $(\mat{W}_{\text{fc2}}, \vect{b}_{\text{fc2}})$. The attention output projection $\mathrm{attn.proj}$ is left at its standard initialisation: zeroing it would sever the path that lets context tokens inform mask tokens through attention, which is essential for a sequence-to-sequence predictor. The candidate is therefore the appropriate analogue of the ReZero / Fixup recipe for an attention-based predictor.

\paragraph{Gradients are preserved.} Although the parameter is zero, $\partial \mathcal{L} / \partial \mat{W}_{\text{fc2}}$ is non-zero at the first optimisation step (it depends on the post-activation output and the upstream error signal, not on $\mat{W}_{\text{fc2}}$ itself), so the MLP becomes active after a single update. The candidate is a delayed-engagement scheme, not a frozen-parameter scheme.

\subsection{What zero-init does to the predictor block at \texorpdfstring{$t = 0$}{t=0}}
\label{app:zero_init:composition}

The remaining attention residual in~\eqref{eq:zero_init_block} depends on $\beta_{\text{init}}$. Let $\mat{A}(\beta_{\text{init}}) := \mathrm{Softmax}\bigl(\mat{Q}\mat{K}^\top / \sqrt{d} + \beta_{\text{init}} \mat{I}\bigr)$ denote the attention map. With per-token value projection $\mat{V}\,\mathrm{LN}_1(\vect{x})$ and output projection $\mat{O}$, the block under zero-init MLP becomes:
\begin{itemize}
\item \textbf{Diagonal-bias regime ($\beta_{\text{init}} \!\gg\! \log L$).} $\mat{A}(\beta_{\text{init}}) \to \mat{I}_{\text{token}}$ (Eq.~\eqref{eq:saturation} of Appendix~\ref{app:beta_choice}), so each token's output depends on its own input only:
\begin{equation}
\mathrm{block}(\vect{x})_i \;=\; \vect{x}_i \,+\, \mat{O}\,\mat{V}\,\mathrm{LN}_1(\vect{x})_i\,.
\label{eq:zi_diag_bias_block}
\end{equation}
A per-token linear transform of $\vect{x}_i$, full-rank in token space.
\item \textbf{Diffuse regime ($\beta_{\text{init}} = 0$).} $\mat{A}(\beta_{\text{init}}) \approx (1/L)\,\mat{1}\mat{1}^\top$ (uniform attention at random init), and each token's output depends on the average of all tokens' inputs only:
\begin{equation}
\mathrm{block}(\vect{x})_i \;=\; \vect{x}_i \,+\, \mat{O}\,\mat{V}\,\overline{\mathrm{LN}_1(\vect{x})}\,, \qquad \overline{\vect{z}} := \tfrac{1}{L} \sum_{j=1}^L \vect{z}_j\,.
\label{eq:zi_diffuse_block}
\end{equation}
The attention contribution is rank-one in token space: a single direction is broadcast to every token.
\end{itemize}

\paragraph{The decisive factor at $t = 0$ is the support of the predictor's deviation from identity.} View the predictor at $t=0$ as a linear map on the stacked token--feature space $\mathbb{R}^{LN}$. With zero-init MLP, the deviation from identity is the attention residual itself:
\begin{equation}
\mat{W}_p \;=\; \mat{I}_{LN} \,+\, \mat{T}\!\otimes\!(\mat{O}\mat{V})\,, \qquad
\mat{T} \;=\; \begin{cases}\mat{I}_L & \text{(diagonal-bias regime)} \\[2pt] (1/L)\,\mat{1}_L \mat{1}_L^\top & \text{(diffuse regime)} \end{cases}
\label{eq:zi_predictor_decomp}
\end{equation}
The two regimes differ in the \emph{token-space support} of $\mat T$: full rank in the diagonal-bias case, rank one in the diffuse case. Computing $\mat{P} = \mat{W}_p^\top \mat{W}_p = \mat{I}_{LN} + \mat{T} \otimes \bigl((\mat{O}\mat{V})^\top + \mat{O}\mat{V} + (\mat{O}\mat{V})^\top(\mat{O}\mat{V})\bigr)$ shows that $\mat P$ inherits this structure mode-by-mode:
\begin{itemize}
    \item In the diagonal-bias regime ($\mat T = \mat I_L$), every $(k,\ell)$ pair receives the predictor enhancement: $p_\ell$ takes the form $1 + \lambda_\ell(\mat S)$ with $\mat S = (\mat{O}\mat{V})^\top + \mat{O}\mat{V} + (\mat{O}\mat{V})^\top(\mat{O}\mat{V})$, on \emph{all} $LN$ modes. Likewise $q_\ell$ inherits a per-token enhancement on all modes from $\mat W_p^\top \bar{\mat M}_2$.
    \item In the diffuse regime ($\mat T = (1/L)\mat 1\mat 1^\top$, eigenvalues $\{1, 0, \dots, 0\}$), the predictor enhancement is supported only on the \emph{all-ones} direction in token space: $N$ of the $LN$ modes receive the full enhancement $1 + \lambda_\ell(\mat S)$, while the remaining $LN - N$ modes have $p_\ell = 1$ and $q_\ell$ at the trivial baseline given by $\bar{\mat M}_2$ alone.
\end{itemize}

The drive--decay ratio along these unaffected $LN - N$ modes therefore reduces to the data-side baseline $\mu_{k\ell} = (c_k/a_k)\cdot(q_\ell^{(0)}/p_\ell^{(0)})$ with $p_\ell^{(0)}=1$ and $q_\ell^{(0)}$ given by $\bar{\mat M}_2$. In other words, in the diffuse regime the predictor amplifies $\mu_{k\ell}$ on at most $N$ of the $LN$ token-feature modes; in the diagonal-bias regime it amplifies on all $LN$. Zero-init the MLP output suppresses an \emph{additional}, smaller deviation from identity on top of the attention residual~\eqref{eq:zi_predictor_decomp}, but does not change the support of the dominant attention deviation; in the diffuse regime, the attention contribution alone already determines which token-feature modes the predictor can lift above threshold. The theoretical prediction is therefore that zero-init alone (without the diagonal bias) leaves the bulk of the $LN$ modes at their unenhanced baseline, where the data-side ratio $c_k/a_k$ has to do all of the work; its expected effect on accuracy in the diffuse regime is small.

\subsection{Empirical results}
\label{app:zero_init:empirical}

We run the four cells of the grid on the I-JEPA CIFAR-100 benchmark with ViT-Small/4, $L = 64$, $p_{m_1} = 0.85$, at $100$ and $300$ epoch budgets.

\begin{table}[h]
\centering
\small
\setlength{\tabcolsep}{4pt}
\caption{\small\textbf{Linear-probe accuracy across the $2\times 2$ grid, CIFAR-100, ViT-Small/4, $p_{m_1}=0.85$.} Mean $\pm$ standard deviation over seeds. $\Delta$ and two-sided $p$-values are computed against the Standard cell. \texttt{ResidualPred} (the design retained in the main paper) and \texttt{ResidualPred + ZeroInit} are statistically indistinguishable in mean.}
\label{tab:zero_init_results}
\begin{tabular}{l | c c | c c}
\toprule
 & \multicolumn{2}{c|}{\textbf{100 epochs}} & \multicolumn{2}{c}{\textbf{300 epochs}} \\
Variant & Acc.\ (\%) & $\Delta$/$p$ vs Std & Acc.\ (\%) & $\Delta$/$p$ vs Std \\
\midrule
Standard                                        & $36.83 \pm 1.49$ & --            & $46.89 \pm 0.77$ & --            \\
ZeroInit only                                   & $37.29 \pm 0.65$ & $+0.46$ / $0.55$ & $47.20 \pm 0.80$ & $+0.32$ / $0.69$ \\
\rowcolor{rowgray}
\textbf{ResidualPred}\ (bias only) & $\mathbf{41.14 \pm 1.07}$ & $\mathbf{+4.31}$ / $\mathbf{0.011}$ & $\mathbf{50.36 \pm 0.53}$ & $\mathbf{+3.48}$ / $\mathbf{0.046}$ \\
ResidualPred $+$ ZeroInit                        & $40.85 \pm 0.66$ & $+4.02$ / $0.005$ & $49.62 \pm 0.13$ & $+2.73$ / $0.12$  \\
\bottomrule
\end{tabular}
\end{table}

\section{Adiabatic case}
\label{app:adiabatic}

\subsection{Derivation of Linear JEPA Dynamics (Expected Loss)}

We derive the gradients of the expected loss function $\mathcal{L}$ (Eq. \eqref{eq:loss_linear_expected}) with respect to $\mat{W}$ and $\mat{W}_p$.
\begin{equation*}
\mathcal{L} = \frac{1}{2} \mathbb{E} \left[ \| \mat{W}_p \mat{W} \vect{x}_1 - \texttt{sg}(\mat{M}_2 \mat{W}_a \vect{x}) \|_2^2 \right]
\end{equation*}

\textbf{Gradient with respect to $\mat{W}$:}

We need to compute $\nabla_{\mat{W}} \mathcal{L}$. We can move the gradient operator inside the expectation, as expectation is a linear operator:
$$ \nabla_{\mat{W}} \mathcal{L} = \mathbb{E} \left[ \nabla_{\mat{W}} \frac{1}{2} \| \mat{W}_p \mat{W} \vect{x}_1 - \vect{t} \|_2^2 \right] $$
Let the term inside the expectation be $L_s = \frac{1}{2} \| \mat{W}_p \mat{W} \vect{x}_1 - \vect{t} \|_2^2$. This is a standard least-squares type objective.
The gradient of $L_s$ with respect to $\mat{W}$ can be found using matrix calculus identities. Thus, the gradient is:
$$ \nabla_{\mat{W}} L_s = \mat{W}_p^\top (\mat{W}_p \mat{W} \vect{x}_1 - \vect{t}) \vect{x}_1^\top $$
Substituting this back into the expression for $\nabla_{\mat{W}} \mathcal{L}$ and replacing $\vect{t}$ with its definition:
\begin{align*}
\nabla_{\mat{W}} \mathcal{L} &= \mathbb{E} \left[ \mat{W}_p^\top (\mat{W}_p \mat{W} \vect{x}_1 - \mat{M}_2 \mat{W}_a \vect{x}) \vect{x}_1^\top \right] \\
&= \mathbb{E} \left[ \mat{W}_p^\top \mat{W}_p \mat{W} \vect{x}_1 \vect{x}_1^\top - \mat{W}_p^\top \mat{M}_2 \mat{W}_a \vect{x} \vect{x}_1^\top \right]
\end{align*}
Using the linearity of expectation and the definitions $\mat{\Sigma}_{11} = \mathbb{E}[\vect{x}_1 \vect{x}_1^\top]$ and $\mat{\Sigma}_{X1} = \mathbb{E}[\vect{x} \vect{x}_1^\top]$:
\begin{align*}
\nabla_{\mat{W}} \mathcal{L} &= \mat{W}_p^\top \mat{W}_p \mat{W} \mathbb{E}[\vect{x}_1 \vect{x}_1^\top] - \mathbb{E}[\mat{W}_p^\top \mat{M}_2 \mat{W}_a \vect{x} \vect{x}_1^\top] \\
&= \mat{W}_p^\top \mat{W}_p \mat{W} \mat{\Sigma}_{11} - \mathbb{E}[\mat{W}_p^\top \mat{M}_2 \mat{W}_a \vect{x} \vect{x}_1^\top] \quad (*)
\end{align*}

If we further assume that the masking strategy $\mat{M}_2$ is independent of $\vect{x}, \vect{x}_1$, let $\bar{\mat{M}}_2 = \mathbb{E}[\mat{M}_2]$, and assume $\mat{W}_a$ is deterministic (e.g., an Exponential Moving Average update, treated as constant within the expectation), the second term simplifies:
$$ \mathbb{E}[\mat{W}_p^\top \mat{M}_2 \mat{W}_a \vect{x} \vect{x}_1^\top] = \mat{W}_p^\top \mathbb{E}[\mat{M}_2] \mat{W}_a \mathbb{E}[\vect{x} \vect{x}_1^\top] = \mat{W}_p^\top \bar{\mat{M}}_2 \mat{W}_a \mat{\Sigma}_{X1} $$
Substituting this into $(*)$:
$$ \nabla_{\mat{W}} \mathcal{L} = \mat{W}_p^\top \mat{W}_p \mat{W} \mat{\Sigma}_{11} - \mat{W}_p^\top \bar{\mat{M}}_2 \mat{W}_a \mat{\Sigma}_{X1} $$

\textbf{Gradient with respect to $\mat{W}_p$:}
We find the gradient inside the expectation:
\begin{equation*}
\nabla_{\mat{W}_p} \| \mat{W}_p (\mat{W} \vect{x}_1) - (\mat{M}_2 \mat{W}_a \vect{x}) \|_2^2 = 2 (\mat{W}_p \mat{W} \vect{x}_1 - \mat{M}_2 \mat{W}_a \vect{x}) (\mat{W} \vect{x}_1)^\top
\end{equation*}
Taking the expectation:
\begin{align*}
\nabla_{\mat{W}_p} \mathcal{L} &= \frac{1}{2} \mathbb{E} \left[ 2 (\mat{W}_p \mat{W} \vect{x}_1 - \mat{M}_2 \mat{W}_a \vect{x}) (\mat{W} \vect{x}_1)^\top \right] \\
&= \mathbb{E} \left[ \mat{W}_p (\mat{W} \vect{x}_1) (\mat{W} \vect{x}_1)^\top - (\mat{M}_2 \mat{W}_a \vect{x}) (\mat{W} \vect{x}_1)^\top \right] \\
&= \mat{W}_p \mathbb{E} [ (\mat{W} \vect{x}_1) (\mat{W} \vect{x}_1)^\top ] - \mathbb{E} [ (\mat{M}_2 \mat{W}_a \vect{x}) (\mat{W} \vect{x}_1)^\top ] \\
&= \mat{W}_p (\mat{W} \mathbb{E}[\vect{x}_1 \vect{x}_1^\top] \mat{W}^\top) - \mathbb{E} [ \mat{M}_2 \mat{W}_a \vect{x} \vect{x}_1^\top \mat{W}^\top ] \\
&= \mat{W}_p (\mat{W} \mat{\Sigma}_{11} \mat{W}^\top) - \mathbb{E} [ \mat{M}_2 \mat{W}_a \mat{\Sigma}_{X1} \mat{W}^\top ]
\end{align*}
Assuming $\mat{M}_2$ is independent and $\mat{W}_a$ is fixed:
\begin{equation*}
\nabla_{\mat{W}_p} \mathcal{L} = \mat{W}_p (\mat{W} \mat{\Sigma}_{11} \mat{W}^\top) - \bar{\mat{M}}_2 \mat{W}_a \mat{\Sigma}_{X1} \mat{W}^\top
\end{equation*}

\subsection{Stationary Point Equation Derivation}
\label{app:stationary_derivation}

We derive the fixed-point equation for non-trivial stationary points $\mat{W}^* \neq \mat{0}$. Start from the stationary condition for $\mat{W}$, $\nabla_{\mat{W}} \mathcal{L} = \mat{0}$, using the simplified form:
\begin{equation*}
\mat{W}_p^\top \mat{W}_p \mat{W} \mat{\Sigma}_{11} - \mat{W}_p^\top \bar{\mat{M}}_2 \mat{W}_a \mat{\Sigma}_{X1} = \mat{0}
\end{equation*}
Substitute $\mat{W}_a = \tau \mat{W}$:
\begin{equation*}
\mat{W}_p^\top \mat{W}_p \mat{W} \mat{\Sigma}_{11} = \tau \mat{W}_p^\top \bar{\mat{M}}_2 \mat{W} \mat{\Sigma}_{X1}
\end{equation*}
Assume $\mat{W}_p$ is invertible. Left-multiply by $(\mat{W}_p^\top)^{-1} = (\mat{W}_p^{-1})^\top$:
\begin{equation*}
\mat{W}_p \mat{W} \mat{\Sigma}_{11} = \tau \bar{\mat{M}}_2 \mat{W} \mat{\Sigma}_{X1}
\end{equation*}
Assume $\mat{\Sigma}_{11}$ is invertible. Right-multiply by $\mat{\Sigma}_{11}^{-1}$:
\begin{equation*}
\mat{W}_p \mat{W} = \tau \bar{\mat{M}}_2 \mat{W} (\mat{\Sigma}_{X1} \mat{\Sigma}_{11}^{-1})
\end{equation*}
Define $\mat{A}_R = \mat{\Sigma}_{X1} \mat{\Sigma}_{11}^{-1} \in \mathbb{R}^{N \times N}$.
\begin{equation*}
\mat{W}_p \mat{W} = \tau \bar{\mat{M}}_2 \mat{W} \mat{A}_R
\end{equation*}
Assume $\mat{W}_p$ is invertible. Left-multiply by $\mat{W}_p^{-1}$:
\begin{equation*}
\mat{W} = \tau \mat{W}_p^{-1} \bar{\mat{M}}_2 \mat{W} \mat{A}_R
\end{equation*}
Define $\mat{A}_L = \mat{W}_p^{-1} \bar{\mat{M}}_2 \in \mathbb{R}^{d \times d}$.
\begin{equation} \label{eq:app_fixed_point}
\mat{W} = \tau \mat{A}_L \mat{W} \mat{A}_R
\end{equation}

To transform this into the eigenvalue problem, apply the vectorization operator $\vecop(\cdot)$:
\begin{equation*}
\vecop(\mat{W}) = \tau \vecop(\mat{A}_L \mat{W} \mat{A}_R)
\end{equation*}
Using the Kronecker product property $\vecop(\mat{A}\mat{X}\mat{B}) = (\mat{B}^\top \otimes \mat{A}) \vecop(\mat{X})$:
\begin{equation*}
\vecop(\mat{W}) = \tau (\mat{A}_R^\top \otimes \mat{A}_L) \vecop(\mat{W})
\end{equation*}
Rearranging yields the linear system for $\vecop(\mat{W})$:
\begin{equation*}
[\mat{I}_{dN} - \tau (\mat{A}_R^\top \otimes \mat{A}_L)] \vecop(\mat{W}) = \vect{0}
\end{equation*}
This system has non-trivial solutions $\vecop(\mat{W}) \neq \vect{0}$ if and only if the matrix $[\mat{I}_{dN} - \tau (\mat{A}_R^\top \otimes \mat{A}_L)]$ is singular. This occurs when $1$ is an eigenvalue of $\tau (\mat{A}_R^\top \otimes \mat{A}_L)$, or equivalently:
\begin{equation*}
\frac{1}{\tau} \in \Spec(\mat{A}_R^\top \otimes \mat{A}_L)
\end{equation*}
This confirms the condition stated in Theorem \ref{thm:stability} for the existence of non-trivial stationary points.

\textbf{Conditions for Real Eigenvalues:}
The eigenvalues $\lambda$ of $\mat{K} = \mat{A}_R^\top \otimes \mat{A}_L$ are, in general, complex numbers. They are the products of the eigenvalues of $\mat{A}_R^\top$ (which are the same as those of $\mat{A}_R$) and the eigenvalues of $\mat{A}_L$.
The eigenvalues of $\mat{K}$ are guaranteed to be real if both $\mat{A}_R$ and $\mat{A}_L$ are Hermitian (symmetric in the real case).
\begin{itemize}
    \item $\mat{A}_R = \mat{\Sigma}_{X1} \mat{\Sigma}_{11}^{-1} = \mathbb{E}[\vect{x} \vect{x}_1^\top] (\mathbb{E}[\vect{x}_1 \vect{x}_1^\top])^{-1}$. $\mat{\Sigma}_{11}$ is symmetric positive semi-definite (assumed positive definite). However, $\mat{\Sigma}_{X1}$ is generally not symmetric, as it encodes the cross-correlation between the full input and the masked input. Thus, $\mat{A}_R$ is generally not symmetric.
    \item $\mat{A}_L = \mat{W}_p^{-1} \bar{\mat{M}}_2$. $\bar{\mat{M}}_2$ is diagonal, hence symmetric. However, the predictor matrix $\mat{W}_p$ is generally not symmetric, nor is its inverse. Thus, $\mat{A}_L$ is generally not symmetric.
\end{itemize}
Therefore, we must generally consider complex eigenvalues $\lambda$. The stability analysis relies on the magnitude $|\lambda|$, and $\tau_{crit}$ is defined using the eigenvalue with the largest magnitude, $|\lambda_{\max}|$. If, under specific simplifying assumptions (e.g., particular data structures or predictor forms), both $\mat{A}_L$ and $\mat{A}_R$ were symmetric, then all eigenvalues would be real, and $\tau_{crit} = 1 / \max |\lambda_i| = 1 / |\lambda_{\text{dominant}}|$, where $\lambda_{\text{dominant}}$ is the eigenvalue with the largest absolute value.

\subsection{The adiabatic stability theorem}
\label{app:adiabatic_theorem}

Prior analyses, notably \cite{tianUnderstandingSelfsupervisedLearning2021}, often assume that the target weights $\mat{W}_a$ remain proportional to the encoder weights $\mat{W}$. This corresponds to an adiabatic limit of the Jacobian analysis of Section~\ref{sec:theory}, where EMA is much faster than encoder dynamics ($\beta\to\infty$). In this regime, illustrated in Figure~\ref{fig:adiabatic}, $\mat{W}_a$ tracks $\mat{W}$ almost instantaneously, yielding $\vect{vec}(\mat{W}_a)\leftarrow \tau\,\vect{vec}(\mat{W})$ along unstable modes. This approximation leads to an analogous stability threshold, formalized below and derived in Appendix~\ref{app:stationary_derivation}. This bridge to prior Siamese analyses sits in the appendix because it does not motivate the predictor of Section~\ref{sec:residual_predictor}.

\begin{figure}[htb]
    \centering
    \includegraphics[width=\linewidth]{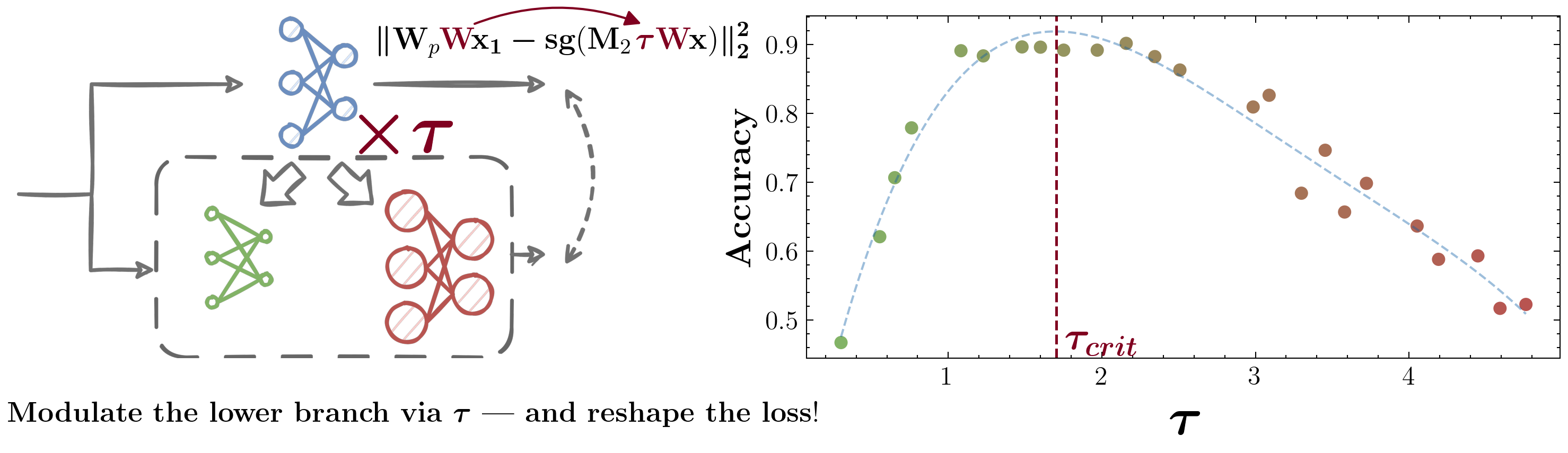}
    \caption{\textbf{Adiabatic stability.} Assuming the target encoder tracks the context encoder proportionally, the parameter $\tau$ can be used to avoid representational collapse also yielding better performance; see the analysis in Appendix~\ref{app:breaking_point}.}
    \label{fig:adiabatic}
\end{figure}

\begin{theorem}[Adiabatic Stability Condition] \label{thm:stability}
In the adiabatic limit, non-trivial stationary points $\mat{W}^* \neq \mat{0}$ exist if and only if the effective coupling $\tau$ satisfies:
\begin{equation} \label{eq:tau_condition}
\frac{1}{\tau} \in \Spec(\mat{A}_R^\top \otimes \mat{A}_L)
\end{equation}
\end{theorem}
with $\mat{A}_R = \mat{\Sigma}_{X1} \mat{\Sigma}_{11}^{-1}$, $\mat{A}_L = \mat{W}_p^{-1} \bar{\mat{M}}_2$ and assuming $\mat{W_a} \leftarrow \tau \mat{W}$. This allows us to define the critical value $\tau_{\text{crit}} := \frac{1}{|\lambda_{\max}(\mat{A}_L)||\lambda_{\max}(\mat{A}_R)|}$ that also defines a region of better downstream performance, as the next subsection shows.

\subsection{The Breaking Point: When JEPA Escapes Collapse}
\label{app:breaking_point}

\subsubsection{Bifurcation Phenomenon}
\label{sec:loss_landscape_characterisation}

The bifurcation behavior in JEPA training dynamics stems from the spectral condition stated in Theorem~\ref{thm:stability}, which shows that non-trivial stationary points emerge if and only if $\frac{1}{\tau} \in \Spec(\mat{A}_R^\top \otimes \mat{A}_L)$. Since the spectrum of the Kronecker product satisfies $\Spec(\mat{A}_R^\top \otimes \mat{A}_L) = \{ \lambda_i \mu_j \}$, where $\lambda_i \in \Spec(\mat{A}_R^\top)$ and $\mu_j \in \Spec(\mat{A}_L)$, the threshold for bifurcation is governed by the largest eigenvalue of $\mat{A}_L$ and $\mat{A}_R$. We define the critical value:
\begin{equation}
\label{eq:tau_crit}
\tau_{\text{crit}} \;:=\; \frac{1}{|\lambda_{\max}(\mat{A}_L)||\lambda_{\max}(\mat{A}_R)|}\,,
\end{equation}
which marks the bifurcation point: for $\tau < \tau_{\text{crit}}$, the only stationary point is the trivial solution $\mat W = \mat 0$, while for $\tau > \tau_{\text{crit}}$, non-trivial solutions $\mat W \neq \mat 0$ emerge, indicating a qualitative change in the structure of the loss landscape.

\begin{figure}[htb]
    \centering
    \includegraphics[width=\textwidth]{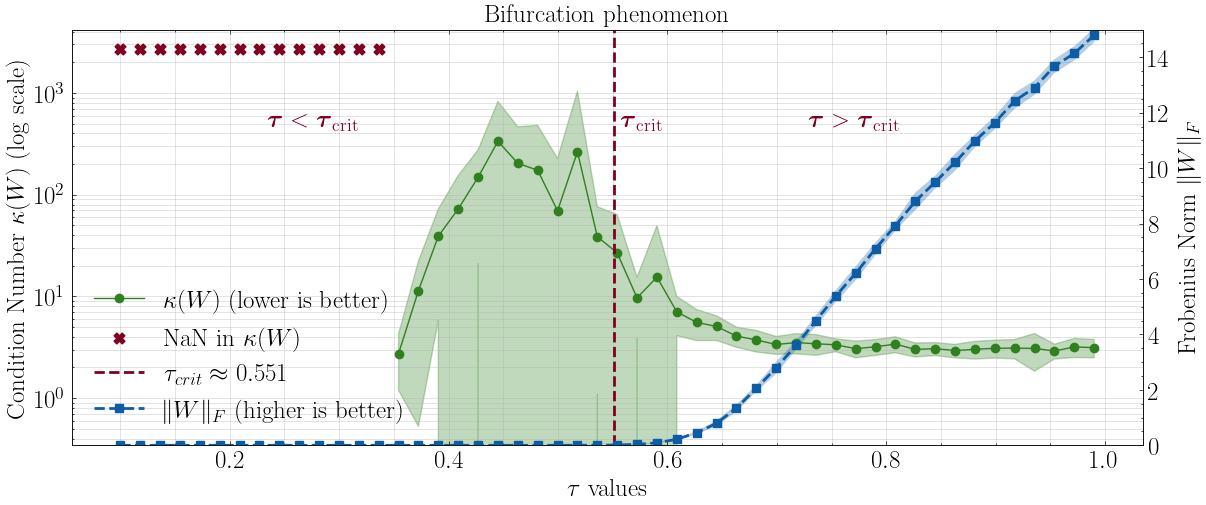} 
    \caption{\textbf{Bifurcation Phenomenon.} Empirical analysis of JEPA collapse metrics as a function of the parameter $\tau$. The plot displays the mean Frobenius Norm (right axis) and mean Condition Number (left axis, log scale) of the encoder weights $\mat{W}$, averaged over 20 experimental runs. The shaded area represents the standard deviation. A critical threshold, empirically determined at $\tau_{crit} \approx 0.551$, delineates the transition between the collapsed and non-collapsed operational regimes.}
    \label{fig:jepa_metrics_vs_tau}
\end{figure}

Empirical evidence supporting this theoretical framework is presented in Figure \ref{fig:jepa_metrics_vs_tau}. This figure illustrates the behavior of key collapse metrics, specifically, the Frobenius norm $||\mat{W}||_F$ and the condition number $\kappa(\mat{W})$ (the ratio of the largest to smallest singular value, $\sigma_{\text{max}}(\mat{W}^*) / \sigma_{\text{min}}(\mat{W}^*)$) of the encoder weights $\mat{W}$, as a function of the parameter $\tau$. 

The behavior of these metrics across different regimes of $\tau$ can be delineated as follows:

\begin{itemize}
  \item \textbf{Subcritical Regime ($\tau < \tau_{crit}$):}  
  The encoder converges to the trivial solution $\mat{W} \approx \mat{0}$.
  \begin{itemize}
    \item Frobenius norm remains close to zero $\big(||\mat{W}||_F \approx 0\big)$, indicating no meaningful transformation is learned.
    \item Condition number is extremely high or undefined, reflecting that $\mat{W}$ is ill-conditioned or nearly singular.
  \end{itemize}

  \item \textbf{Supercritical Regime ($\tau > \tau_{crit}$):}  
  Non-trivial solutions $\mat{W}^* \ne \mat{0}$ emerge and the encoder escapes collapse.
  \begin{itemize}
    \item Frobenius norm increases sharply, signaling effective learning of structured representations.
    \item Condition number stabilizes at moderate values, indicating a well-conditioned matrix that preserves feature diversity.
  \end{itemize}
\end{itemize}

\subsection{Loss Landscape Transformation}
\label{sec:loss_landscape_visualisation}

To further elucidate the structural changes in the JEPA loss landscape as a consequence of the bifurcation at $\tau_{\text{crit}}$ (Eq. \ref{eq:tau_crit}), Figure \ref{fig:loss_landscape_2d} provides a conceptual, two-dimensional projection. These visualizations illustrate the loss function $\mathcal{L}$ over a simplified parameter space of the context encoder matrix $\mat{W}$. This space is spanned by $r_1$ and $r_2$. These factors scale two distinct groups of eigenvalues of the encoder matrix $\mat{W}$: $r_1$ scales the group of eigenvalues that are inherently larger, while $r_2$ scales the group of eigenvalues that are inherently smaller. This parameterization allows for exploration of changes in $\mat{W}$'s overall norm and its spectral shape. The figure contrasts the landscape topology for parameter regimes below and above an empirically determined critical value $\tau_{\text{crit}} \approx 0.400$ for this specific JEPA-like setup.

\begin{figure}[htbp]
    \centering
    \includegraphics[width=0.95\textwidth]{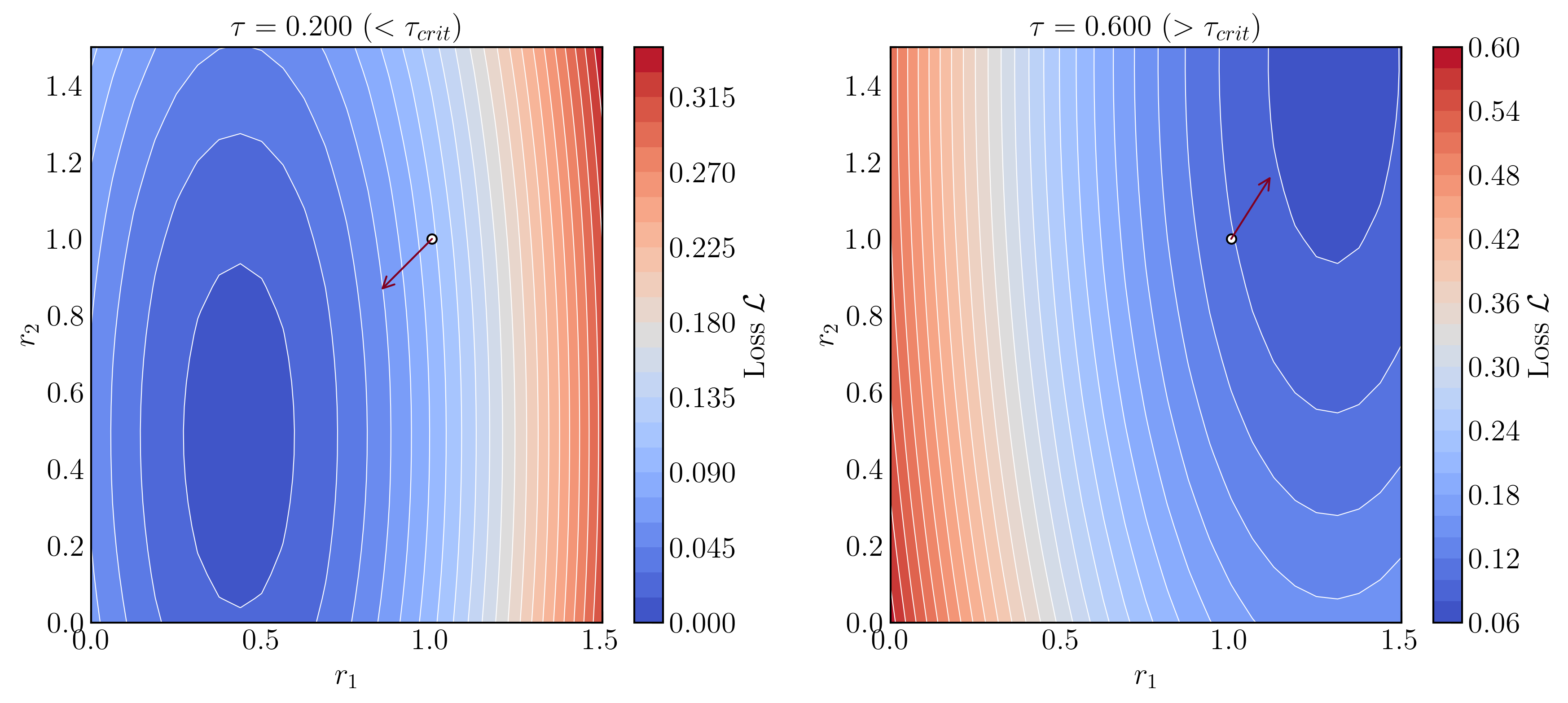} 
    \caption{\textbf{Loss landscape projection.} Two-dimensional projection of a JEPA-like loss landscape in the Jannis dataset. An empirical critical value of $\tau_{\text{crit}} \approx 0.400$ is noted for this illustration.
    (\textbf{Left}) Subcritical regime ($\tau = 0.200 < \tau_{\text{crit}}$): The loss landscape exhibits a global minimum closer to the origin ($r_1 \approx 0, r_2 \approx 0$), indicating the stability of the collapsed solution $\mat{W} \approx \mat{0}$. 
    (\textbf{Right}) Supercritical regime ($\tau = 0.600 > \tau_{\text{crit}}$): The origin has transformed into a region of higher loss, and new, distinct minima with $\mathcal{L} > 0$ have emerged away from the origin.}
    \label{fig:loss_landscape_2d}
\end{figure}

The key observations from these illustrative landscapes are:

\begin{itemize}
    \item \textbf{In the subcritical regime (Figure \ref{fig:loss_landscape_2d}, Left), where $\tau = 0.200 < \tau_{\text{crit}}$}:
    The loss function $\mathcal{L}$ displays a distinct global minimum at the origin ($r_1 \approx 0, r_2 \approx 0$). This topology signifies that any optimization trajectory, including one starting from the original $W$ ($r_1=1, r_2=1$), is naturally guided towards the collapsed state where the online encoder $\mat{W} \approx \mat{0}$. The landscape is characteristic of a system where this trivial solution is the most stable configuration. 

    \item \textbf{In the supercritical regime (Figure \ref{fig:loss_landscape_2d}, Right), where $\tau = 0.600 > \tau_{\text{crit}}$}:
    A qualitative transformation of the loss landscape is evident. The origin ($r_1 \approx 0, r_2 \approx 0$) no longer corresponds to a minimum loss value; instead, it has acquired characteristics of an unstable point, exhibiting a significantly higher loss compared to newly formed minima. These new minima, corresponding to non-trivial solutions $\mat{W}^* \neq \mat{0}$, are now located at finite distances from the origin.
\end{itemize}

These visualizations effectively demonstrate the bifurcation phenomenon's impact on the loss landscape. The transition from a landscape dominated by a single minimum at the origin to one featuring multiple, off-origin minima as $\tau$ crosses $\tau_{\text{crit}}$ is fundamental to understanding how JEPA and similar architectures can avoid representational collapse and learn useful features. 
\subsection{The criterion under joint predictor training}
\label{app:joint}

A JEPA trains the predictor together with the encoder. This section shows that the fixed-predictor
criterion remains the exact local criterion under joint training, and bounds the predictor's drift
until the encoder escapes.

\subsubsection{The collapsed set is a manifold of equilibria}

Write the full flow on $(\mat{W},\mat{W}_a,\mat{W}_p)$ with predictor learning rate $\eta_p$, in units
of the encoder learning rate. The predictor gradient of Appendix~\ref{app:adiabatic} is
\begin{equation}
\nabla_{\mat{W}_p}\mathcal{L}
 = \mat{W}_p\bigl(\mat{W}\mat{\Sigma}_{11}\mat{W}^\top\bigr)
 - \bar{\mat{M}}_2\mat{W}_a\mat{\Sigma}_{X1}\mat{W}^\top .
\label{eq:predgrad}
\end{equation}

\begin{lemma}[Equilibria, block structure, and speed]
\label{lem:equilibria}
Let $\mathcal{E} := \{(\mat{W},\mat{W}_a,\mat{W}_p) : \mat{W}=\mat{0},\ \mat{W}_a=\mat{0}\}$. Then
\begin{enumerate}[label=(\roman*),leftmargin=2em]
\item $\mathcal{E}$ is a $d^2$-dimensional manifold of equilibria of the full flow;
\item the linearisation at any point of $\mathcal{E}$ is block diagonal,
$\begin{psmallmatrix}\mathcal{J}(\mat{W}_p^0) & 0\\ 0 & \mat{0}_{d^2\times d^2}\end{psmallmatrix}$,
for every $\eta_p$, so Eq.~\eqref{eq:full_jacobian} is the exact linearisation of the coupled system and the
$d^2$ zero eigenvalues are the neutral directions tangent to $\mathcal{E}$;
\item at any state with $\|\mat{W}_p\|_2 \le w$ and $\max(\|\mat{W}\|_F,\|\mat{W}_a\|_F) \le \varepsilon$
the predictor's speed obeys
$\|\dot{\mat{W}}_p\|_F \le \eta_p\,(w\,s_\Sigma + p_{tgt}s_X)\,\varepsilon^2$, against
$\|\dot{\mat{W}}\|_F \le w(w s_\Sigma + p_{tgt}s_X)\varepsilon$, where
$s_\Sigma = \|\mat{\Sigma}_{11}\|_2$, $s_X = \|\mat{\Sigma}_{X1}\|_2$ and
$p_{tgt} = \|\bar{\mat{M}}_2\|_2$.
\end{enumerate}
\end{lemma}

\begin{proof}
(i) Every monomial of the three fields contains at least one factor $\mat{W}$ or $\mat{W}_a$, so all
three vanish at $\mat{W}=\mat{W}_a=\mat{0}$ for every $\mat{W}_p$.

(ii) Differentiate the three components at $(\mat{0},\mat{0},\mat{W}_p^0)$. The partial of the encoder
field in $(\mat{W},\mat{W}_a)$ vectorises to $\mathcal{J}(\mat{W}_p^0)$; its partial in $\mat{W}_p$ is
$-(\delta\mat{W}_p^\top\mat{W}_p^0 + (\mat{W}_p^0)^\top\delta\mat{W}_p)\mat{W}\mat{\Sigma}_{11}
+ \delta\mat{W}_p^\top\bar{\mat{M}}_2\mat{W}_a\mat{\Sigma}_{X1}$, which vanishes at the base point,
killing the upper-right block. The target field is $\beta(\delta\mat{W}-\delta\mat{W}_a)$ and carries no
$\mat{W}_p$ dependence. Every term of the predictor field is at least bilinear in
$(\mat{W},\mat{W}_a)$ by Eq.~\eqref{eq:predgrad}, so every term of its derivative keeps a factor $\mat{W}$ or
$\mat{W}_a$ and vanishes; both of its partials are zero, for every $\eta_p$, which kills the bottom row.

(iii) Both summands of Eq.~\eqref{eq:predgrad} are bilinear in $(\mat{W},\mat{W}_a)$, and bounding each in
the spectral norm gives the stated inequality.
\end{proof}

\subsubsection{How far the predictor drifts before the encoder escapes}

Write $\vect{z} = [\mathrm{vec}(\mat{W})^\top, \mathrm{vec}(\mat{W}_a)^\top]^\top$ with
$\|\vect{z}\|^2 = \|\mat{W}\|_F^2 + \|\mat{W}_a\|_F^2$, $\varepsilon := \|\vect{z}(0)\|$,
$\mathcal{J}_0 := \mathcal{J}(\mat{W}_p^0)$, $\mat{D}(t) := \mat{W}_p(t) - \mat{W}_p^0$ and
$\rho(t) := \|\mat{D}(t)\|_F$. Then
$\dot{\vect{z}} = \mathcal{J}(\mat{W}_p(t))\,\vect{z}$ exactly. The result below assumes the exact form of
Assumption~\ref{ass:spectral_decoupling} with a real drive, at initialisation only: orthogonal $\mat{U}$ and $\mat{V}$ with
$\mat{\Sigma}_{11} = \mat{U}\mathrm{diag}(a_k)\mat{U}^\top$,
$\mat{\Sigma}_{X1} = \mat{U}\mathrm{diag}(c_k)\mat{U}^\top$,
$\mat{W}_p^{0\top}\mat{W}_p^0 = \mat{V}\mathrm{diag}(p_\ell)\mat{V}^\top$ and
$\mat{W}_p^{0\top}\bar{\mat{M}}_2 = \mat{V}\mathrm{diag}(q_\ell)\mat{V}^\top$, where $a_k > 0$,
$p_\ell > 0$ and $c_k, q_\ell$ are real. Then $\mathcal{J}_0$ splits into the $2\times 2$ blocks
$\begin{psmallmatrix}-\sigma_{k\ell} & \gamma_{k\ell}\\ \beta & -\beta\end{psmallmatrix}$ with
$\sigma_{k\ell} = a_k p_\ell$, $\gamma_{k\ell} = c_k q_\ell$ and discriminant
$\Delta_{k\ell} = (\beta-\sigma_{k\ell})^2 + 4\beta\gamma_{k\ell}$. Decoupling is not imposed at
$t > 0$, so the statements below concern the spectrum of $\mathcal{J}(\mat{W}_p(t))$.

\begin{theorem}[Drift bound]
\label{thm:drift}
Assume exact decoupling at initialisation as above, $\Delta_{k\ell} \neq 0$ for every mode, and
$\mathcal{U} \neq \emptyset$. Let $\lambda_* > 0$ be the spectral abscissa of $\mathcal{J}_0$. Let
$\bar{\mat{O}}$ be the orthogonal change of basis that block-diagonalises $\mathcal{J}_0$, let
$\widehat{\mat{X}}$ be the block-diagonal similarity that brings each block to diagonal form with
unit-norm columns (real blocks) or to real rotation form (complex blocks), and set
$\kappa := \|\widehat{\mat{X}}\|_2\|\widehat{\mat{X}}^{-1}\|_2$. With
$\vect{y}(0) := \widehat{\mat{X}}^{-1}\bar{\mat{O}}^\top\vect{z}(0)$ and $\mat{\Pi}_*$ the coordinate
projection onto the eigenvalue $\lambda_*$, let $\alpha := \|\mat{\Pi}_*\vect{y}(0)\|/\|\vect{y}(0)\| > 0$
be the alignment of the initialisation. Set $C_p := 2w_0 s_\Sigma + p_{tgt}s_X$ and
$L_E := 3w_0 s_\Sigma + p_{tgt}s_X$ with $w_0 = \|\mat{W}_p^0\|_2$ and $p_{tgt} = \|\bar{\mat{M}}_2\|_2$,
$L_\alpha := \log(4\kappa^2/\alpha)$, and
\begin{equation}
\rho_\infty := \frac{\eta_p\,C_p\,(1 + L_\alpha)}{\lambda_*}\,\delta^2, \qquad
\bar T := \frac{1}{\lambda_*}\log\frac{4\kappa\delta}{\alpha\varepsilon}.
\end{equation}
If (C0) $\kappa\varepsilon \le \delta$, (C1) $2\rho_\infty \le w_0$ and
(C2) $2\kappa L_E\rho_\infty\bar T \le \log(1+\alpha/2)$, then the first time $T_\delta$ at which
$\|\vect{z}\|$ exceeds $\delta$ satisfies $T_\delta \le \bar T$, and
\begin{equation}
\bigl\|\mat{W}_p(t) - \mat{W}_p^0\bigr\|_F \;\le\; \rho_\infty \qquad \text{for all } t \le T_\delta .
\end{equation}
The bound does not depend on the initialisation scale $\varepsilon$.
\end{theorem}

\begin{proof}
Write $\bar\rho := 2\rho_\infty$ and $c := \kappa L_E\bar\rho$. We use
$\|\mat{A}\mat{B}\mat{C}\|_F \le \|\mat{A}\|_2\|\mat{B}\|_F\|\mat{C}\|_2$ and
$\|\mat{A}\otimes\mat{B}\|_2 = \|\mat{A}\|_2\|\mat{B}\|_2$.

Step 0, coordinates. Let $\mat{\Lambda} := \widehat{\mat{X}}^{-1}\bar{\mat{O}}^\top\mathcal{J}_0\bar{\mat{O}}\widehat{\mat{X}}$.
Every block of $\mat{\Lambda}$ is normal and the largest real part of its spectrum is $\lambda_*$, so
$\|e^{\mat{\Lambda}t}\|_2 = e^{\lambda_* t}$ for $t \ge 0$. With
$\vect{y}(t) := \widehat{\mat{X}}^{-1}\bar{\mat{O}}^\top\vect{z}(t)$ and
$\mat{E}(t) := \mathcal{J}(\mat{W}_p(t)) - \mathcal{J}_0$ and $\widetilde{\mat{E}}(t) := \widehat{\mat{X}}^{-1}\bar{\mat{O}}^\top\mat{E}(t)\bar{\mat{O}}\widehat{\mat{X}}$,
\begin{equation}
\dot{\vect{y}} = \mat{\Lambda}\vect{y} + \widetilde{\mat{E}}(t)\vect{y}, \qquad
\|\widetilde{\mat{E}}(t)\|_2 \le \kappa\|\mat{E}(t)\|_2, \qquad
\varepsilon/\|\widehat{\mat{X}}\|_2 \le \|\vect{y}(0)\| \le \|\widehat{\mat{X}}^{-1}\|_2\,\varepsilon .
\end{equation}

Step 1, perturbation of the Jacobian. If $\rho(t) \le \bar\rho \le w_0$, then
$\|\mat{W}_p(t)\|_2 \le 2w_0$, and since
$\mat{W}_p^\top\mat{W}_p - \mat{W}_p^{0\top}\mat{W}_p^0 = \mat{W}_p^\top\mat{D} + \mat{D}^\top\mat{W}_p^0$,
\begin{equation}
\|\mat{E}(t)\|_2 \le s_\Sigma\,\|\mat{W}_p^\top\mat{W}_p - \mat{W}_p^{0\top}\mat{W}_p^0\|_2
 + s_X\,p_{tgt}\,\|\mat{D}\|_2 \le L_E\,\rho(t).
\end{equation}

Step 2, drift generator. Under the same condition, Lemma~\ref{lem:equilibria}(iii) with
$w = 2w_0$ and $\varepsilon$ replaced by $\|\vect{z}(t)\|$, which bounds both $\|\mat{W}\|_F$ and
$\|\mat{W}_a\|_F$, gives $\|\dot{\mat{W}}_p\|_F \le \eta_p C_p\|\vect{z}(t)\|^2$, hence
$\rho(t) \le \eta_p C_p\int_0^t\|\vect{z}(s)\|^2\,ds$.

Step 3, bootstrap. Let $\tau := \min(T_\delta, \bar T)$ and let $\mathcal{B}$ be the set of
$T \in [0,\tau]$ with $\sup_{[0,T]}\rho \le \bar\rho$. It contains $0$ and is closed by continuity.
Fix $T \in \mathcal{B}$. On $[0,T]$, Steps 1 and 2 apply and $\|\widetilde{\mat{E}}\|_2 \le c$. By
Duhamel's formula, $\phi(t) := e^{-\lambda_* t}\|\vect{y}(t)\|$ obeys
$\phi(t) \le \phi(0) + c\int_0^t\phi(s)\,ds$, and Gr\"onwall's inequality gives
$\|\vect{y}(t)\| \le \|\vect{y}(0)\|e^{\tilde\lambda t}$ with $\tilde\lambda := \lambda_* + c$, so
$\|\vect{z}(t)\| \le \kappa\varepsilon e^{\tilde\lambda t}$. Also $\|\vect{z}\| \le \delta$ on
$[0,T]$. Put $L := \log(\delta/(\kappa\varepsilon)) \ge 0$ by (C0) and $s^* := L/\tilde\lambda$, and
split the drift integral at $s^*$:
\begin{equation}
\int_0^t\|\vect{z}\|^2\,ds \le \int_0^{s^*}\kappa^2\varepsilon^2e^{2\tilde\lambda s}\,ds + \delta^2(t-s^*)_+
\le \frac{\delta^2}{2\tilde\lambda} + \delta^2(t-s^*)_+ .
\end{equation}
Since $\bar T = (L + L_\alpha)/\lambda_*$ and (C2) gives $c \le \lambda_*\log(1+\alpha/2)/(L+L_\alpha)$,
\begin{equation}
(t-s^*)_+ \le \bar T - s^* = \frac{L\,c}{\lambda_*\tilde\lambda} + \frac{L_\alpha}{\lambda_*}
\le \frac{\log(1+\alpha/2)}{\tilde\lambda} + \frac{L_\alpha}{\lambda_*}
\le \frac{1/2 + L_\alpha}{\lambda_*},
\end{equation}
using $\tilde\lambda \ge \lambda_*$ and $\log(1+\alpha/2) \le \log\tfrac32 < \tfrac12$ for
$\alpha \le 1$. Therefore $\rho(t) \le \eta_p C_p\delta^2(1+L_\alpha)/\lambda_* = \rho_\infty < \bar\rho$
on $[0,T]$, and by continuity the strict inequality persists slightly beyond $T$. So $\mathcal{B}$ is
open in $[0,\tau]$, hence equal to it, and the drift bound holds on $[0,\tau]$.

Step 4, the encoder escapes. Suppose $T_\delta > \bar T$, so that Step 3 holds on
$[0,\bar T]$. The coordinates selected by $\mat{\Pi}_*$ are real eigendirections with
$\mat{\Lambda}\mat{\Pi}_* = \lambda_*\mat{\Pi}_*$, because $\mu_{k\ell} > 1$ forces
$\Delta_{k\ell} > 0$. By Duhamel's formula and the growth bound of Step 3,
\begin{equation}
\|\mat{\Pi}_*\vect{y}(t)\| \ge e^{\lambda_* t}\alpha\|\vect{y}(0)\|
 - c\,\|\vect{y}(0)\|\int_0^t e^{\lambda_*(t-s)}e^{\tilde\lambda s}\,ds
 = e^{\lambda_* t}\|\vect{y}(0)\|\bigl[\alpha - (e^{ct} - 1)\bigr].
\end{equation}
By (C2), $e^{ct} - 1 \le \alpha/2$ for $t \le \bar T$, so
$\|\mat{\Pi}_*\vect{y}(t)\| \ge (\alpha/2)\|\vect{y}(0)\|e^{\lambda_* t}$. Combining this with
$\|\vect{z}(t)\| \ge \|\vect{y}(t)\|/\|\widehat{\mat{X}}^{-1}\|_2$ and the lower bound on
$\|\vect{y}(0)\|$ from Step 0 gives $\|\vect{z}(t)\| \ge (\alpha\varepsilon/2\kappa)\,e^{\lambda_* t}$.
At $t = \bar T$ the right-hand side equals $2\delta$, which contradicts $T_\delta > \bar T$. Hence
$T_\delta \le \bar T$, $\tau = T_\delta$, and the drift bound holds up to $T_\delta$.
\end{proof}

\begin{corollary}[The classification survives]
\label{cor:margin}
Under the hypotheses of Theorem~\ref{thm:drift}, let
$g_0 := \min\{|\mathrm{Re}\,\lambda| : \lambda \in \Spec(\mathcal{J}_0)\}$ be the distance of the
initial spectrum to the imaginary axis, and let $\lambda^+_{k\ell}$ denote the larger root of mode
$(k,\ell)$. If in addition $\kappa L_E\rho_\infty < g_0$, then for all $t \le T_\delta$ the Jacobian
$\mathcal{J}(\mat{W}_p(t))$ has no eigenvalue on the imaginary axis, has exactly $|\mathcal{U}|$
eigenvalues in the open right half plane, and each of them lies within $\kappa L_E\rho_\infty$ of some
$\lambda^+_{k\ell}$ with $(k,\ell) \in \mathcal{U}$. The condition $\kappa L_E\rho_\infty < g_0$ holds
whenever $\kappa L_E\rho_\infty < (\beta+\sigma_{k\ell})/2$ on every mode with $\Delta_{k\ell} < 0$ and
$|\mu_{k\ell} - 1| > C'\delta^2$ on every mode with $\Delta_{k\ell} > 0$, where
\begin{equation}
C' := \frac{\eta_p\,\kappa\,L_E\,C_p\,(1 + L_\alpha)}{\lambda_*}\,
\max_{k\ell}\frac{\beta + \sigma_{k\ell} + \sqrt{\beta\sigma_{k\ell}(\mu_{k\ell})_+}}{\beta\sigma_{k\ell}} .
\end{equation}
\end{corollary}

\begin{proof}
By Theorem~\ref{thm:drift} and Step 1 of its proof, $\|\mat{E}(t)\|_2 \le L_E\rho_\infty =: e_{\max}$ for
$t \le T_\delta$, and the hypothesis reads $\kappa e_{\max} < g_0$. The matrix $\mathcal{J}_0$ is
diagonalisable, and on complex blocks the real rotation form differs from a complex eigenbasis by
$1/\sqrt2$ times a unitary matrix, so $\kappa$ is the condition number of an eigenvector matrix of
$\mathcal{J}_0$. The Bauer--Fike theorem \citep{bauerNormsExclusionTheorems1960} then places every eigenvalue $\hat\lambda$ of
$\mathcal{J}(\mat{W}_p(t)) = \mathcal{J}_0 + \mat{E}(t)$ within $\kappa e_{\max}$ of some eigenvalue
$\lambda_j$ of $\mathcal{J}_0$.

First, no eigenvalue lies on the axis: if $\hat\lambda$ were purely imaginary, then
$|\hat\lambda - \lambda_j| \ge |\mathrm{Re}\,\lambda_j| \ge g_0 > \kappa e_{\max}$ for every $j$, a
contradiction.

Second, the count in the right half plane is constant. The eigenvalues of $\mathcal{J}(\mat{W}_p(t))$ depend continuously on $t$ and
stay in the disc of radius $R := \sup_{t \le T_\delta}\|\mathcal{J}(\mat{W}_p(t))\|_2$. None of them
meets the boundary of $\{\mathrm{Re}\,\lambda > 0\} \cap \{|\lambda| < R+1\}$, so the number inside
is locally constant in $t$ and hence constant on $[0,T_\delta]$. At $t = 0$ a mode with
$\mu_{k\ell} > 1$ has roots of product $\beta\sigma_{k\ell}(1-\mu_{k\ell}) < 0$ and contributes one
positive root; a real mode with $\mu_{k\ell} < 1$ has positive product and negative sum
$-(\beta+\sigma_{k\ell})$ and contributes none; a complex mode has real part $-(\beta+\sigma_{k\ell})/2$
and contributes none. The count is therefore $|\mathcal{U}|$.

Third, localisation. If $\mathrm{Re}\,\hat\lambda > 0$, its Bauer--Fike partner satisfies
$\mathrm{Re}\,\lambda_j > -\kappa e_{\max} > -g_0$, hence $\mathrm{Re}\,\lambda_j \ge g_0 > 0$, and the
only eigenvalues of $\mathcal{J}_0$ in the right half plane are the $\lambda^+_{k\ell}$ with
$(k,\ell) \in \mathcal{U}$.

Finally, the margin form. Fix a mode. If $\Delta_{k\ell} < 0$, both roots have real part
$-(\beta+\sigma_{k\ell})/2$. If $\Delta_{k\ell} > 0$, both roots are real with product
$\beta\sigma_{k\ell}(1-\mu_{k\ell})$. Since
$\sqrt{\Delta_{k\ell}} \le |\beta-\sigma_{k\ell}| + 2\sqrt{\beta\sigma_{k\ell}(\mu_{k\ell})_+}$, the
larger modulus is at most $\beta + \sigma_{k\ell} + \sqrt{\beta\sigma_{k\ell}(\mu_{k\ell})_+}$, and the
smaller one is at least
\begin{equation}
g_{k\ell} := \frac{\beta\sigma_{k\ell}\,|1-\mu_{k\ell}|}{\beta + \sigma_{k\ell} + \sqrt{\beta\sigma_{k\ell}(\mu_{k\ell})_+}} .
\end{equation}
So $g_0$ is at least the minimum of $g_{k\ell}$ over real modes and of $(\beta+\sigma_{k\ell})/2$ over
complex ones. On a real mode, $|\mu_{k\ell} - 1| > C'\delta^2$ and the definition of $C'$ give
$g_{k\ell} > \kappa L_E\rho_\infty$, because
$\rho_\infty = \eta_p C_p(1+L_\alpha)\delta^2/\lambda_*$. Together with the condition on complex modes
this yields $g_0 > \kappa L_E\rho_\infty$.
\end{proof}

Condition (C2) holds when $\varepsilon \ge (4\kappa\delta/\alpha)\,e^{-K_2/\delta^2}$ with
$K_2 := \lambda_*^2\log(1+\alpha/2)/(2\kappa L_E\eta_p C_p(1+L_\alpha))$: $\varepsilon$ may be
exponentially small in $1/\delta^2$ but not zero at fixed $\delta$.

\subsubsection{Numerical check with the predictor trained concurrently}

With the predictor trained concurrently from its initialisation, the sign of the leading eigenvalue
of the frozen-predictor Jacobian of Eq.~\eqref{eq:full_jacobian} at initialisation classified escape
against collapse in $104$ of $104$ runs at equal learning rates, $104$ of $104$ at three times the
predictor rate, and $102$ of $104$ at ten times; both disagreements lie at $|\lambda_{\mathrm{init}}| \le 0.004$.

\section{Training Details}
\label{app:training_details}

This section presents the implementation details of the project.

\subsection{Programming environment}

The experimental codebase, available at \url{https://github.com/jose-melo/drive-vs-decay}, was implemented in Python and relies on several widely used scientific computing and deep learning libraries. Table~\ref{tab:env} summarizes the main dependencies and their corresponding versions. All experiments were conducted on a single NVIDIA HGX A100 GPU with 40\,GB of memory.

\begin{table}[htb]
    \centering
    \caption{Main libraries used in the project.}
    \label{tab:env}
    \scalebox{0.85}{
        \begin{tabular}{ll}
            \toprule
            \textbf{Library} & \textbf{Description} \\
            \midrule
            \texttt{\textbf{Python} v3.12} & Core programming language used for all experiments \\
            \texttt{\textbf{pytorch}} $\geq$ 2.0.0 & Deep learning framework used to implement all models \\
            \texttt{\textbf{numpy}} $\geq$ 1.24.0 & Numerical computing and array operations \\
            \texttt{\textbf{sciPy}} $\geq$ 1.10.0 & Scientific computing utilities and linear algebra \\
            \texttt{\textbf{pandas}} $\geq$ 2.0.0 & Tabular data processing and experiment management \\
            \texttt{\textbf{scikit-learn}} $\geq$ 1.2.0 & Baseline models and evaluation utilities \\
            \texttt{\textbf{pytorch-lightning}} $\geq$ 2.0.0 & Training loops and experiment orchestration \\
            \texttt{\textbf{matplotlib}} $\geq$ 3.7.0 & Visualization and figure generation \\
            \texttt{\textbf{wandb}} $\geq$ 0.15.0 & Experiment tracking and logging \\
            \texttt{\textbf{tqdm}} $\geq$ 4.65.0 & Progress bars for long-running experiments \\
            \bottomrule
        \end{tabular}
    }
\end{table}

\subsection{Tabular-domain experiments}

We conduct experiments on three standard tabular benchmarks: \texttt{Jannis}, \texttt{Helena}, and \texttt{ALOI}, which cover a broad range of dataset scales, feature dimensionalities, and classification complexities. \texttt{Helena} is a medium-scale, high-dimensional dataset with 100 classes and purely numerical features, testing multi-class discrimination under moderate sample sizes. \texttt{ALOI} is a large-scale, high-cardinality classification benchmark with 1{,}000 classes and 128 numerical features, posing a challenging regime for representation learning and scalability. Together, these datasets enable a systematic evaluation across increasing dimensionality, class cardinality, and computational cost.

\begin{table}[h!]
    \centering
    \caption{Dataset characteristics}
    \label{tab:data_characteristics}
    \scalebox{0.9}{
    \begin{tabular}{cccc}
    \toprule
        & \texttt{Jannis} & \texttt{Helena} & \texttt{ALOI} \\
    \midrule
        Samples & 83,733 & 65,196 & 108,000 \\
        Numerical & $54$ & $27$ & $128$ \\
        Classes & $4$ & $100$ & $1{,}000$ \\
        Metric & Accuracy & Accuracy & Accuracy \\
    \bottomrule
    \end{tabular}
    }
\end{table}

\subsection{Image-domain experiments}
\label{app:image-experiments}

For the image experiments, we pretrain I-JEPA \citep{assranSelfSupervisedLearningImages2023} on three image benchmarks at modest resolution:
\begin{itemize}
    \item \textbf{CIFAR-10} ($32\!\times\!32$, 10 classes, 50k train images);
    \item \textbf{CIFAR-100} ($32\!\times\!32$, 100 classes, 50k train images);
    \item \textbf{STL-10} ($96\!\times\!96$, 10 classes, 5k labeled + 100k unlabeled).
\end{itemize}

\section{Evaluation Metrics}
\label{app:evaluation_metrics}

\subsection{Quantifying Representational Collapse.}

A central failure mode of Joint Embedding Predictive Architectures and related self-supervised systems is \emph{representational collapse}, where learned embeddings concentrate in a low-dimensional subspace or degenerate toward rank-deficient solutions. Such collapse is not necessarily reflected by the training loss alone, as non-contrastive objectives may admit low-loss solutions with severely impoverished geometry. Consequently, evaluation must explicitly probe the \emph{intrinsic dimensionality} and spectral diversity of the learned representations. To this end, we adopt a spectral measure that captures how information is distributed across orthogonal directions in the embedding space, thereby providing a principled and architecture-agnostic proxy for collapse.

\paragraph{Effective Rank.}
Let $\mat{Z} \in \mathbb{R}^{n \times d}$ denote a matrix of embeddings, obtained by stacking $n$ representation vectors of dimension $d$. Let
\begin{equation}
\mat{Z} = \mat{U} \, \diag(\sigma_1,\ldots,\sigma_r)\, \mat{V}^\top
\end{equation}
be the singular value decomposition of $\mat{Z}$, where $r = \operatorname{rank}(\mat{Z})$ and $\sigma_1 \ge \cdots \ge \sigma_r \ge 0$ are the non-zero singular values. We define a normalized spectral distribution
\begin{equation}
p_i \;\coloneqq\; \frac{\sigma_i}{\sum_{j=1}^r \sigma_j}, 
\qquad i = 1,\ldots,r,
\end{equation}
which can be interpreted as a probability mass function over singular modes, reflecting how representation energy is allocated across orthogonal directions.

The \emph{effective rank} is then defined as the exponential of the Shannon entropy of this distribution:
\begin{equation}
\label{eq:effective_rank}
\mathrm{ERank}(\mat{Z})
\;\coloneqq\;
\exp\!\left( 
- \sum_{i=1}^r p_i \log p_i
\right).
\end{equation}

\paragraph{Interpretation and Theoretical Properties.}
The effective rank admits several desirable properties for diagnosing collapse:
\begin{itemize}
    \item \textbf{Upper and lower bounds.} One has $1 \le \mathrm{ERank}(\mat{Z}) \le r$. The lower bound is achieved when all mass concentrates on a single singular value (complete collapse to a one-dimensional subspace), while the upper bound is achieved when the spectrum is uniform, $p_i = 1/r$, corresponding to maximally spread representations.
    \item \textbf{Sensitivity to spectral concentration.} Unlike the algebraic rank, which is binary and insensitive to magnitude, $\mathrm{ERank}$ continuously penalizes dominance by a small number of singular modes. This makes it well-suited to detect \emph{soft collapse}, where representations remain formally full-rank but are effectively low-dimensional.
    \item \textbf{Scale invariance.} Since $\{p_i\}$ is normalized, $\mathrm{ERank}$ is invariant to global rescaling of representations, isolating geometric collapse from trivial norm effects.
\end{itemize}

\paragraph{Relation to Collapse Dynamics.}
From a dynamical perspective, collapse corresponds to a regime in which training amplifies only a small subset of eigendirections of the representation covariance while suppressing others. In this regime, the spectral distribution $\{p_i\}$ becomes sharply peaked, reducing entropy and driving $\mathrm{ERank}$ toward its minimum. Conversely, stable learning that preserves multiple informative modes yields a flatter spectrum and higher effective rank. As such, $\mathrm{ERank}$ provides a direct quantitative probe of whether training dynamics maintain a rich, multi-dimensional representation geometry or drift toward degenerate fixed points.

Throughout our experiments, we report effective rank as a primary diagnostic for representational collapse, complementing downstream performance metrics and enabling a fine-grained analysis of how different stabilization strategies affect the intrinsic dimensionality of learned embeddings.
\section{Full Results}
\label{app:full_results}

The aggregated grid over context ratio and predictor scale is Table~\ref{tab:jepa_grid_results_merged} in the main text; the per-dataset tables follow.

\subsection{Tabular-JEPA Variants}
\begin{table*}[t]
\centering
\tiny
\setlength{\tabcolsep}{2pt}
\renewcommand{\arraystretch}{1.1}

\caption{Detailed results for Tabular--JEPA variants across ALOI, Helena, and Jannis datasets. Each row shows results for a specific context ratio and predictor scale configuration, with means and standard errors (SEM) aggregated over seeds. Collapse statistics ($\mu$ quartiles), effective rank (Erank), final training loss, and test accuracy are reported.}
\label{tab:detailed_tabular_jepa_a}

\resizebox{\linewidth}{!}{%
\begin{tabular}{llcc||cc cc||cc cc||c c}
\toprule
\multirow{2}{*}{Dataset} &
\multirow{2}{*}{Context} &
\multirow{2}{*}{\shortstack{Predictor\\scale}} &
\multirow{2}{*}{\shortstack{Config\\ID}} &
\multicolumn{2}{c}{\shortstack{Epoch 0\\$\mu_{50}$}} &
\multicolumn{2}{c||}{\shortstack{Epoch 0\\Erank}} &
\multicolumn{2}{c}{\shortstack{Last epoch\\$\mu_{50}$}} &
\multicolumn{2}{c||}{\shortstack{Last epoch\\Erank}} &
\multirow{2}{*}{\shortstack{Final\\loss}} &
\multirow{2}{*}{\shortstack{Test\\Acc.}} \\
\cmidrule(lr){5-6}\cmidrule(lr){7-8}\cmidrule(lr){9-10}\cmidrule(lr){11-12}
& & & & Mean & SEM & Mean & SEM & Mean & SEM & Mean & SEM & & \\
\midrule
ALOI & 0.2 & 0.01 & 0 & $1.46$ & $0.029$ & $19.44$ & $0.17$ & $0.64$ & $0.006$ & $9.45$ & $0.06$ & $0.0350$ & $0.699$ \\
Helena & 0.2 & 0.01 & 0 & $20.06$ & $1.520$ & $1.40$ & $0.00$ & $4.84$ & $0.361$ & $2.01$ & $0.07$ & $0.0106$ & $0.229$ \\
Jannis & 0.2 & 0.01 & 0 & $2.01$ & $0.039$ & $11.41$ & $0.22$ & $1.92$ & $0.015$ & $6.12$ & $0.19$ & $0.0425$ & $0.658$ \\
ALOI & 0.2 & 0.05 & 1 & $0.68$ & $0.007$ & $20.50$ & $0.30$ & $0.64$ & $0.003$ & $7.65$ & $0.08$ & $0.0219$ & $0.668$ \\
Helena & 0.2 & 0.05 & 1 & $5.83$ & $0.737$ & $1.52$ & $0.02$ & $6.48$ & $1.336$ & $1.51$ & $0.03$ & $0.0041$ & $0.179$ \\
Jannis & 0.2 & 0.05 & 1 & $0.86$ & $0.009$ & $11.12$ & $0.31$ & $1.00$ & $0.007$ & $4.52$ & $0.08$ & $0.0280$ & $0.639$ \\
ALOI & 0.2 & 0.10 & 2 & $0.41$ & $0.006$ & $18.01$ & $0.09$ & $0.41$ & $0.007$ & $6.66$ & $0.04$ & $0.0153$ & $0.650$ \\
Helena & 0.2 & 0.10 & 2 & $1.79$ & $0.044$ & $1.60$ & $0.03$ & $1.45$ & $0.300$ & $1.27$ & $0.04$ & $0.0013$ & $0.164$ \\
Jannis & 0.2 & 0.10 & 2 & $0.45$ & $0.007$ & $10.19$ & $0.21$ & $0.49$ & $0.006$ & $3.81$ & $0.08$ & $0.0281$ & $0.623$ \\
ALOI & 0.2 & 0.20 & 3 & $0.21$ & $0.003$ & $14.50$ & $0.09$ & $0.18$ & $9.21\cdot 10^{-4}$ & $7.28$ & $0.04$ & $0.0212$ & $0.716$ \\
Helena & 0.2 & 0.20 & 3 & $0.94$ & $0.120$ & $1.52$ & $0.04$ & $2.49$ & $1.587$ & $1.14$ & $0.00$ & $2.95\cdot 10^{-4}$ & $0.162$ \\
Jannis & 0.2 & 0.20 & 3 & $0.22$ & $0.002$ & $9.73$ & $0.22$ & $0.22$ & $0.002$ & $3.76$ & $0.07$ & $0.0304$ & $0.608$ \\
ALOI & 0.2 & 0.30 & 4 & $0.15$ & $0.003$ & $13.02$ & $0.20$ & $0.11$ & $3.74\cdot 10^{-4}$ & $6.85$ & $0.06$ & $0.0232$ & $0.723$ \\
Helena & 0.2 & 0.30 & 4 & $0.69$ & $0.058$ & $1.43$ & $0.03$ & $2.98$ & $1.225$ & $1.10$ & $0.01$ & $2.39\cdot 10^{-4}$ & $0.154$ \\
Jannis & 0.2 & 0.30 & 4 & $0.15$ & $0.003$ & $8.82$ & $0.27$ & $0.14$ & $0.001$ & $3.53$ & $0.10$ & $0.0282$ & $0.607$ \\
\bottomrule
\end{tabular}%
}
\end{table*}

\begin{table*}[t]
\centering
\tiny
\setlength{\tabcolsep}{2pt}
\renewcommand{\arraystretch}{1.1}

\caption{Detailed results for Tabular--JEPA variants across ALOI, Helena, and Jannis datasets (continued).}
\label{tab:detailed_tabular_jepa_b}

\resizebox{\linewidth}{!}{%
\begin{tabular}{llcc||cc cc||cc cc||c c}
\toprule
\multirow{2}{*}{Dataset} &
\multirow{2}{*}{Context} &
\multirow{2}{*}{\shortstack{Predictor\\scale}} &
\multirow{2}{*}{\shortstack{Config\\ID}} &
\multicolumn{2}{c}{\shortstack{Epoch 0\\$\mu_{50}$}} &
\multicolumn{2}{c||}{\shortstack{Epoch 0\\Erank}} &
\multicolumn{2}{c}{\shortstack{Last epoch\\$\mu_{50}$}} &
\multicolumn{2}{c||}{\shortstack{Last epoch\\Erank}} &
\multirow{2}{*}{\shortstack{Final\\loss}} &
\multirow{2}{*}{\shortstack{Test\\Acc.}} \\
\cmidrule(lr){5-6}\cmidrule(lr){7-8}\cmidrule(lr){9-10}\cmidrule(lr){11-12}
& & & & Mean & SEM & Mean & SEM & Mean & SEM & Mean & SEM & & \\
\midrule
ALOI & 0.2 & 0.50 & 5 & $0.09$ & $0.002$ & $11.11$ & $0.04$ & $0.06$ & $9.17\cdot 10^{-4}$ & $5.47$ & $0.05$ & $0.0198$ & $0.689$ \\
Helena & 0.2 & 0.50 & 5 & $1.14$ & $0.091$ & $1.37$ & $0.01$ & $13.63$ & $3.757$ & $1.06$ & $0.00$ & $1.73\cdot 10^{-4}$ & $0.128$ \\
Jannis & 0.2 & 0.50 & 5 & $0.09$ & $0.002$ & $8.26$ & $0.13$ & $0.08$ & $0.001$ & $3.09$ & $0.06$ & $0.0203$ & $0.585$ \\
ALOI & 0.2 & 1.00 & 6 & $0.05$ & $9.99\cdot 10^{-4}$ & $9.52$ & $0.22$ & $0.03$ & $3.51\cdot 10^{-4}$ & $3.96$ & $0.15$ & $0.0157$ & $0.625$ \\
Helena & 0.2 & 1.00 & 6 & $0.67$ & $0.153$ & $1.36$ & $0.02$ & $24.84$ & $1.856$ & $1.04$ & $0.00$ & $1.34\cdot 10^{-4}$ & $0.119$ \\
Jannis & 0.2 & 1.00 & 6 & $0.04$ & $9.46\cdot 10^{-4}$ & $7.87$ & $0.07$ & $0.04$ & $0.001$ & $2.27$ & $0.04$ & $0.0100$ & $0.572$ \\
ALOI & 0.3 & 0.01 & 7 & $1.23$ & $0.041$ & $20.53$ & $0.11$ & $0.58$ & $0.002$ & $12.02$ & $0.07$ & $0.0404$ & $0.744$ \\
Helena & 0.3 & 0.01 & 7 & $28.36$ & $3.318$ & $1.60$ & $0.07$ & $2.93$ & $0.178$ & $1.95$ & $0.01$ & $0.0090$ & $0.239$ \\
Jannis & 0.3 & 0.01 & 7 & $1.40$ & $0.057$ & $14.65$ & $0.35$ & $0.54$ & $0.089$ & $25.92$ & $4.06$ & $0.5441$ & $0.703$ \\
ALOI & 0.3 & 0.05 & 8 & $0.66$ & $0.001$ & $20.57$ & $0.20$ & $0.58$ & $0.001$ & $9.99$ & $0.13$ & $0.0246$ & $0.715$ \\
Helena & 0.3 & 0.05 & 8 & $6.51$ & $0.751$ & $1.63$ & $0.02$ & $4.64$ & $1.023$ & $1.52$ & $0.00$ & $0.0039$ & $0.174$ \\
Jannis & 0.3 & 0.05 & 8 & $0.82$ & $0.004$ & $11.67$ & $0.36$ & $0.84$ & $0.005$ & $6.93$ & $0.32$ & $0.0635$ & $0.658$ \\
ALOI & 0.3 & 0.10 & 9 & $0.39$ & $0.007$ & $17.79$ & $0.04$ & $0.41$ & $0.004$ & $8.14$ & $0.19$ & $0.0158$ & $0.690$ \\
Helena & 0.3 & 0.10 & 9 & $2.48$ & $0.253$ & $1.61$ & $0.01$ & $2.15$ & $0.191$ & $1.34$ & $0.03$ & $0.0022$ & $0.149$ \\
Jannis & 0.3 & 0.10 & 9 & $0.45$ & $0.009$ & $10.54$ & $0.12$ & $0.46$ & $0.007$ & $4.63$ & $0.05$ & $0.0363$ & $0.637$ \\
\bottomrule
\end{tabular}%
}
\end{table*}

\begin{table*}[t]
\centering
\tiny
\setlength{\tabcolsep}{2pt}
\renewcommand{\arraystretch}{1.1}

\caption{Detailed results for Tabular--JEPA variants across ALOI, Helena, and Jannis datasets (continued).}
\label{tab:detailed_tabular_jepa_c}

\resizebox{\linewidth}{!}{%
\begin{tabular}{llcc||cc cc||cc cc||c c}
\toprule
\multirow{2}{*}{Dataset} &
\multirow{2}{*}{Context} &
\multirow{2}{*}{\shortstack{Predictor\\scale}} &
\multirow{2}{*}{\shortstack{Config\\ID}} &
\multicolumn{2}{c}{\shortstack{Epoch 0\\$\mu_{50}$}} &
\multicolumn{2}{c||}{\shortstack{Epoch 0\\Erank}} &
\multicolumn{2}{c}{\shortstack{Last epoch\\$\mu_{50}$}} &
\multicolumn{2}{c||}{\shortstack{Last epoch\\Erank}} &
\multirow{2}{*}{\shortstack{Final\\loss}} &
\multirow{2}{*}{\shortstack{Test\\Acc.}} \\
\cmidrule(lr){5-6}\cmidrule(lr){7-8}\cmidrule(lr){9-10}\cmidrule(lr){11-12}
& & & & Mean & SEM & Mean & SEM & Mean & SEM & Mean & SEM & & \\
\midrule
ALOI & 0.3 & 0.20 & 10 & $0.23$ & $0.007$ & $13.72$ & $0.27$ & $0.18$ & $7.53\cdot 10^{-4}$ & $7.90$ & $0.08$ & $0.0228$ & $0.738$ \\
Helena & 0.3 & 0.20 & 10 & $1.08$ & $0.019$ & $1.51$ & $0.02$ & $4.30$ & $1.980$ & $1.15$ & $0.00$ & $3.37\cdot 10^{-4}$ & $0.153$ \\
Jannis & 0.3 & 0.20 & 10 & $0.23$ & $9.50\cdot 10^{-4}$ & $9.38$ & $0.11$ & $0.22$ & $0.001$ & $4.01$ & $0.02$ & $0.0283$ & $0.625$ \\
ALOI & 0.3 & 0.30 & 11 & $0.15$ & $0.001$ & $11.82$ & $0.05$ & $0.11$ & $0.001$ & $6.66$ & $0.04$ & $0.0213$ & $0.706$ \\
Helena & 0.3 & 0.30 & 11 & $0.78$ & $0.063$ & $1.38$ & $0.01$ & $5.41$ & $2.414$ & $1.11$ & $0.00$ & $2.12\cdot 10^{-4}$ & $0.142$ \\
Jannis & 0.3 & 0.30 & 11 & $0.15$ & $9.40\cdot 10^{-4}$ & $8.67$ & $0.10$ & $0.14$ & $0.001$ & $3.55$ & $0.01$ & $0.0228$ & $0.612$ \\
ALOI & 0.3 & 0.50 & 12 & $0.10$ & $9.09\cdot 10^{-4}$ & $10.19$ & $0.20$ & $0.06$ & $3.84\cdot 10^{-4}$ & $5.18$ & $0.04$ & $0.0175$ & $0.686$ \\
Helena & 0.3 & 0.50 & 12 & $0.59$ & $0.016$ & $1.33$ & $0.02$ & $9.66$ & $2.677$ & $1.04$ & $0.00$ & $1.30\cdot 10^{-4}$ & $0.098$ \\
Jannis & 0.3 & 0.50 & 12 & $0.09$ & $0.002$ & $8.32$ & $0.28$ & $0.08$ & $0.002$ & $3.04$ & $0.04$ & $0.0157$ & $0.589$ \\
ALOI & 0.3 & 1.00 & 13 & $0.05$ & $8.16\cdot 10^{-4}$ & $8.73$ & $0.09$ & $0.03$ & $3.37\cdot 10^{-4}$ & $3.77$ & $0.04$ & $0.0131$ & $0.612$ \\
Helena & 0.3 & 1.00 & 13 & $0.43$ & $0.093$ & $1.27$ & $0.01$ & $8.89$ & $1.223$ & $1.02$ & $0.00$ & $4.76\cdot 10^{-5}$ & $0.085$ \\
Jannis & 0.3 & 1.00 & 13 & $0.04$ & $7.46\cdot 10^{-4}$ & $7.66$ & $0.25$ & $0.04$ & $7.88\cdot 10^{-4}$ & $2.17$ & $0.06$ & $0.0065$ & $0.561$ \\
ALOI & 0.4 & 0.01 & 14 & $1.10$ & $0.009$ & $22.17$ & $0.32$ & $0.53$ & $0.004$ & $15.86$ & $0.37$ & $0.0600$ & $0.782$ \\
Helena & 0.4 & 0.01 & 14 & $27.94$ & $2.370$ & $1.63$ & $0.03$ & $2.59$ & $0.167$ & $1.96$ & $0.02$ & $0.0087$ & $0.243$ \\
Jannis & 0.4 & 0.01 & 14 & $0.97$ & $0.058$ & $21.86$ & $0.42$ & $0.32$ & $0.004$ & $29.17$ & $0.35$ & $0.6053$ & $0.695$ \\
ALOI & 0.4 & 0.05 & 15 & $0.57$ & $0.024$ & $22.14$ & $0.21$ & $0.52$ & $0.002$ & $12.88$ & $0.12$ & $0.0348$ & $0.759$ \\
Helena & 0.4 & 0.05 & 15 & $6.85$ & $0.522$ & $1.74$ & $0.03$ & $4.10$ & $0.613$ & $1.58$ & $0.01$ & $0.0045$ & $0.183$ \\
Jannis & 0.4 & 0.05 & 15 & $0.72$ & $0.012$ & $14.18$ & $0.69$ & $0.30$ & $0.004$ & $29.00$ & $0.06$ & $0.5001$ & $0.697$ \\
ALOI & 0.4 & 0.10 & 16 & $0.39$ & $0.003$ & $17.52$ & $0.24$ & $0.40$ & $0.007$ & $10.21$ & $0.11$ & $0.0211$ & $0.736$ \\
Helena & 0.4 & 0.10 & 16 & $2.60$ & $0.304$ & $1.66$ & $0.02$ & $2.53$ & $0.463$ & $1.44$ & $0.00$ & $0.0032$ & $0.171$ \\
Jannis & 0.4 & 0.10 & 16 & $0.44$ & $0.003$ & $10.73$ & $0.44$ & $0.43$ & $0.002$ & $6.12$ & $0.06$ & $0.0502$ & $0.655$ \\
\bottomrule
\end{tabular}%
}
\end{table*}

\begin{table*}[t]
\centering
\tiny
\setlength{\tabcolsep}{2pt}
\renewcommand{\arraystretch}{1.1}

\caption{Detailed results for Tabular--JEPA variants across ALOI, Helena, and Jannis datasets (continued).}
\label{tab:detailed_tabular_jepa_d}

\resizebox{\linewidth}{!}{%
\begin{tabular}{llcc||cc cc||cc cc||c c}
\toprule
\multirow{2}{*}{Dataset} &
\multirow{2}{*}{Context} &
\multirow{2}{*}{\shortstack{Predictor\\scale}} &
\multirow{2}{*}{\shortstack{Config\\ID}} &
\multicolumn{2}{c}{\shortstack{Epoch 0\\$\mu_{50}$}} &
\multicolumn{2}{c||}{\shortstack{Epoch 0\\Erank}} &
\multicolumn{2}{c}{\shortstack{Last epoch\\$\mu_{50}$}} &
\multicolumn{2}{c||}{\shortstack{Last epoch\\Erank}} &
\multirow{2}{*}{\shortstack{Final\\loss}} &
\multirow{2}{*}{\shortstack{Test\\Acc.}} \\
\cmidrule(lr){5-6}\cmidrule(lr){7-8}\cmidrule(lr){9-10}\cmidrule(lr){11-12}
& & & & Mean & SEM & Mean & SEM & Mean & SEM & Mean & SEM & & \\
\midrule
ALOI & 0.4 & 0.20 & 17 & $0.22$ & $0.007$ & $13.60$ & $0.18$ & $0.18$ & $1.10\cdot 10^{-4}$ & $8.57$ & $0.15$ & $0.0252$ & $0.747$ \\
Helena & 0.4 & 0.20 & 17 & $1.18$ & $0.078$ & $1.46$ & $0.01$ & $4.61$ & $0.305$ & $1.16$ & $0.02$ & $3.61\cdot 10^{-4}$ & $0.148$ \\
Jannis & 0.4 & 0.20 & 17 & $0.22$ & $0.001$ & $9.57$ & $0.12$ & $0.22$ & $4.25\cdot 10^{-4}$ & $4.19$ & $0.04$ & $0.0263$ & $0.623$ \\
ALOI & 0.4 & 0.30 & 18 & $0.16$ & $5.16\cdot 10^{-4}$ & $11.63$ & $0.17$ & $0.11$ & $0.002$ & $7.11$ & $0.08$ & $0.0225$ & $0.736$ \\
Helena & 0.4 & 0.30 & 18 & $0.89$ & $0.151$ & $1.40$ & $0.01$ & $11.67$ & $4.567$ & $1.10$ & $0.02$ & $1.98\cdot 10^{-4}$ & $0.127$ \\
Jannis & 0.4 & 0.30 & 18 & $0.15$ & $0.003$ & $8.18$ & $0.06$ & $0.14$ & $0.001$ & $3.45$ & $0.04$ & $0.0169$ & $0.614$ \\
ALOI & 0.4 & 0.50 & 19 & $0.10$ & $0.001$ & $9.94$ & $0.13$ & $0.06$ & $5.98\cdot 10^{-4}$ & $5.25$ & $0.09$ & $0.0170$ & $0.693$ \\
Helena & 0.4 & 0.50 & 19 & $0.55$ & $0.054$ & $1.32$ & $0.01$ & $8.82$ & $1.423$ & $1.05$ & $0.00$ & $8.27\cdot 10^{-5}$ & $0.104$ \\
Jannis & 0.4 & 0.50 & 19 & $0.09$ & $0.002$ & $8.26$ & $0.27$ & $0.08$ & $0.002$ & $2.88$ & $0.02$ & $0.0113$ & $0.588$ \\
ALOI & 0.4 & 1.00 & 20 & $0.05$ & $5.20\cdot 10^{-4}$ & $8.38$ & $0.11$ & $0.03$ & $1.76\cdot 10^{-4}$ & $3.57$ & $0.09$ & $0.0123$ & $0.604$ \\
Helena & 0.4 & 1.00 & 20 & $0.28$ & $0.030$ & $1.26$ & $0.01$ & $2.79$ & $0.698$ & $1.01$ & $0.00$ & $2.74\cdot 10^{-5}$ & $0.069$ \\
Jannis & 0.4 & 1.00 & 20 & $0.04$ & $3.73\cdot 10^{-4}$ & $7.52$ & $0.35$ & $0.04$ & $1.67\cdot 10^{-4}$ & $2.10$ & $0.06$ & $0.0048$ & $0.562$ \\
ALOI & 0.5 & 0.01 & 21 & $0.97$ & $0.003$ & $23.35$ & $0.28$ & $0.50$ & $0.001$ & $19.37$ & $0.22$ & $0.0871$ & $0.812$ \\
Helena & 0.5 & 0.01 & 21 & $22.38$ & $2.732$ & $1.82$ & $0.06$ & $1.98$ & $0.158$ & $2.29$ & $0.06$ & $0.0150$ & $0.256$ \\
Jannis & 0.5 & 0.01 & 21 & $0.73$ & $0.026$ & $25.78$ & $0.16$ & $0.22$ & $0.013$ & $26.55$ & $0.96$ & $0.1944$ & $0.689$ \\
ALOI & 0.5 & 0.05 & 22 & $0.58$ & $0.005$ & $22.09$ & $0.23$ & $0.50$ & $0.004$ & $17.31$ & $0.34$ & $0.0563$ & $0.804$ \\
Helena & 0.5 & 0.05 & 22 & $5.37$ & $0.207$ & $1.90$ & $0.02$ & $3.94$ & $0.885$ & $1.80$ & $0.02$ & $0.0076$ & $0.203$ \\
Jannis & 0.5 & 0.05 & 22 & $0.50$ & $0.016$ & $20.61$ & $0.42$ & $0.21$ & $0.007$ & $28.71$ & $0.36$ & $0.1863$ & $0.693$ \\
ALOI & 0.5 & 0.10 & 23 & $0.39$ & $0.009$ & $17.08$ & $0.04$ & $0.38$ & $0.005$ & $12.62$ & $0.16$ & $0.0310$ & $0.776$ \\
Helena & 0.5 & 0.10 & 23 & $2.71$ & $0.194$ & $1.75$ & $0.04$ & $2.97$ & $0.584$ & $1.61$ & $0.02$ & $0.0052$ & $0.180$ \\
Jannis & 0.5 & 0.10 & 23 & $0.42$ & $0.008$ & $11.37$ & $0.58$ & $0.16$ & $0.006$ & $28.73$ & $0.55$ & $0.1113$ & $0.696$ \\
ALOI & 0.5 & 0.20 & 24 & $0.22$ & $0.010$ & $13.74$ & $0.20$ & $0.18$ & $0.001$ & $9.09$ & $0.14$ & $0.0287$ & $0.762$ \\
Helena & 0.5 & 0.20 & 24 & $1.18$ & $0.117$ & $1.45$ & $0.02$ & $14.86$ & $4.664$ & $1.15$ & $0.03$ & $2.90\cdot 10^{-4}$ & $0.148$ \\
Jannis & 0.5 & 0.20 & 24 & $0.22$ & $0.003$ & $9.42$ & $0.09$ & $0.21$ & $0.002$ & $5.04$ & $0.16$ & $0.0173$ & $0.640$ \\
ALOI & 0.5 & 0.30 & 25 & $0.16$ & $9.55\cdot 10^{-4}$ & $11.45$ & $0.10$ & $0.11$ & $7.21\cdot 10^{-4}$ & $7.04$ & $0.05$ & $0.0213$ & $0.739$ \\
Helena & 0.5 & 0.30 & 25 & $0.83$ & $0.076$ & $1.41$ & $0.01$ & $5.56$ & $0.694$ & $1.11$ & $0.01$ & $2.08\cdot 10^{-4}$ & $0.144$ \\
Jannis & 0.5 & 0.30 & 25 & $0.16$ & $0.002$ & $8.62$ & $0.24$ & $0.15$ & $5.17\cdot 10^{-4}$ & $3.59$ & $0.12$ & $0.0094$ & $0.619$ \\
ALOI & 0.5 & 0.50 & 26 & $0.10$ & $0.003$ & $9.80$ & $0.00$ & $0.07$ & $4.76\cdot 10^{-4}$ & $5.29$ & $0.05$ & $0.0163$ & $0.697$ \\
Helena & 0.5 & 0.50 & 26 & $0.56$ & $0.029$ & $1.31$ & $0.00$ & $16.82$ & $2.302$ & $1.04$ & $0.01$ & $6.46\cdot 10^{-5}$ & $0.106$ \\
Jannis & 0.5 & 0.50 & 26 & $0.09$ & $0.001$ & $7.68$ & $0.13$ & $0.09$ & $0.002$ & $2.74$ & $0.04$ & $0.0050$ & $0.599$ \\
ALOI & 0.5 & 1.00 & 27 & $0.05$ & $2.31\cdot 10^{-4}$ & $8.41$ & $0.05$ & $0.03$ & $2.82\cdot 10^{-4}$ & $3.65$ & $0.07$ & $0.0126$ & $0.605$ \\
Helena & 0.5 & 1.00 & 27 & $0.30$ & $0.024$ & $1.26$ & $0.02$ & $2.90$ & $1.130$ & $1.00$ & $7.17\cdot 10^{-4}$ & $3.89\cdot 10^{-7}$ & $0.062$ \\
Jannis & 0.5 & 1.00 & 27 & $0.04$ & $4.25\cdot 10^{-4}$ & $7.34$ & $0.15$ & $0.04$ & $1.98\cdot 10^{-4}$ & $2.04$ & $0.04$ & $0.0023$ & $0.565$ \\
ALOI & 0.6 & 0.01 & 28 & $0.87$ & $0.005$ & $25.14$ & $0.20$ & $0.44$ & $0.003$ & $24.83$ & $0.17$ & $0.1736$ & $0.864$ \\
Helena & 0.6 & 0.01 & 28 & $19.18$ & $2.844$ & $2.09$ & $0.05$ & $1.94$ & $0.166$ & $2.68$ & $0.11$ & $0.0182$ & $0.271$ \\
Jannis & 0.6 & 0.01 & 28 & $0.58$ & $0.017$ & $29.33$ & $0.18$ & $0.19$ & $0.002$ & $28.97$ & $0.16$ & $0.2704$ & $0.688$ \\
ALOI & 0.6 & 0.05 & 29 & $0.54$ & $0.007$ & $22.48$ & $0.33$ & $0.41$ & $8.99\cdot 10^{-4}$ & $16.44$ & $0.16$ & $0.0366$ & $0.833$ \\
Helena & 0.6 & 0.05 & 29 & $4.34$ & $0.246$ & $2.12$ & $0.05$ & $4.30$ & $0.441$ & $2.05$ & $0.02$ & $0.0098$ & $0.224$ \\
Jannis & 0.6 & 0.05 & 29 & $0.41$ & $0.002$ & $23.94$ & $0.02$ & $0.18$ & $0.003$ & $30.27$ & $0.26$ & $0.2489$ & $0.695$ \\
ALOI & 0.6 & 0.10 & 30 & $0.38$ & $0.005$ & $17.33$ & $0.24$ & $0.31$ & $0.002$ & $12.44$ & $0.13$ & $0.0322$ & $0.811$ \\
Helena & 0.6 & 0.10 & 30 & $2.11$ & $0.224$ & $1.77$ & $0.03$ & $11.73$ & $4.667$ & $1.77$ & $0.01$ & $0.0089$ & $0.189$ \\
Jannis & 0.6 & 0.10 & 30 & $0.40$ & $0.002$ & $12.78$ & $0.07$ & $0.14$ & $0.002$ & $28.69$ & $0.19$ & $0.1660$ & $0.692$ \\
ALOI & 0.6 & 0.20 & 31 & $0.21$ & $0.001$ & $13.47$ & $0.09$ & $0.16$ & $0.002$ & $8.39$ & $0.06$ & $0.0277$ & $0.753$ \\
Helena & 0.6 & 0.20 & 31 & $1.15$ & $0.158$ & $1.52$ & $0.01$ & $16.53$ & $6.522$ & $1.24$ & $0.01$ & $9.71\cdot 10^{-4}$ & $0.154$ \\
Jannis & 0.6 & 0.20 & 31 & $0.23$ & $0.006$ & $9.74$ & $0.25$ & $0.19$ & $0.003$ & $8.14$ & $0.13$ & $0.0634$ & $0.652$ \\
ALOI & 0.6 & 0.30 & 32 & $0.15$ & $0.003$ & $11.79$ & $0.16$ & $0.10$ & $0.001$ & $6.75$ & $0.12$ & $0.0235$ & $0.722$ \\
Helena & 0.6 & 0.30 & 32 & $0.76$ & $0.034$ & $1.39$ & $0.00$ & $10.88$ & $1.892$ & $1.13$ & $0.01$ & $3.04\cdot 10^{-4}$ & $0.144$ \\
Jannis & 0.6 & 0.30 & 32 & $0.15$ & $7.04\cdot 10^{-4}$ & $8.74$ & $0.02$ & $0.15$ & $0.001$ & $3.89$ & $0.07$ & $0.0192$ & $0.629$ \\
ALOI & 0.6 & 0.50 & 33 & $0.09$ & $5.30\cdot 10^{-4}$ & $10.24$ & $0.05$ & $0.06$ & $4.26\cdot 10^{-4}$ & $5.13$ & $0.06$ & $0.0192$ & $0.690$ \\
Helena & 0.6 & 0.50 & 33 & $0.44$ & $0.059$ & $1.31$ & $0.01$ & $6.11$ & $3.534$ & $1.02$ & $0.01$ & $3.22\cdot 10^{-5}$ & $0.071$ \\
Jannis & 0.6 & 0.50 & 33 & $0.09$ & $0.002$ & $8.06$ & $0.17$ & $0.09$ & $0.001$ & $2.70$ & $0.05$ & $0.0081$ & $0.586$ \\
ALOI & 0.6 & 1.00 & 34 & $0.05$ & $1.10\cdot 10^{-4}$ & $8.43$ & $0.21$ & $0.03$ & $3.95\cdot 10^{-4}$ & $3.40$ & $0.15$ & $0.0134$ & $0.597$ \\
Helena & 0.6 & 1.00 & 34 & $0.24$ & $0.016$ & $1.25$ & $0.01$ & $5.90$ & $0.950$ & $1.01$ & $0.00$ & $2.03\cdot 10^{-5}$ & $0.074$ \\
Jannis & 0.6 & 1.00 & 34 & $0.05$ & $4.99\cdot 10^{-4}$ & $7.40$ & $0.17$ & $0.04$ & $0.001$ & $1.98$ & $0.06$ & $0.0033$ & $0.558$ \\
\bottomrule
\end{tabular}%
}
\end{table*}

\begin{table*}[t]
\centering
\tiny
\setlength{\tabcolsep}{2pt}
\renewcommand{\arraystretch}{1.1}

\caption{Detailed results for Tabular--JEPA variants across ALOI, Helena, and Jannis datasets (continued).}
\label{tab:detailed_tabular_jepa_e}

\resizebox{\linewidth}{!}{%
\begin{tabular}{llcc||cc cc||cc cc||c c}
\toprule
\multirow{2}{*}{Dataset} &
\multirow{2}{*}{Context} &
\multirow{2}{*}{\shortstack{Predictor\\scale}} &
\multirow{2}{*}{\shortstack{Config\\ID}} &
\multicolumn{2}{c}{\shortstack{Epoch 0\\$\mu_{50}$}} &
\multicolumn{2}{c||}{\shortstack{Epoch 0\\Erank}} &
\multicolumn{2}{c}{\shortstack{Last epoch\\$\mu_{50}$}} &
\multicolumn{2}{c||}{\shortstack{Last epoch\\Erank}} &
\multirow{2}{*}{\shortstack{Final\\loss}} &
\multirow{2}{*}{\shortstack{Test\\Acc.}} \\
\cmidrule(lr){5-6}\cmidrule(lr){7-8}\cmidrule(lr){9-10}\cmidrule(lr){11-12}
& & & & Mean & SEM & Mean & SEM & Mean & SEM & Mean & SEM & & \\
\midrule
ALOI & 0.7 & 0.01 & 35 & $0.79$ & $0.018$ & $26.39$ & $0.09$ & $0.40$ & $0.005$ & $16.88$ & $0.21$ & $0.0309$ & $0.850$ \\
Helena & 0.7 & 0.01 & 35 & $12.34$ & $0.464$ & $2.67$ & $0.06$ & $1.62$ & $0.188$ & $2.83$ & $0.07$ & $0.0233$ & $0.272$ \\
Jannis & 0.7 & 0.01 & 35 & $0.47$ & $0.021$ & $30.97$ & $0.02$ & $0.17$ & $0.003$ & $29.25$ & $0.50$ & $0.2604$ & $0.690$ \\
ALOI & 0.7 & 0.05 & 36 & $0.53$ & $0.003$ & $22.37$ & $0.01$ & $0.37$ & $0.006$ & $13.36$ & $0.23$ & $0.0217$ & $0.821$ \\
Helena & 0.7 & 0.05 & 36 & $3.48$ & $0.272$ & $2.23$ & $0.06$ & $4.95$ & $0.874$ & $2.11$ & $0.03$ & $0.0106$ & $0.217$ \\
Jannis & 0.7 & 0.05 & 36 & $0.30$ & $0.003$ & $27.97$ & $0.15$ & $0.17$ & $9.88\cdot 10^{-4}$ & $30.87$ & $0.16$ & $0.2361$ & $0.695$ \\
ALOI & 0.7 & 0.10 & 37 & $0.34$ & $0.002$ & $17.35$ & $0.11$ & $0.25$ & $0.003$ & $10.09$ & $0.21$ & $0.0245$ & $0.798$ \\
Helena & 0.7 & 0.10 & 37 & $1.87$ & $0.177$ & $1.77$ & $0.02$ & $9.93$ & $1.704$ & $1.72$ & $0.01$ & $0.0084$ & $0.187$ \\
Jannis & 0.7 & 0.10 & 37 & $0.35$ & $0.009$ & $15.30$ & $0.46$ & $0.13$ & $0.003$ & $30.03$ & $0.74$ & $0.1651$ & $0.701$ \\
ALOI & 0.7 & 0.20 & 38 & $0.20$ & $0.003$ & $13.23$ & $0.12$ & $0.14$ & $0.002$ & $7.43$ & $0.10$ & $0.0252$ & $0.744$ \\
Helena & 0.7 & 0.20 & 38 & $0.87$ & $0.172$ & $1.48$ & $0.02$ & $22.80$ & $5.967$ & $1.26$ & $0.02$ & $0.0014$ & $0.165$ \\
Jannis & 0.7 & 0.20 & 38 & $0.23$ & $0.006$ & $9.57$ & $0.20$ & $0.12$ & $0.006$ & $13.90$ & $0.62$ & $0.1009$ & $0.669$ \\
ALOI & 0.7 & 0.30 & 39 & $0.14$ & $0.002$ & $11.55$ & $0.10$ & $0.09$ & $0.001$ & $6.02$ & $0.09$ & $0.0230$ & $0.716$ \\
Helena & 0.7 & 0.30 & 39 & $0.69$ & $0.100$ & $1.40$ & $0.01$ & $17.11$ & $5.718$ & $1.14$ & $0.00$ & $4.09\cdot 10^{-4}$ & $0.146$ \\
Jannis & 0.7 & 0.30 & 39 & $0.16$ & $0.003$ & $9.06$ & $0.20$ & $0.15$ & $0.005$ & $4.76$ & $0.10$ & $0.0389$ & $0.638$ \\
ALOI & 0.7 & 0.50 & 40 & $0.08$ & $0.002$ & $10.23$ & $0.07$ & $0.05$ & $2.85\cdot 10^{-4}$ & $4.75$ & $0.08$ & $0.0210$ & $0.672$ \\
Helena & 0.7 & 0.50 & 40 & $0.44$ & $0.020$ & $1.32$ & $0.01$ & $7.22$ & $1.065$ & $1.03$ & $0.01$ & $9.93\cdot 10^{-5}$ & $0.088$ \\
Jannis & 0.7 & 0.50 & 40 & $0.09$ & $0.002$ & $8.16$ & $0.21$ & $0.10$ & $0.003$ & $2.73$ & $0.03$ & $0.0119$ & $0.591$ \\
ALOI & 0.7 & 1.00 & 41 & $0.04$ & $4.09\cdot 10^{-4}$ & $8.52$ & $0.13$ & $0.03$ & $3.64\cdot 10^{-4}$ & $3.14$ & $0.07$ & $0.0158$ & $0.540$ \\
Helena & 0.7 & 1.00 & 41 & $0.20$ & $0.009$ & $1.25$ & $0.01$ & $4.48$ & $1.336$ & $1.01$ & $0.00$ & $1.24\cdot 10^{-5}$ & $0.067$ \\
Jannis & 0.7 & 1.00 & 41 & $0.05$ & $3.09\cdot 10^{-4}$ & $7.69$ & $0.35$ & $0.04$ & $9.66\cdot 10^{-4}$ & $1.88$ & $0.03$ & $0.0036$ & $0.551$ \\
\bottomrule
\end{tabular}%
}
\end{table*}

\subsection{T-JEPA}

\begin{table*}[t]
\centering
\small
\renewcommand{\arraystretch}{1.25}

\caption{Detailed results for T-JEPA on Jannis dataset. Each row shows results for a specific context ratio and predictor scale configuration, with means aggregated over seeds.}
\label{tab:detailed_tjepa}

\resizebox{\linewidth}{!}{%
\begin{tabular}{ccc||c c||c c||c c}
\toprule
\multirow{2}{*}{Context} &
\multirow{2}{*}{\shortstack{Predictor\\scale}} &
\multirow{2}{*}{\shortstack{Config\\ID}} &
\multicolumn{1}{c}{\shortstack{Epoch 0\\$\mu_{50}$}} &
\multicolumn{1}{c||}{\shortstack{Epoch 0\\Erank}} &
\multicolumn{1}{c}{\shortstack{Last epoch\\$\mu_{50}$}} &
\multicolumn{1}{c||}{\shortstack{Last epoch\\Erank}} &
\multirow{2}{*}{\shortstack{Final\\loss}} &
\multirow{2}{*}{\shortstack{Test\\Acc.}} \\
\addlinespace[8pt] 
\midrule
$0.1$ & $0.01$ & $0$  & $0.48$ & $6.54$ & $0.48$ & $5.53$ & $5.74\times10^{-5}$ & $0.464$ \\
$0.1$ & $0.05$ & $1$  & $0.09$ & $7.43$ & $0.08$ & $15.09$ & $5.185\times10^{-1}$ & $0.654$ \\
$0.1$ & $0.10$ & $2$  & $0.03$ & $7.02$ & $0.04$ & $3.93$ & $1.34\times10^{-6}$ & $0.464$ \\
$0.1$ & $0.20$ & $3$  & $0.02$ & $7.00$ & $0.02$ & $5.18$ & $3.93\times10^{-8}$ & $0.464$ \\
$0.2$ & $0.01$ & $4$  & $1.11$ & $5.74$ & $0.96$ & $4.61$ & $8.87\times10^{-5}$ & $0.464$ \\
$0.2$ & $0.05$ & $5$  & $0.21$ & $7.67$ & $0.20$ & $17.10$ & $2.556\times10^{-1}$ & $0.673$ \\
$0.2$ & $0.10$ & $6$  & $0.11$ & $6.89$ & $0.10$ & $4.65$ & $5.07\times10^{-7}$ & $0.464$ \\
$0.2$ & $0.20$ & $7$  & $0.05$ & $7.30$ & $0.05$ & $4.90$ & $6.52\times10^{-6}$ & $0.464$ \\
$0.3$ & $0.01$ & $8$  & $1.82$ & $6.25$ & $1.47$ & $11.30$ & $3.017\times10^{-1}$ & $0.639$ \\
$0.3$ & $0.05$ & $9$  & $0.35$ & $7.71$ & $0.33$ & $25.40$ & $3.077\times10^{-1}$ & $0.718$ \\
$0.3$ & $0.10$ & $10$ & $0.13$ & $6.54$ & $0.15$ & $4.70$ & $1.80\times10^{-7}$ & $0.464$ \\
$0.3$ & $0.20$ & $11$ & $0.07$ & $7.35$ & $0.08$ & $4.97$ & $3.89\times10^{-4}$ & $0.464$ \\
$0.4$ & $0.01$ & $12$ & $1.85$ & $6.11$ & $2.10$ & $4.25$ & $3.74\times10^{-5}$ & $0.464$ \\
$0.4$ & $0.05$ & $13$ & $0.46$ & $6.94$ & $0.44$ & $25.27$ & $1.691\times10^{-1}$ & $0.729$ \\
$0.4$ & $0.10$ & $14$ & $0.20$ & $7.35$ & $0.25$ & $4.53$ & $6.32\times10^{-7}$ & $0.464$ \\
$0.4$ & $0.20$ & $15$ & $0.12$ & $6.65$ & $0.12$ & $4.67$ & $5.54\times10^{-7}$ & $0.464$ \\
$0.5$ & $0.01$ & $16$ & $2.67$ & $6.41$ & $2.25$ & $4.74$ & $6.56\times10^{-5}$ & $0.464$ \\
$0.5$ & $0.05$ & $17$ & $0.71$ & $7.10$ & $0.59$ & $28.29$ & $1.451\times10^{-1}$ & $0.704$ \\
$0.5$ & $0.10$ & $18$ & $0.32$ & $7.15$ & $0.29$ & $4.32$ & $7.92\times10^{-7}$ & $0.464$ \\
$0.5$ & $0.20$ & $19$ & $0.17$ & $7.07$ & $0.19$ & $4.78$ & $4.96\times10^{-6}$ & $0.464$ \\
$0.6$ & $0.01$ & $20$ & $4.22$ & $6.23$ & $4.27$ & $4.65$ & $2.30\times10^{-5}$ & $0.464$ \\
$0.6$ & $0.05$ & $21$ & $0.89$ & $7.39$ & $0.95$ & $27.48$ & $1.266\times10^{-1}$ & $0.725$ \\
$0.6$ & $0.10$ & $22$ & $0.49$ & $6.90$ & $0.58$ & $4.94$ & $1.84\times10^{-4}$ & $0.464$ \\
\bottomrule
\end{tabular}%
}
\end{table*}

\subsection{T-JEPA with ResidualPred -- per-run details}
\label{app:full_results_tjepa_respred}

Tables~\ref{tab:tjepa_respred_perseed} and~\ref{tab:tjepa_respred_aggregate} expand the T-JEPA $\times$ ResidualPred ablation that backs the tabular columns of Table~\ref{tab:respred-tabular-image} in the main text. Table~\ref{tab:tjepa_respred_perseed} is deliberately the subcritical stress grid: its five configurations sit where the criterion places both predictors below threshold, so no per-seed prevention claim is made there; the criterion widens the escape basin without guaranteeing escape. The recurring probe values are class priors. On \texttt{jannis}, $46.01\%$ is the frequency of class $3$ and $34.38\%$ that of class $1$ in the training split, so a collapsed representation yields a near-constant probe and the same value recurs to two decimals across unrelated runs. Escapes are discontinuous, at $65.36\%$ on \texttt{jannis} and $32.58\%$ on \texttt{helena}; probe accuracy and effective rank co-sign on every row; no Standard run escapes; and on \texttt{helena} the separation is complete, six against six, with a ResidualPred minimum of $15.39$ against a Standard maximum of $11.58$ on the probe.

\begin{figure}[htb]
\centering
\includegraphics[width=0.5\linewidth]{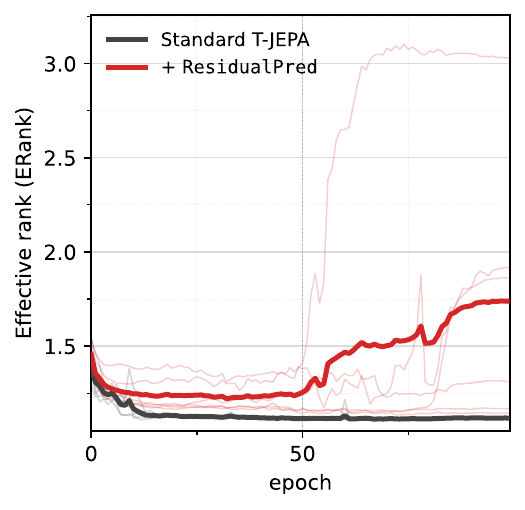}
\caption{\small\textbf{Helena ERank trajectories} at deliberately brittle configurations, where the criterion predicts collapse (mean across six runs; per-seed thin lines). The observable range is about $[1.1, 9]$, since min--max preprocessing leaves a large common mean (collapse attractor near $1.1$, fresh encoder $1.3$ to $1.9$, raw features near $8.8$); Standard lands at $1.11$ to $1.12$ on all six seeds, every \code{ResidualPred} seed above, at $1.14$ to $3.03$, and the probe moves from $8.6\%$ to $21.3\%$.}
\label{fig:erank_trajectories_helena_avg}
\end{figure}

\begin{table*}[t]
\centering
\small
\setlength{\tabcolsep}{4pt}
\renewcommand{\arraystretch}{1.05}
\caption{Per-seed T-JEPA + ResidualPred results on \texttt{jannis} and \texttt{helena}, on the subcritical stress grid, where the criterion places both predictors below threshold: no per-seed prevention is claimed, and the recurring probe values are class priors ($46.01\%$ and $34.38\%$ on \texttt{jannis}). ERank values are reported at epoch $0$ and at the last training epoch; final loss is the training loss at the last epoch; probe accuracy is the linear-probe test score evaluated on the frozen encoder at the last epoch.}
\label{tab:tjepa_respred_perseed}
\resizebox{0.95\linewidth}{!}{%
\begin{tabular}{llccc||cc||c|c}
\toprule
Variant & Dataset & $p_{m_1}$ & $\alpha$ & Seed & ERank$_0$ & ERank$_{\mathrm{last}}$ & Final loss & Probe acc.\ (\%) \\
\midrule
Standard & jannis & $0.075$ & $0.20$ & $0$ & $1.71$ & $1.27$ & $4.78\!\times\!10^{-7}$ & $46.01$ \\
Standard & jannis & $0.075$ & $0.20$ & $1$ & $1.79$ & $1.29$ & $2.54\!\times\!10^{-5}$ & $34.38$ \\
Standard & jannis & $0.075$ & $0.20$ & $2$ & $1.79$ & $1.29$ & $6.33\!\times\!10^{-9}$ & $46.01$ \\
Standard & jannis & $0.150$ & $0.50$ & $0$ & $1.65$ & $1.28$ & $2.37\!\times\!10^{-9}$ & $46.01$ \\
Standard & jannis & $0.150$ & $0.50$ & $1$ & $1.71$ & $1.28$ & $1.09\!\times\!10^{-8}$ & $34.38$ \\
Standard & jannis & $0.150$ & $0.50$ & $2$ & $1.66$ & $1.27$ & $5.32\!\times\!10^{-9}$ & $46.01$ \\
Standard & jannis & $0.150$ & $1.00$ & $0$ & $1.78$ & $1.27$ & $2.12\!\times\!10^{-12}$ & $46.01$ \\
Standard & jannis & $0.150$ & $1.00$ & $1$ & $1.68$ & $1.28$ & $3.18\!\times\!10^{-9}$ & $46.01$ \\
Standard & jannis & $0.150$ & $1.00$ & $2$ & $1.76$ & $1.28$ & $1.18\!\times\!10^{-8}$ & $46.01$ \\
\midrule
+ResPred & jannis & $0.075$ & $0.20$ & $0$ & $1.70$ & $1.25$ & $4.75\!\times\!10^{-8}$ & $34.38$ \\
+ResPred & jannis & $0.075$ & $0.20$ & $1$ & $1.69$ & $2.90$ & $6.27\!\times\!10^{-2}$ & $\mathbf{65.36}$ \\
+ResPred & jannis & $0.075$ & $0.20$ & $2$ & $1.69$ & $1.33$ & $6.13\!\times\!10^{-9}$ & $34.38$ \\
+ResPred & jannis & $0.150$ & $0.50$ & $0$ & $1.79$ & $1.27$ & $1.38\!\times\!10^{-6}$ & $46.01$ \\
+ResPred & jannis & $0.150$ & $0.50$ & $1$ & $1.81$ & $1.26$ & $2.14\!\times\!10^{-8}$ & $34.38$ \\
+ResPred & jannis & $0.150$ & $0.50$ & $2$ & $1.90$ & $4.03$ & $1.40\!\times\!10^{-1}$ & $\mathbf{61.65}$ \\
+ResPred & jannis & $0.150$ & $1.00$ & $0$ & $1.77$ & $1.27$ & $2.49\!\times\!10^{-5}$ & $46.01$ \\
+ResPred & jannis & $0.150$ & $1.00$ & $1$ & $1.79$ & $1.24$ & $1.29\!\times\!10^{-4}$ & $46.46$ \\
+ResPred & jannis & $0.150$ & $1.00$ & $2$ & $1.85$ & $1.27$ & $3.83\!\times\!10^{-8}$ & $46.01$ \\
\midrule
Standard & helena & $0.150$ & $0.50$ & $0$ & $1.54$ & $1.11$ & $2.05\!\times\!10^{-4}$ & $\phantom{0}6.14$ \\
Standard & helena & $0.150$ & $0.50$ & $1$ & $1.37$ & $1.12$ & $7.56\!\times\!10^{-3}$ & $10.05$ \\
Standard & helena & $0.150$ & $0.50$ & $2$ & $1.34$ & $1.12$ & $2.18\!\times\!10^{-3}$ & $11.58$ \\
Standard & helena & $0.150$ & $1.00$ & $0$ & $1.37$ & $1.12$ & $3.23\!\times\!10^{-4}$ & $\phantom{0}6.14$ \\
Standard & helena & $0.150$ & $1.00$ & $1$ & $1.29$ & $1.11$ & $2.08\!\times\!10^{-3}$ & $11.27$ \\
Standard & helena & $0.150$ & $1.00$ & $2$ & $1.37$ & $1.12$ & $2.31\!\times\!10^{-4}$ & $\phantom{0}6.14$ \\
\midrule
+ResPred & helena & $0.150$ & $0.50$ & $0$ & $1.53$ & $3.03$ & $1.62\!\times\!10^{-1}$ & $\mathbf{32.58}$ \\
+ResPred & helena & $0.150$ & $0.50$ & $1$ & $1.45$ & $1.86$ & $1.67\!\times\!10^{-2}$ & $\mathbf{27.63}$ \\
+ResPred & helena & $0.150$ & $0.50$ & $2$ & $1.38$ & $1.31$ & $1.19\!\times\!10^{-2}$ & $15.48$ \\
+ResPred & helena & $0.150$ & $1.00$ & $0$ & $1.54$ & $1.91$ & $1.07\!\times\!10^{-2}$ & $21.06$ \\
+ResPred & helena & $0.150$ & $1.00$ & $1$ & $1.47$ & $1.14$ & $2.03\!\times\!10^{-4}$ & $15.48$ \\
+ResPred & helena & $0.150$ & $1.00$ & $2$ & $1.38$ & $1.17$ & $4.24\!\times\!10^{-4}$ & $15.39$ \\
\bottomrule
\end{tabular}%
}
\end{table*}

\begin{table}[t]
\centering
\small
\setlength{\tabcolsep}{5pt}
\renewcommand{\arraystretch}{1.1}
\caption{Per-configuration aggregates (mean $\pm$ standard deviation across $3$ seeds) for the runs of Table~\ref{tab:tjepa_respred_perseed}.}
\label{tab:tjepa_respred_aggregate}
\resizebox{\linewidth}{!}{%
\begin{tabular}{llcc||cc|cc|cc}
\toprule
Variant & Dataset & $p_{m_1}$ & $\alpha$ & \multicolumn{2}{c|}{ERank$_{\mathrm{last}}$} & \multicolumn{2}{c|}{Final loss} & \multicolumn{2}{c}{Probe acc.\ (\%)} \\
\cmidrule(lr){5-6}\cmidrule(lr){7-8}\cmidrule(lr){9-10}
& & & & Mean & Std & Mean & Std & Mean & Std \\
\midrule
Standard & jannis & $0.075$ & $0.20$ & $1.28$ & $0.01$ & $8.61\!\times\!10^{-6}$ & $1.45\!\times\!10^{-5}$ & $42.13$ & $6.71$ \\
Standard & jannis & $0.150$ & $0.50$ & $1.28$ & $0.01$ & $6.21\!\times\!10^{-9}$ & $4.35\!\times\!10^{-9}$ & $42.13$ & $6.71$ \\
Standard & jannis & $0.150$ & $1.00$ & $1.27$ & $0.00$ & $5.00\!\times\!10^{-9}$ & $6.11\!\times\!10^{-9}$ & $46.01$ & $0.00$ \\
Standard & helena & $0.150$ & $0.50$ & $1.12$ & $0.00$ & $3.31\!\times\!10^{-3}$ & $3.81\!\times\!10^{-3}$ & $\phantom{0}9.25$ & $2.81$ \\
Standard & helena & $0.150$ & $1.00$ & $1.12$ & $0.01$ & $8.79\!\times\!10^{-4}$ & $1.04\!\times\!10^{-3}$ & $\phantom{0}7.85$ & $2.97$ \\
\midrule
+ResPred & jannis & $0.075$ & $0.20$ & $1.83$ & $0.93$ & $2.09\!\times\!10^{-2}$ & $3.62\!\times\!10^{-2}$ & $44.71$ & $17.89$ \\
+ResPred & jannis & $0.150$ & $0.50$ & $2.18$ & $1.59$ & $4.68\!\times\!10^{-2}$ & $8.11\!\times\!10^{-2}$ & $47.35$ & $13.68$ \\
+ResPred & jannis & $0.150$ & $1.00$ & $1.26$ & $0.02$ & $5.12\!\times\!10^{-5}$ & $6.83\!\times\!10^{-5}$ & $46.16$ & $\phantom{0}0.26$ \\
+ResPred & helena & $0.150$ & $0.50$ & $2.07$ & $0.88$ & $6.34\!\times\!10^{-2}$ & $8.52\!\times\!10^{-2}$ & $25.23$ & $\phantom{0}8.80$ \\
+ResPred & helena & $0.150$ & $1.00$ & $1.41$ & $0.44$ & $3.79\!\times\!10^{-3}$ & $6.02\!\times\!10^{-3}$ & $17.31$ & $\phantom{0}3.25$ \\
\bottomrule
\end{tabular}%
}
\end{table}

\subsection{I-JEPA Detailed Results on CIFAR-10, CIFAR-100, and STL-10}
\label{app:full_results_image}

Tables~\ref{tab:ijepa_perseed_image} and~\ref{tab:ijepa_aggregate_image} expand the image columns of Table~\ref{tab:respred-tabular-image} in the main text. All runs use the I-JEPA configuration of Section~\ref{sec:methodology} with ViT-Small/$4$ on CIFAR-10/CIFAR-100 ($L = 64$ tokens) and ViT-Small/$8$ on STL-10 ($L = 144$ tokens), context retention $p_{m_1} = 0.85$, AdamW optimiser, and $100$-epoch cosine schedule.

\begin{table*}[t]
\centering
\small
\setlength{\tabcolsep}{4pt}
\renewcommand{\arraystretch}{1.05}
\caption{Per-seed I-JEPA results on CIFAR-10, CIFAR-100, and STL-10 at $p_{m_1} = 0.85$, $100$ epochs. $\mu_{50}^{(0)}$ is the median per-mode initialisation stability ratio; $\mathrm{ERank}_0$, $\mathrm{ERank}_{100}$ are the effective rank at epoch $0$ and epoch $100$; Probe is the linear-probe test accuracy at epoch $100$.}
\label{tab:ijepa_perseed_image}
\resizebox{0.92\linewidth}{!}{%
\begin{tabular}{llc||cccc}
\toprule
Dataset & Variant & Seed & $\mu_{50}^{(0)}$ & ERank$_0$ & ERank$_{100}$ & Probe (\%) \\
\midrule
CIFAR-10  & Standard & $0$ & $0.0076$ & $4.14$ & $34.16$ & $62.89$ \\
CIFAR-10  & Standard & $1$ & $0.0083$ & $4.29$ & $40.99$ & $64.44$ \\
CIFAR-10  & Standard & $2$ & $0.0093$ & $4.16$ & $36.93$ & $66.31$ \\
CIFAR-10  & +ResPred & $0$ & $0.0076$ & $4.14$ & $50.93$ & $68.78$ \\
CIFAR-10  & +ResPred & $1$ & $0.0083$ & $4.29$ & $55.90$ & $69.10$ \\
CIFAR-10  & +ResPred & $2$ & $0.0093$ & $4.16$ & $63.09$ & $67.95$ \\
\midrule
CIFAR-100 & Standard & $\phantom{00}0$  & $0.0077$ & $3.59$ & $41.85$ & $39.32$ \\
CIFAR-100 & Standard & $\phantom{0}07$  & $0.0085$ & $3.68$ & $31.90$ & $35.46$ \\
CIFAR-100 & Standard & $\phantom{0}42$  & $0.0089$ & $3.61$ & $44.17$ & $36.14$ \\
CIFAR-100 & Standard & $123$ & $0.0078$ & $3.62$ & $32.31$ & $36.25$ \\
CIFAR-100 & Standard & $314$ & $0.0085$ & $3.71$ & $44.58$ & $37.00$ \\
CIFAR-100 & +ResPred & $\phantom{00}0$  & $0.0080$ & $3.58$ & $45.84$ & $41.46$ \\
CIFAR-100 & +ResPred & $\phantom{0}07$  & $0.0077$ & $3.65$ & $45.34$ & $40.36$ \\
CIFAR-100 & +ResPred & $\phantom{0}42$  & $0.0072$ & $3.60$ & $44.41$ & $41.44$ \\
CIFAR-100 & +ResPred & $123$ & $0.0078$ & $3.62$ & $41.32$ & $39.99$ \\
CIFAR-100 & +ResPred & $314$ & $0.0078$ & $3.69$ & $44.46$ & $41.00$ \\
\midrule
STL-10    & Standard & $0$ & $0.0259$ & $7.38$ & $131.75$ & $71.49$ \\
STL-10    & Standard & $1$ & $0.0286$ & $6.89$ & $118.46$ & $70.01$ \\
STL-10    & Standard & $2$ & $0.0274$ & $6.98$ & $115.58$ & $72.25$ \\
STL-10    & +ResPred & $0$ & $0.0250$ & $7.34$ & $163.39$ & $74.78$ \\
STL-10    & +ResPred & $1$ & $0.0286$ & $6.89$ & $163.66$ & $76.20$ \\
STL-10    & +ResPred & $2$ & $0.0271$ & $6.99$ & $162.93$ & $74.95$ \\
\bottomrule
\end{tabular}%
}
\end{table*}

\begin{table}[t]
\centering
\small
\setlength{\tabcolsep}{6pt}
\renewcommand{\arraystretch}{1.1}
\caption{Per-cell aggregates of Table~\ref{tab:ijepa_perseed_image} (mean $\pm$ standard deviation across seeds). The Probe column reproduces the image rows of Table~\ref{tab:respred-tabular-image} of the main text exactly to two decimal places.}
\label{tab:ijepa_aggregate_image}
\begin{tabular}{ll|c|cc|cc|cc}
\toprule
Dataset & Variant & $n$ & \multicolumn{2}{c|}{ERank$_0$} & \multicolumn{2}{c|}{ERank$_{100}$} & \multicolumn{2}{c}{Probe (\%)} \\
\cmidrule(lr){4-5}\cmidrule(lr){6-7}\cmidrule(lr){8-9}
& & & Mean & Std & Mean & Std & Mean & Std \\
\midrule
CIFAR-10  & Standard & $3$ & $4.20$ & $0.08$ & $\phantom{0}37.36$ & $\phantom{0}3.44$ & $64.55$ & $1.71$ \\
CIFAR-10  & +ResPred & $3$ & $4.20$ & $0.08$ & $\phantom{0}56.64$ & $\phantom{0}6.12$ & $68.61$ & $0.59$ \\
CIFAR-100 & Standard & $5$ & $3.64$ & $0.05$ & $\phantom{0}38.96$ & $\phantom{0}6.35$ & $36.83$ & $1.49$ \\
CIFAR-100 & +ResPred & $5$ & $3.63$ & $0.04$ & $\phantom{0}44.28$ & $\phantom{0}1.76$ & $40.85$ & $0.66$ \\
STL-10    & Standard & $3$ & $7.08$ & $0.26$ & $121.93$ & $\phantom{0}8.63$ & $71.25$ & $1.14$ \\
STL-10    & +ResPred & $3$ & $7.07$ & $0.23$ & $163.33$ & $\phantom{0}0.37$ & $75.31$ & $0.78$ \\
\bottomrule
\end{tabular}
\end{table}

\subsection{Image experiments beyond CIFAR: protocol}
\label{app:image_scale}

All runs use one I-JEPA pipeline \citep{assranSelfSupervisedLearningImages2023} with identical loss,
optimiser, EMA teacher and masking, a ViT-S/16 encoder at $128$ pixels unless stated otherwise, and
differ only in the predictor; the loss-side arms of Appendix~\ref{app:image_scale:comparison} add one loss term. ImageNet-100 \citep{tianContrastiveMultiviewCoding2020}: $150$ epochs,
predictor depths $2$ and $6$, three seeds per predictor; the backbone check compares ViT-S/16 and
ViT-B/16 at epoch $100$, depth $6$. ImageNet-1k \citep{dengImageNetLargeScaleHierarchical2009}: a
matched-budget pilot of $30$ epochs at depth $6$ with five seeds per predictor, and two stresses that
change one variable each, $224$ pixels at $30$ epochs (five seeds) and $90$ epochs at $128$ pixels
(three seeds).

Readouts on the frozen encoder are a linear probe, a parameter-free nearest-neighbour classifier
($k = 20$) and an attentive probe; effective rank is RankMe \citep{garridoRankMeAssessingDownstream2023},
logged every five epochs. On ImageNet-1k the nearest-neighbour and attentive readouts use a
class-balanced bank of $200$ images per class, identical for every arm, a deviation from the
full-split protocol.

\subsection{ImageNet-100 at two depths and three probes}
\label{app:image_scale:in100}

\begin{table}[htb]
\centering\small
\setlength{\tabcolsep}{6pt}
\caption{ImageNet-100, I-JEPA ViT-S/16, $128$ pixels, $150$ epochs, three seeds per predictor:
\texttt{ResidualPred} minus the standard predictor. Every \texttt{ResidualPred} seed exceeds every
standard seed on all three probes at both depths; exact paired permutation $p = 0.031$ per probe.}
\label{tab:in100}
\begin{tabular}{@{}lcc@{}}
\toprule
Readout & Depth 2 & Depth 6 \\
\midrule
Linear probe & $+3.6$ pp & $+5.6$ pp \\
Nearest neighbour, $k = 20$ & $+8.3$ pp & $+5.5$ pp \\
Attentive probe & $+2.8$ pp & $+5.5$ pp \\
Effective rank & $+49.1$ ($90.7 \to 139.8$) & $+26.9$ \\
\bottomrule
\end{tabular}
\end{table}

The two predictors start matched in effective rank, and \texttt{ResidualPred}'s is higher at every
logged epoch from the fifth onward in every matched seed. At the matched $100$-epoch, depth-$6$ setting
the effect holds on a backbone with four times the parameters: ViT-B/16 goes from $0.339$ to $0.398$
on the linear probe, $+6.0$ points with no seed overlap, and effective rank rises by $23.9$; ViT-S/16
gains $4.4$ points at the same setting. Across the $18$ ImageNet-100 and ViT-B runs, no
effective-rank trajectory falls more than $0.3\%$ below its running maximum after escape.

\subsection{A thousand-class pilot and its two stresses}
\label{app:image_scale:in1k}

\begin{table}[htb]
\centering\small
\setlength{\tabcolsep}{5pt}
\caption{The thousand-class pilot (ImageNet-1k, ViT-S/16, $128$ pixels, depth $6$, $30$ epochs,
matched budget), seed means over five seeds per predictor. On the linear probe and effective rank
every \texttt{ResidualPred} seed exceeds every standard seed ($p = 0.008$); the nearest-neighbour and
attentive readouts use the class-balanced bank.}
\label{tab:in1k}
\begin{tabular}{@{}lccc@{}}
\toprule
Readout at epoch 30 & Standard & \texttt{ResidualPred} & $\Delta$ \\
\midrule
Linear probe & $0.2522$ & $\mathbf{0.3169}$ & $+6.48$ pp \\
Nearest neighbour, $k = 20$ & $0.0916$ & $\mathbf{0.1774}$ & $+8.58$ pp \\
Attentive probe & $0.2847$ & $\mathbf{0.3292}$ & $+4.45$ pp \\
Effective rank & $107.2$ & $\mathbf{188.5}$ & $+81.3$ \\
\bottomrule
\end{tabular}
\end{table}

The predictors start rank-matched, separate by epoch five, and the gap widens in seed means. The
parameter-free margin is larger at a thousand classes than at a hundred ($+8.58$ against $+5.5$
points at depth $6$). This is a matched-budget ablation at $128$ pixels; nothing is claimed about the
official full-resolution, full-schedule recipe.

\begin{table}[htb]
\centering\small
\setlength{\tabcolsep}{5pt}
\caption{The two stresses, recipe otherwise unchanged. Linear-probe top-1, seed mean and range.
$224$ pixels: five seeds per predictor. $90$ epochs: three seeds per predictor, complete
separation.}
\label{tab:in1k_stress}
\begin{tabular}{@{}lccc@{}}
\toprule
Setting & Standard & \texttt{ResidualPred} & $\Delta$ \\
\midrule
$224$ pixels, $30$ epochs & $0.3398$ ($0.3366$ to $0.3419$) & $\mathbf{0.4014}$ ($0.3834$ to $0.4091$) & $+6.16$ pp \\
$128$ pixels, $90$ epochs & $0.3279$ ($0.3208$ to $0.3348$) & $\mathbf{0.3624}$ ($0.3507$ to $0.3727$) & $+3.45$ pp \\
\bottomrule
\end{tabular}
\end{table}

At $224$ pixels the margin is kept, with the \texttt{ResidualPred} seed spread widening to $2.6$
points. Over $90$ epochs the margin attenuates to $+3.45$ points; the standard predictor's effective
rank rises from $111$ at epoch $30$ to $179$ at epoch $90$ in seed means, against $200$ to $251$ for \texttt{ResidualPred},
and matching \texttt{ResidualPred}'s $30$-epoch accuracy costs the standard predictor up to three
times the compute.

\subsection{Controls: frozen and skip-only predictors}
\label{app:image_scale:controls}

Two controls at the ImageNet-100 depth-$6$ setting, three seeds each. A skip-only predictor is
degenerate in a JEPA: with no path from context to target the encoder gradient is identically zero,
so it serves only as a null anchor. Freezing the exact \texttt{ResidualPred} predictor at its
initialisation, with everything else trained, collapses the representation instead of reverting to the
standard predictor: the linear probe falls from the trained \texttt{ResidualPred}'s $0.427$ to $0.099$ and
effective rank from $82.1$ to $1.67$, every frozen seed below every standard seed. On two seeds with identical data order the frozen
and trained runs start at the same effective rank ($24.7444$ and $25.2536$) and are separated by epoch
five ($1.15$ to $2.73$ against $25.9$ to $30.6$ across the three seeds). The \texttt{ResidualPred}
predictor denied learning loses the entire gain.

In the linear model with $\mat{W}_p$ fixed, the regime of Lemma~\ref{lemma:coupled_dynamics}, the
fastest super-critical mode dominates and effective rank relaxes to $1.00$ to $1.04$ in all $17$
settings screened; in the trained system, $\lambda_{\max}(\mat{W}_p^\top\mat{W}_p)$
re-equilibrates to $0.80$ to $1.10$ from initialisations spanning a $500\times$ range. In the public
solo-learn BYOL implementation \citep{costaSololearnLibrarySelfsupervised2022} (CIFAR-100, ResNet-18,
no moving-average teacher), effective rank and accuracy increase with the ratio of predictor to
encoder learning rate over $\{0, 0.1, 1, 10\}$ in both seeds, the frozen predictor being the worst.

\subsection{Comparison with loss-side mechanisms at the pilot setting}
\label{app:image_scale:comparison}

SIGReg is implemented as the Epps--Pulley statistic of LeJEPA
\citep{balestrieroLeJEPAProvableScalable2025}, with $48$ resampled slices per step computed per rank,
and added to the I-JEPA loss with weight $\lambda$. This isolates SIGReg as an anti-collapse term and
is not a reproduction of LeJEPA, which removes the EMA teacher and the predictor. $\lambda = 0.05$ is
the authors' recommended default, bracketed by $\lambda \in \{0.02, 0.15\}$. The weighting is
additive; the convex form of the released LeJEPA algorithm differs from it by a rescaling of the
prediction term by about $5\%$ at this dose.

\begin{table}[htb]
\centering\small
\setlength{\tabcolsep}{5pt}
\caption{Loss-side comparison at the pilot setting (ImageNet-1k, $128$ pixels, $30$ epochs, depth
$6$); $n$ is the number of seeds. SIGReg beats the standard predictor by $3.98$ points and
\texttt{ResidualPred} beats SIGReg by $2.50$ points, each with complete separation; all $35$ pairwise seed differences between \texttt{ResidualPred} and SIGReg
are positive, from $+1.42$ to $+3.73$ points.}
\label{tab:sigreg}
\resizebox{\linewidth}{!}{%
\begin{tabular}{@{}lcccccc@{}}
\toprule
Predictor and loss & $n$ & Linear & Range & Nearest neighbour & Attentive & Effective rank \\
\midrule
Standard & 5 & $0.2522$ & $0.2409$ to $0.2628$ & $0.0916$ & $0.2847$ & $107.2$ \\
Standard $+$ SIGReg, $\lambda = 0.02$ & 2 & $0.2821$ & $0.2789$ to $0.2853$ & & & $154$ to $180$ \\
Standard $+$ SIGReg, $\lambda = 0.05$ (default) & 7 & $0.2919$ & $0.2830$ to $0.3001$ & $0.1600$ & $0.3100$ & $179$ to $196$ \\
Standard $+$ SIGReg, $\lambda = 0.15$ & 2 & $0.2865$ & $0.2861$ to $0.2870$ & & & $190$ to $201$ \\
\texttt{ResidualPred}, $\beta_{init} = 10$ & 5 & $\mathbf{0.3169}$ & $0.3143$ to $0.3203$ & $0.1774$ & $0.3292$ & $188.5$ \\
\texttt{ResidualPred} $+$ SIGReg, $\lambda = 0.05$ & 5 & $\mathbf{0.3223}$ & $0.3159$ to $0.3304$ & $0.1903$ & $0.3341$ & $205$ to $208$ \\
\bottomrule
\end{tabular}}
\end{table}

SIGReg is active in this pipeline: its penalty falls from $0.033$ at initialisation to $0.0011$ by the
final epoch, and it raises effective rank from $107$ to $186$. \texttt{ResidualPred} beats it in this
setting, and the two compose: combining them adds $18.2$ to the effective rank of
\texttt{ResidualPred} alone ($90\%$ interval $+15.0$ to $+21.4$). Mimetic attention initialisation
\citep{trockmanMimeticInitializationSelfAttention2023} improves on the standard predictor by $1.6$ points
with seed overlap and trails \texttt{ResidualPred} by $4.9$ points. A VICReg arm at the coefficient
ratio of the original paper destabilised in both seeds (linear probe $0.026$ and $0.062$) and is
excluded; the C-JEPA integration \citep{moConnectingJointEmbeddingPredictive} differs from it, so these
runs say nothing about VICReg itself. \texttt{ResidualPred} adds one hyperparameter,
$\beta_{init} = 10$, unchanged on every dataset, and no loss term or per-step computation.

\section{Seven anti-collapse mechanisms as doses on one ratio: verification on the linear model}
\label{app:mechanisms}

For each mechanism below we derive its effect on $\mu = \gamma/\sigma$ under
Assumption~\ref{ass:spectral_decoupling}, predict from the initialisation how many modes escape, and
count the modes that escape when the linear model, Eq.~\eqref{eq:dW_coupled} with the mechanism added,
is integrated. Across $152$ (mechanism, dose, seed) runs the predicted count equals the measured count
in all $152$. All runs are simulations of the linear model: the table verifies the derivations against
the model they concern and is not a test of the criterion against JEPA training, which is
Section~\ref{sec:residual_predictor} and Appendix~\ref{app:image_scale}.

\begin{table}[htb]
\centering\small
\setlength{\tabcolsep}{4pt}
\renewcommand{\arraystretch}{1.15}
\caption{Seven anti-collapse mechanisms as doses on one ratio, verified on the linear model: number of
escaping modes predicted from the derived $\gamma$ and $\sigma$ at initialisation against the number
that grows when Eq.~\eqref{eq:dW_coupled} is integrated with the mechanism added. All $152$ runs are
simulations of the linear model, not deep networks.}
\label{tab:unification}
\resizebox{\linewidth}{!}{%
\begin{tabular}{@{}p{3.3cm}p{7.6cm}p{5.6cm}@{}}
\toprule
\textbf{Mechanism} & \textbf{Derived effect on $\mu = \gamma/\sigma$} & \textbf{Predicted vs measured} \\
\midrule
Weight decay on the encoder &
Raises every mode's decay by $\eta_{wd}$: escape iff $\gamma_i - \sigma_i > \eta_{wd}$. &
Five doses, three seeds ($15$ runs); predicted counts $24/19/8/2/0$, matched in every run. \\
\addlinespace
Predictor scaling (Mechanism~1) &
$\mu(\alpha) = \mu(1)/\alpha$; nothing escapes above $\alpha_{ext} = \max_i \mu_i(1)$. &
$\alpha$ ladder bracketing $\alpha_{ext} = 4.16$, three seeds ($30$ runs); count exact in every run. \\
\addlinespace
SIGReg \citep{balestrieroLeJEPAProvableScalable2025} &
Lowers the decay's predictor factor $p_\ell$ by $\rho$: escape iff $\mu > 1 - \rho/p_\ell$. &
Fifteen doses, three seeds ($45$ runs), plus six teacher-free doses ($p_{tgt} = 0$); counts $224$ to $512$ and
$[0, 64, 160, 288, 384, 512]$, exact in every run. \\
\addlinespace
VICReg hinge \citep{bardesVICRegVarianceInvarianceCovarianceRegularization2022} &
Lowers the decay's predictor factor by $\lambda_v/s$ for a mode of standard deviation $s$; rescue dose
$\propto p_\ell(1-\mu)$. &
Four doses, three seeds ($12$ runs); $12$ of $12$ exact. \\
\addlinespace
Mimetic initialisation \citep{trockmanMimeticInitializationSelfAttention2023} &
Doses the alignment ratio $q_\ell/p_\ell$ once, washed out at rate $\zeta$. &
Eight wash-out rates, three seeds ($24$ runs); $24$ of $24$ exact. \\
\addlinespace
Identity predictor &
With $\mat{W}_p = \mat{I}$ frozen, mode $k$ escapes iff $s_k > (1-p)/(p_{tgt} - p)$. &
Four configurations plus a BYOL-style setting, three seeds ($15$ runs); $5/8/4/8$ and $0$, exact. \\
\addlinespace
EMA momentum \citep{grillBootstrapYourOwn2020} &
$\beta$ cannot flip the sign of the constant term: for a real drive, momentum sets escape speed, not
membership. &
Five rates, $\beta = 0.01$ to $50$ ($5$ runs); count $8$ at all five. \\
\midrule
\textbf{All seven} & & \textbf{$152$ runs; predicted $=$ measured in all $152$.} \\
\bottomrule
\end{tabular}}
\end{table}

\subsection{The seven doses}
\label{app:mechanisms:doses}

Throughout, $a_k, c_k$ are the eigenvalues of $\mat{\Sigma}_{11}$ and $\mat{\Sigma}_{X1}$ along data
mode $k$, and $p_\ell, q_\ell$ those of $\mat{P} = \mat{W}_p^\top\mat{W}_p$ and
$\mat{Q} = \mat{W}_p^\top\bar{\mat{M}}_2$ along predictor mode $\ell$, so that
$\gamma_{k\ell} = c_k q_\ell$, $\sigma_{k\ell} = a_k p_\ell$ and $\mu_{k\ell} = \gamma_{k\ell}/\sigma_{k\ell}$.

\begin{proposition}[Encoder weight decay is an additive floor]
\label{prop:wd}
Adding $-\eta_{wd}\mat{W}$ to Eq.~\eqref{eq:dW_coupled} replaces the decay operator by
$\mat{\Sigma}_{11}\otimes\mat{P} + \eta_{wd}\mat{I}$, so $\sigma'_{k\ell} = \sigma_{k\ell} + \eta_{wd}$
with $\gamma_{k\ell}$ unchanged, and mode $(k,\ell)$ escapes if and only if
$\gamma_{k\ell} - \sigma_{k\ell} > \eta_{wd}$. This assumes the decay acts on the context encoder only.
\end{proposition}

\begin{proposition}[Predictor scaling divides the ratio]
\label{prop:scale}
Under $\mat{W}_p \mapsto \alpha\mat{W}_p$, $p_\ell$ scales as $\alpha^2$ and $q_\ell$ as $\alpha$, so
$\mu_{k\ell}(\alpha) = \mu_{k\ell}(1)/\alpha$ and nothing escapes above
$\alpha_{ext} = \max_{k\ell}\mu_{k\ell}(1)$.
\end{proposition}

\begin{proposition}[An isotropy penalty is negative weight decay in embedding space]
\label{prop:sigreg}
Let the auxiliary term contribute $+\lambda\kappa_0\mat{W}\mat{\Sigma}_{11}$ to the flow as
$\mat{W}\to\mat{0}$, with $\kappa_0$ twice the magnitude of the estimator's slope at zero embedding
variance divided by the embedding dimension. The decay block becomes
$-\mat{\Sigma}_{11}\otimes[\mat{P} - \rho\mat{I}]$ with $\rho := \lambda\kappa_0/(1-\lambda)$. For
$p_\ell > \rho$ escape holds if and only if $\mu_{k\ell} > 1 - \rho/p_\ell$; the mode also escapes for
$p_\ell \le \rho$ with $\gamma_{k\ell} > 0$, and for $p_\ell < \rho$ with $\gamma_{k\ell} = 0$. The linearisation uses only the
second moment of a centred distribution, and $\kappa_0$ is calibrated per setup.
\end{proposition}

\begin{proposition}[The variance hinge mirrors weight decay]
\label{prop:vicreg}
With a hinge $\lambda_v\sum_i\max(0, g-s_i)$ on embedding standard deviations $s_i$, the drive is
unchanged and, while the hinge is active, the decay becomes $\sigma_{k\ell}(s) = a_k(p_\ell - \lambda_v/s_\ell)$.
In the small-weight limit this is $a_k(p_\ell - \tilde\lambda)$ with
$\tilde\lambda = \lambda_v/\sqrt{\varepsilon_v}$, so for $p_\ell > \tilde\lambda$ escape holds if and only if
$\mu_{k\ell} > 1 - \tilde\lambda/p_\ell$ ($\varepsilon_v$ is the variance offset inside the hinge's square root). The covariance term is cubic in the weights and does not
enter the linearisation.
\end{proposition}

\begin{proposition}[Mimetic initialisation is a decaying dose on the predictor factor]
\label{prop:mimetic}
For a normal predictor with $\ell$-th eigenvalue $|\lambda_\ell|e^{\mathrm{i}\phi_\ell}$,
$\mathrm{Re}\,\mu_{k\ell} = (c_k/a_k)(p_{tgt}/|\lambda_\ell|)\cos\phi_\ell$. For
$\mat{W}_p(t) \propto m(t)\mat{I} + \mat{O}$ at fixed Frobenius norm, with $\mat{O}$ orthogonal with
rotation angles $\theta_\ell$ and identity dose $m(t) = m_0e^{-\zeta t}$,
$\cos\phi_\ell(t) = \bigl(m(t) + \cos\theta_\ell\bigr)/\sqrt{m(t)^2 + 2m(t)\cos\theta_\ell + 1}$. In the
adiabatic approximation a mode escapes when $\int_0^T\lambda^+_{k}(t)\,dt > \log(10^3)$.
\end{proposition}

\begin{proposition}[A frozen identity predictor is a condition on the data spectrum]
\label{prop:identity}
With $\mat{W}_p = \mat{I}$ fixed, $\mat{P} = \mat{I}$ and $\mat{Q} = p_{tgt}\mat{I}$. Under Bernoulli
feature masking with keep probability $p$ on a unit-diagonal data covariance with spectrum $\{s_k\}$,
$\mu_k = p_{tgt}s_k/\bigl((1-p) + p\,s_k\bigr)$, so mode $k$ escapes if and only if
$s_k > (1-p)/(p_{tgt}-p)$, which requires $p_{tgt} > p$. Under symmetric additive view noise instead of
masking (the BYOL-style setting), $c_k/a_k < 1$ for every mode and nothing escapes.
\end{proposition}

\begin{proposition}[Target-encoder momentum sets speed, not membership]
\label{prop:ema}
In $\lambda^2 + (\beta+\sigma_{k\ell})\lambda + \beta\sigma_{k\ell}(1-\mu_{k\ell})$ the rate
$\beta>0$ does not change the sign of the constant term, so for a real drive the escaping set is
independent of $\beta$, while the escape rate
$\lambda^+(\beta) = \tfrac12\bigl[-(\beta+\sigma) + \sqrt{(\beta-\sigma)^2 + 4\beta\gamma}\bigr]$ is
strictly increasing in $\beta$ for a super-critical mode.
\end{proposition}

\subsection{Four published results recovered as limiting cases}
\label{app:mechanisms:recoveries}

DirectPred \citep{tianUnderstandingSelfsupervisedLearning2021} sets the predictor from a moving average
of the embedding second moment, which makes $\mat{P}$ and $\mat{Q}$ commute; with predictor eigenvalue
$\hat p_k = \sqrt{s_k} + \varepsilon$, every surviving mode is held at $\mu_k = 1$ at its fixed point, and their growth condition
$\Delta_j > 0$ is $\mu_j > 1$ in the adiabatic limit where the target is a fixed multiple $\tau$ of the
context encoder. With weight decay $\eta$, mode $k$ has ratio $\tau c_k p_{tgt}\hat p/(a_k\hat p^2 + \eta)$, whose
supremum over $\hat p$ gives the death threshold $\eta > \tau^2\max_k(c_k^2 p_{tgt}^2/a_k)/4$, which
reproduces their isotropic threshold with a relative error of $4.1\times10^{-9}$.
The rate equation of \citet{halvagalImplicitVarianceRegularization2023} is the flow-level case of the
criterion under predictor slaving, and IsoLoss, from the same paper and in the same regime, equalises
escape rates: on the
tabular sweep it compresses the spread of escape epochs from $73$ to $7$.
Orthonormalising the predictor \citep{richemondEdgeOrthogonality2023} pins
$\mat{W}_p^\top\mat{W}_p = \mat{I}$, a projection on the decay side, so $\sigma_{k\ell} = a_k$ and
$q_\ell = p_{tgt}e^{-\mathrm{i}\theta_\ell}$; the exact escape boundary is
$g\cos\theta_\ell + \beta g^2\sin^2\theta_\ell/(\beta + a_k)^2 > a_k$ with $g := p_{tgt}c_k$, which
matched a direct eigensolve in all of $20{,}000$ random draws.

\section{Measuring the drive on images after the predictor equilibrates}
\label{app:equilibration}

This appendix reports the two measurements behind the statement of
Section~\ref{sec:residual_predictor:scale} that, on images, the mediating variable is measured after
predictor equilibration and not at initialisation.

\subsection{The initialisation-time estimator is null, on the whole predictor block}
\label{app:equilibration:init}

The estimator of Section~\ref{sec:practical_control} reads the predictor's two linear maps, which are
initialised identically with and without the diagonal bias, so on images it cannot see the
intervention. We therefore linearised the whole predictor block (attention with the diagonal bias,
feed-forward and both projections) with respect to its input at initialisation, on real masked batches
and identical encoders, and formed $\tilde\mu$ from the singular values of that Jacobian, at predictor
depths $2$, $6$ and $12$ with three seeds each. It does not separate the two predictors: the seed
ranges overlap at every depth, at depth $6$ from $0.313$ to $0.326$ for \texttt{ResidualPred} against
$0.312$ to $0.315$ for the standard predictor. A magnitude functional discards the sign of the drive,
and a large singular value can coexist with no growth (Theorem~\ref{thm:collapse_certificate}(v)).

\subsection{The target-aligned drive after a few predictor-only steps}
\label{app:equilibration:slaved}

We probe along the data eigenmodes $b_\ell$ and read the signed drive
$\hat\gamma_\ell = \langle \bar r_\ell, b_\ell\rangle$, the drive side of Eq.~\eqref{eq:dW_coupled} with
$\mat{W}_p$ replaced by the block's input Jacobian, after $K$ predictor-only steps with the encoder fixed at its
initialisation, the initial encoder as teacher, and the recipe's optimiser and learning rate. This
functional was introduced after the initialisation-time result above.

In every matched initialisation at every depth, \texttt{ResidualPred}'s aligned drive at $K = 10$ is
above the standard predictor's, $9$ of $9$, with per-seed ratios $1.97$, $1.80$, $1.70$ at depth $2$,
$2.60$, $2.14$, $1.17$ at depth $6$, and $5.43$, $4.53$, $9.29$ at depth $12$. The leading real
eigenvalue of the block's transmission operator, which keeps the sign structure the singular-value form
of Appendix~\ref{app:driving_spectrum_reality} discards, is also higher in the same $9$ of $9$, with ratios
from $1.06$ to $8.56$. Over $K \in \{1, 2, 3, 5, 7, 10\}$ the paired ratio is above one in all $36$
(seed, depth, $K$) cells at depths $2$ and $6$, peaking near $K = 5$ to $7$.

These measurements support two statements: the identity bias raises the target-aligned drive the predictor
establishes in its first equilibration steps, and the initialisation-time quantity is null on images. Embedding
effective rank is the only quantity claimed to track the accuracy gain, as necessary rather than
sufficient and as depth-limited: on CIFAR-100 at predictor depth $12$, SIGReg raises effective rank by
$34.5$ to $61.1$ while accuracy does not move. No estimator of $\mu$ is claimed to mediate the accuracy gain.

\newpage
\section*{NeurIPS Paper Checklist}

\begin{enumerate}

\item {\bf Claims}
    \item[] Question: Do the main claims made in the abstract and introduction accurately reflect the paper's contributions and scope?
    \item[] Answer: \answerYes{}
    \item[] Justification: Abstract and Sec.~\ref{sec:introduction} state the drive--decay theory, the $\mu$-threshold, the $800{+}$-configuration validation, and the \texttt{ResidualPred} gain on tabular and image benchmarks; each claim maps to a section (Sec.~\ref{sec:theory}, Sec.~\ref{sec:practical_control}, Sec.~\ref{sec:residual_predictor}).

\item {\bf Limitations}
    \item[] Question: Does the paper discuss the limitations of the work performed by the authors?
    \item[] Answer: \answerYes{}
    \item[] Justification: Sec.~\ref{sec:discussion} (Limitations paragraph) discusses linearity, the spectral-decoupling assumption, dataset-dependent tightness, and the local nature of $\mu{=}1$.

\item {\bf Theory assumptions and proofs}
    \item[] Question: For each theoretical result, does the paper provide the full set of assumptions and a complete (and correct) proof?
    \item[] Answer: \answerYes{}
    \item[] Justification: Lemma~\ref{lemma:coupled_dynamics}, Assumption~\ref{ass:spectral_decoupling}, Corollary~\ref{cor:rank-control}, and Theorem~\ref{thm:stability} are stated with assumptions; full proofs are in App.~\ref{app:detailed_derivation}, App.~\ref{app:rank-control-proof}, App.~\ref{app:assumptions}.

    \item {\bf Experimental result reproducibility}
    \item[] Question: Does the paper fully disclose all the information needed to reproduce the main experimental results of the paper to the extent that it affects the main claims and/or conclusions of the paper (regardless of whether the code and data are provided or not)?
    \item[] Answer: \answerYes{}
    \item[] Justification: Sec.~\ref{sec:methodology} and App.~\ref{app:training_details} report datasets, splits, optimiser, schedule, and seed protocol; supplementary material ships the full pipeline with reproduction scripts.

\item {\bf Open access to data and code}
    \item[] Question: Does the paper provide open access to the data and code, with sufficient instructions to faithfully reproduce the main experimental results, as described in supplemental material?
    \item[] Answer: \answerYes{}
    \item[] Justification: Anonymised code, configs, and per-run results are in the supplementary, with \texttt{reproduce\_all.sh}, \texttt{reproduce\_subset.sh}, \texttt{test.sh}; tabular data via OpenML, image data via \texttt{torchvision}.

\item {\bf Experimental setting/details}
    \item[] Question: Does the paper specify all the training and test details (e.g., data splits, hyperparameters, how they were chosen, type of optimizer) necessary to understand the results?
    \item[] Answer: \answerYes{}
    \item[] Justification: Sec.~\ref{sec:methodology} (splits, optimiser, schedule), App.~\ref{app:training_details} (full hyperparameters), and App.~\ref{app:beta_choice}/App.~\ref{app:zero_init} (variant-specific settings).

\item {\bf Experiment statistical significance}
    \item[] Question: Does the paper report error bars suitably and correctly defined or other appropriate information about the statistical significance of the experiments?
    \item[] Answer: \answerYes{}
    \item[] Justification: Tables report mean$\pm$stdev across seeds; the rank-control study reports Spearman $\rho$.

\item {\bf Experiments compute resources}
    \item[] Question: For each experiment, does the paper provide sufficient information on the computer resources (type of compute workers, memory, time of execution) needed to reproduce the experiments?
    \item[] Answer: \answerYes{}
    \item[] Justification: App.~\ref{app:training_details} and the supplementary \texttt{README.md} list per-experiment GPU-hours on a single A100.

\item {\bf Code of ethics}
    \item[] Question: Does the research conducted in the paper conform, in every respect, with the NeurIPS Code of Ethics \url{https://neurips.cc/public/EthicsGuidelines}?
    \item[] Answer: \answerYes{}
    \item[] Justification: Theoretical and methodological work on representation learning using public datasets; no human subjects, no scraped or sensitive data.

\item {\bf Broader impacts}
    \item[] Question: Does the paper discuss both potential positive societal impacts and negative societal impacts of the work performed?
    \item[] Answer: \answerNA{}
    \item[] Justification: Foundational theoretical work on JEPA training stability; no direct path to a deployable application or societal impact beyond improved representation learning.

\item {\bf Safeguards}
    \item[] Question: Does the paper describe safeguards that have been put in place for responsible release of data or models that have a high risk for misuse (e.g., pre-trained language models, image generators, or scraped datasets)?
    \item[] Answer: \answerNA{}
    \item[] Justification: No high-risk model or dataset is released; only research code and per-run training metrics.

\item {\bf Licenses for existing assets}
    \item[] Question: Are the creators or original owners of assets (e.g., code, data, models), used in the paper, properly credited and are the license and terms of use explicitly mentioned and properly respected?
    \item[] Answer: \answerYes{}
    \item[] Justification: I-JEPA \citep{assranSelfSupervisedLearningImages2023} and T-JEPA \citep{thimonierTJEPAAugmentationFreeSelfSupervised2024} cited; OpenML datasets (\texttt{aloi}, \texttt{helena}, \texttt{jannis}) and \texttt{torchvision} CIFAR-10/100, STL-10 used under their public licenses.

\item {\bf New assets}
    \item[] Question: Are new assets introduced in the paper well documented and is the documentation provided alongside the assets?
    \item[] Answer: \answerYes{}
    \item[] Justification: Supplementary ships an anonymised codebase with \texttt{README.md} and reproduction scripts.

\item {\bf Crowdsourcing and research with human subjects}
    \item[] Question: For crowdsourcing experiments and research with human subjects, does the paper include the full text of instructions given to participants and screenshots, if applicable, as well as details about compensation (if any)?
    \item[] Answer: \answerNA{}
    \item[] Justification: No crowdsourcing or human-subjects research.

\item {\bf Institutional review board (IRB) approvals or equivalent for research with human subjects}
    \item[] Question: Does the paper describe potential risks incurred by study participants, whether such risks were disclosed to the subjects, and whether Institutional Review Board (IRB) approvals (or an equivalent approval/review based on the requirements of your country or institution) were obtained?
    \item[] Answer: \answerNA{}
    \item[] Justification: No human-subjects research.

\item {\bf Declaration of LLM usage}
    \item[] Question: Does the paper describe the usage of LLMs if it is an important, original, or non-standard component of the core methods in this research? Note that if the LLM is used only for writing, editing, or formatting purposes and does \emph{not} impact the core methodology, scientific rigor, or originality of the research, declaration is not required.
    \item[] Answer: \answerNA{}
    \item[] Justification: LLMs are not part of the core methodology; any use was limited to writing and editing.

\end{enumerate}

\end{document}